\documentclass[lettersize,journal, 10pt]{IEEEtran}
\usepackage[utf8]{inputenc} 
\usepackage[T1]{fontenc}    
\usepackage{hyperref}       
\usepackage{url}            
\usepackage{cite}
\usepackage{booktabs}       
\usepackage{amsfonts}       
\usepackage{nicefrac}       
\usepackage{microtype}      
\usepackage{xcolor}         
\usepackage{pifont}

\usepackage{amssymb}
\usepackage{mathtools}
\usepackage{amsthm}
\usepackage{amsmath}
\usepackage{wrapfig}
\usepackage[capitalize,noabbrev]{cleveref}

\usepackage{algorithm}
\usepackage{algorithmic}

\usepackage{graphicx}
\usepackage{subcaption}
\usepackage{adjustbox}

\theoremstyle{plain}
\newtheorem{theorem}{Theorem}
\newtheorem{proposition}{Proposition}
\newtheorem{lemma}{Lemma}

\theoremstyle{definition}
\newtheorem{definition}{Definition}
\newtheorem{assumption}{Assumption}
\theoremstyle{remark}
\newtheorem{remark}[theorem]{Remark}
\newcommand{\A}{\mathcal{A}}
\newcommand{\D}{\mathcal{D}}

\newcommand{\LL}{\mathcal{L}}

\renewcommand{\thetable}{\arabic{table}}

\def\1{\bm{1}}

\DeclareMathAlphabet{\mathsfit}{\encodingdefault}{\sfdefault}{m}{sl}
\SetMathAlphabet{\mathsfit}{bold}{\encodingdefault}{\sfdefault}{bx}{n}

\newcommand{\E}{\mathbb{E}}

\DeclareMathOperator*{\argmax}{arg\,max}

\begin{document}
\title{Meta Stackelberg Game: Robust Federated Learning against Adaptive and Mixed Poisoning Attacks}

\author{Tao Li, Henger Li, Yunian Pan, Tianyi Xu, Zizhan Zheng, Quanyan Zhu 

\thanks{The first two authors contributed equally to this work. (Corresponding author: Tao Li). }
\thanks{Tao Li, Yunian Pan, and Quanyan Zhu are with the Department of Electrical and Computer Engineering, New York University. \texttt{tl2636, yp1170, qz494@nyu.edu}}
\thanks{Henger Li, Tianyi Xu, and Zizhan Zheng are with the Department of Computer Science, Tulane University. \texttt{hli30, txu9, zzheng3@tulane.edu} }

}



\maketitle

\begin{abstract}
    Federated learning (FL) is susceptible to a range of security threats. Although various defense mechanisms have been proposed, they are typically non-adaptive and tailored to specific types of attacks, leaving them insufficient in the face of multiple uncertain, unknown, and adaptive attacks employing diverse strategies. This work formulates adversarial federated learning under a mixture of various attacks as a Bayesian Stackelberg Markov game, based on which we propose the meta-Stackelberg defense composed of pre-training and online adaptation. {The gist is to simulate strong attack behavior using reinforcement learning (RL-based attacks) in pre-training and then design meta-RL-based defense to combat diverse and adaptive attacks.} We develop an efficient meta-learning approach to solve the game, leading to a robust and adaptive FL defense. Theoretically, our meta-learning algorithm, meta-Stackelberg learning, provably converges to the first-order $\varepsilon$-meta-equilibrium point in $O(\varepsilon^{-2})$ gradient iterations with $O(\varepsilon^{-4})$ samples per iteration. Experiments show that our meta-Stackelberg framework performs superbly against strong model poisoning and backdoor attacks of uncertain and unknown types.
\end{abstract}


\section{Introduction}
\IEEEPARstart{F}ederated learning (FL) allows multiple devices with private data to jointly train a model without sharing their local data~\cite{mcmahan2017communication}. However, FL systems are vulnerable to various adversarial attacks such as untargeted model poisoning attacks (e.g., IPM \cite{xie2020fall}, LMP \cite{fang2020local}) and backdoor attacks (e.g., BFL \cite{bagdasaryan2020backdoor}, DBA \cite{zhang2022distributed}). To address these vulnerabilities, various robust aggregation rules such as Krum~\cite{blanchard2017machine}, coordinate-wise trimmed mean~\cite{yin2018byzantine}, and FLTrust~\cite{cao2020FLTrust} have been proposed to defend against untargeted attacks, and both training-stage and post-training defenses such as Norm bounding~\cite{sun2019can}, NeuroClip~\cite{wang2023mm}, and Prun~\cite{wu2020mitigating} have been proposed to mitigate backdoor attacks. Further, dynamic defenses that myopically adapt parameters such as learning rate \cite{ozdayi2021defending}, norm clipping threshold \cite{guo2021resisting}, and regularizer \cite{acar2020federated} have been proposed. However, state-of-the-art defenses remain inadequate in countering advanced adaptive attacks (e.g., the reinforcement learning (RL)-based attacks~\cite{li2022learning, li2023learning}) that dynamically adjust the attack strategy to achieve long-term objectives. Further, current defenses are typically designed to counter specific types of attacks, rendering them ineffective in the presence of mixed attacks. As shown in \Cref{table:1} in Section~\ref{sec:exp}, simply combining existing defenses with manual tuning proves ineffective due to the interference between defense methods, the defender's lack of information about adversaries, and the dynamic nature of FL.

This work proposes a meta-Stackelberg game (meta-SG) framework that obtains superb defense performance even in the presence of strong adaptive attacks and mixed attacks of the same or different types (e.g., the coexistence of model poisoning and backdoor attacks).
\textcolor{green}{}
Our meta-SG defense framework is built upon the following key observations. First, when the attack type (to be defined in Section~\ref{sec:overview}) is known as priori, the defender can utilize the limited amount of local data at the server and publicly available information to build an approximate world model of the FL system. This allows the defender to identify a robust defense policy offline by solving either a Markov decision process (MDP) when the attack is non-adaptive or a 
Markov game when the attack is adaptive. This approach naturally applies to both a single attack and the coexistence of multiple attacks and leads to a (nearly) optimal defense. Second, when the attacks are unknown or uncertain, as in more realistic settings, the problem can be formulated as a Bayesian Stackelberg Markov game (BSMG)~\cite{sengupta20bsmg}, offering a general model for adversarial FL. However, the standard solution concept for BSMG, namely, the Bayesian Stackelberg equilibrium, targets the expected case and does not adapt to the actual attacks of certain unknown/uncertain types. 

To tackle this limitation, we propose in \Cref{def:meta-se} a novel solution concept called meta-Stackelberg equilibrium (meta-SE) for BSMG as a principled way of developing robust and adaptive defenses for FL. By integrating meta-learning and Stackelberg reasoning, meta-SE offers a computationally efficient approach to address information asymmetry in adversarial FL and enables strategic adaptation in online execution in the presence of multiple (adaptive) attackers. Before training an FL model, a meta policy is learned by solving the BSMG using experiences sampled from a set of possible attacks. When facing an actual attacker during online FL training, the meta-policy is quickly adapted using a relatively small number of samples collected on the fly. The proposed meta-SG framework only requires a rough estimate of possible worst-case attacks during meta-training, thanks to the generalization ability brought by meta-learning as theoretically certified in \Cref{main_prop:generalization}.

To solve the BSMG in the pre-training phase, we propose a meta-Stackelberg learning (meta-SL) algorithm (\Cref{algo:meta-sl}) based on the debiased meta-reinforcement learning 
approach in \cite{fallah2021convergence}. The meta-SL provably converges to the first-order $\varepsilon$-approximate meta-SE in $O(\varepsilon^{-2})$ iterations, and the associated sample complexity per iteration is of  $O(\varepsilon^{-4})$. Even though meta-SL achieves state-of-the-art sample efficiency in bi-level stochastic optimization as in  \cite{ji2021bilevel}, its operation involves the Hessian of the defender's value function.

To obtain a more practical solution (to bypass the Hessian computation), we further propose a fully first-order pre-training algorithm, called Reptile meta-SL, inspired by Reptile \cite{nichol2018first}. Reptile meta-SL only utilizes the first-order stochastic gradients from the attacker's and the defender's problem to solve for the approximate equilibrium. The numerical results in \cref{table:1} demonstrate its effectiveness in handling various types of non-adaptive attacks, including mixed attacks, while Fig.~\ref{fig:untargeted} and Fig.~\ref{fig:additional} highlight its efficiency in 
coping with uncertain or unknown attacks, including adaptive attacks.  \textbf{Our contributions} are summarized as follows.
\begin{itemize}
    \item We address critical security problems in FL when attacks are adaptive or mixed with multiple types, which are beyond the manual combination of existing defenses. 
    \item We develop a Bayesian Stackelberg game model (\Cref{subsec:bsmg}) to capture the information asymmetry in the adversarial FL under multiple uncertain/unknown 
    attacks.
    \item To create a strategically adaptable defense, we propose a new equilibrium concept: meta-Stackelberg equilibrium (\ref{eq:meta-se}), where the defender (the leader) designs a meta policy and an adaptation strategy by anticipating and adapting to the attacker's moves, leading to a data-driven approach to tackle information asymmetry.
    \item To learn the meta equilibrium defense in the pre-training phase, we develop meta-Stackelberg learning (\Cref{algo:meta-sl}), an efficient first-order meta RL algorithm, which provably converges to $\varepsilon$-approximate equilibrium in $O(\varepsilon^{-2})$ gradient steps with $O(\varepsilon^{-4})$ samples per iteration, matching the state-of-the-art sample efficiency.
    \item We conduct extensive experiments in real-world settings to demonstrate the superb performance of the meta-Stackelberg method.
\end{itemize}
Our work falls within the realm of RL and game-theoretic defenses against mixed attacks in FL. {To the best of our knowledge, we the first work to utilize RL and game-theoretical techniques to defend against mixed attacks in FL.}
\Cref{sec:related} gives a detailed review of related works.  
\section{Meta Stackelberg Defense Framework}
\label{sec:overview}
As shown in Fig~\ref{fig:fl-game}, the meta-learning framework includes two stages: \textit{pre-training}, \textit{online adaptation}. The \textit{pre-training} stage is implemented in a simulated environment, which allows sufficient training using trajectories generated from the interactions between the defender and 
the attacker with its type randomly sampled from a set of potential attacks. Both adaptive and non-adaptive attacks could be considered for pre-training. After obtaining a meta-policy, the defender will interact with the real FL environment in the \textit{online adaptation} stage to tune its defense policy using {feedback (i.e., model updates and environment parameters)} received in the presence of real attacks that are not necessarily in the pre-training attack set.  
Finally, at the last round of FL training, the defender will perform a post-training defense on the global model.
Pre-training and online adaptation are indispensable in the proposed framework. \cref{table:3} in Appendix~\ref{app:add-exp} summarizes the experiments on directly applying defense learned from pre-training without online adaptation and adaptation from a randomly initialized defense policy without pre-training, both failing to address malicious attacks.

\begin{figure}
    \centering
    \includegraphics[width=1\linewidth]{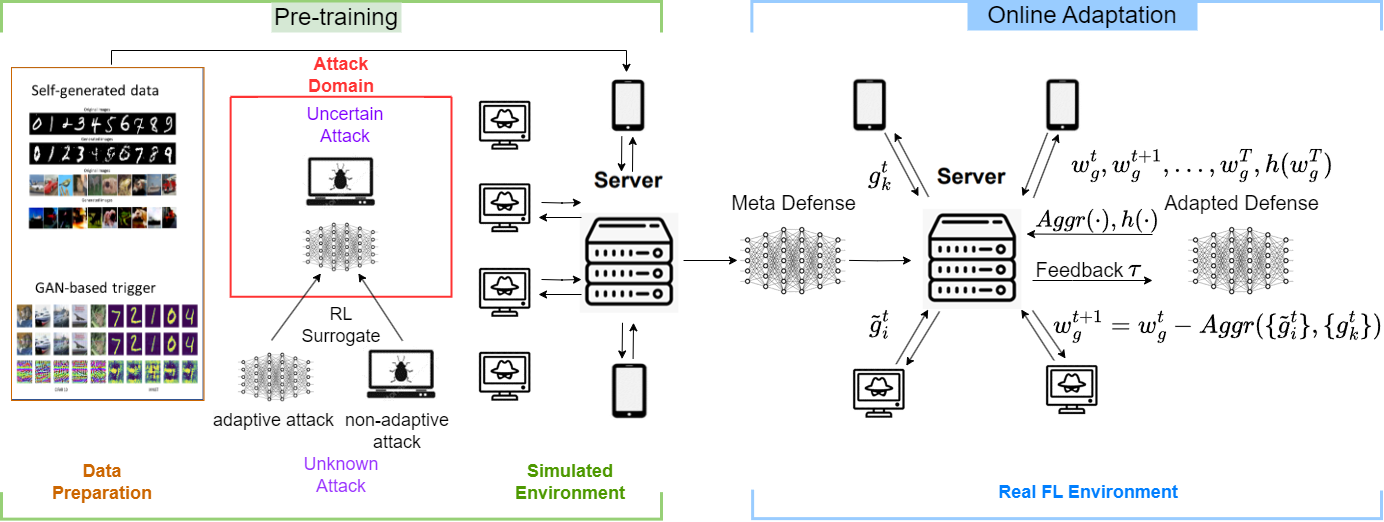}
    \caption{A graphical abstract of meta-Stackelberg defense. In the pertaining stage, a simulated environment is constructed using generated data and the attack domain. The defender utilizes meta-Stackelberg learning (\Cref{algo:meta-sl}) to obtain the meta policy to be online adapted in the real FL.}
    \label{fig:fl-game}
    \vspace{-0.6cm}
\end{figure}
\paragraph{FL objective}  Consider a learning system that includes one server and $n$ clients, each client possesses its own private dataset $D_i={(x_i^j,y_i^j)_{j=1}^{|D_i|}}$ where $|D_i|$ is the size of the dataset for the $i$-th client.
Let $U=\{D_1, D_2, \dots, D_n\}$ denote the collection of all client datasets. The objective of federated learning is to obtain a model $w$ that minimizes the average loss across all the devices: $\min_w {F(w)}:=\frac{1}{n}\sum_{i=1}^{n}f(w, D_i)$, where $f(w, D_i):=\frac{1}{|D_i|}\sum_{j=1}^{|D_i|}\ell(w,(x_i^j,y_i^j))$ is the local empirical loss with $\ell(\cdot,\cdot)$ being the loss function. 

\paragraph{Attack objective} We consider two major categories of attacks: untargeted model poisoning attacks and backdoor attacks. An untargeted model poisoning attack aims to maximize the average model loss, i.e., $\min_w -F(w)$, while a targeted one strives to cause misclassification of poisoned test inputs to one or more target labels (e.g., backdoor attacks). A malicious client $i$ employing targeted attack first produces a poisoned dataset $D'_{i}$ by altering a subset of data samples $(x_i^j,y_i^j) \in D_i$ to $(\hat{x}^{j}_i,c^*)$. Here, $\hat{x}^{j}_i$ is the tainted sample with a backdoor trigger inserted, and $c^* \neq y^j_i, c^*\in C$ is the targeted label. Let $\rho_i=|D'_i|/|D_i|$ denote the poisoning ratio, which is typically unknown to the defender. To simplify the notation, we assume that among the $M = M_1 + M_2$ malicious clients, the first $M_1$ malicious clients carry out targeted attacks, and the following $M_2$ malicious clients undertake an untargeted attack. Note that clients in the same category may use different attack methods. Then, the joint objective of these malicious clients is $\min_{w}F'(w):=\frac{1}{M_1}\sum_{i=1}^{M_1}f(w,D_i')-\frac{1}{M_2}\sum_{i=M_1+1}^{M}f(w, D_i)$.

\paragraph{FL process} At each round $t$ out of $H$ rounds {of FL training}, the server randomly selects a subset of clients $\mathcal{S}^t$ and sends them the most recent global model $w_g^t$. Every benign client $k \in \mathcal{S}^t$ updates the model using their local data via one or more iterations of stochastic gradient descent and returns the model update $g^t_k$ to the server. 
In contrast, an adversary $j \in \mathcal{S}^t$ creates a malicious model update $\widetilde{g}^t_j$ and sends it back. The server then collects the set of model updates $\{{\widetilde{g}_i^t}\cup {\widetilde{g}_j^t}\cup {g_k^t}\}_{i,j,k\in \mathcal{S}^t}$, for $i\in \{1, \ldots, M_1\}, j\in \{M_1+1, \ldots, M\}, k\in \mathcal{S}^t\setminus\{1,\ldots, M \}$, utilizes an aggregation rule $Aggr$ to combine them, and updates the global model with the learning rate $\eta^t$: $w_g^{t+1}=w_g^t-\eta^t Aggr({\widetilde{g}_i^t}\cup {\widetilde{g}_j^t}\cup {g_k^t})$, which is then sent to clients in round $t+1$. 
At the end of each round, the defender performs a post-training defense $h(\cdot)$ on the global model $\widehat{w}_g^t=h(w_g^t)$ 
to evaluate the current defense performance. Only at the final round $H$ or whenever a client is leaving the FL systems, the global model with post-training defense $\widehat{w}_g^t$ will be sent to all (leaving) clients.

\paragraph{Attack types} We introduce the concept of \emph{attack type} to differentiate various attack scenarios, which typically include the following three aspects. The first aspect is the attack objective chosen by a malicious client. Let $\Omega_1$ be the set of all possible attack objectives from the defender's knowledge base. We set $\Omega_1=\{\text{untargeted, targeted}\}$ in this work. The second aspect specifies the attack method (i.e., the algorithm used to generate the actual attack policy such as IPM and DBA) adopted by a malicious client. Let $\Omega_2$ be the set of all possible attack methods from the defender's knowledge base. The third aspect captures the configuration associated with an attack method, including its hyperparameters and other attributes (e.g., 
triggers implanted in backdoor attacks, labels used in targeted attacks, and attacker's knowledge about the FL system). Let $\Omega_3$ denote the set of all possible configurations. For each malicious client $i$, the tuple
$(\omega_1, \omega_2, \omega_3)_i$ specifies its particular attack type. Let $\xi=\{(\omega_1, \omega_2, \omega_3)_i\}_{i=1}^M$ be the joint attack type. The following refers to $\xi\in (\Omega_1\times\Omega_2\times\Omega_3)^M$ as the attack type in the FL process. Further, let $\Xi\subset(\Omega_1\times\Omega_2\times\Omega_3)^M$ denote the domain of attacks the defender is aware of. \cref{table:attacks} in Appendix~\ref{app:exp-setup} summarizes the types of all the attacks considered in this work. However, the actual attack type encountered during FL training is not necessary in $\Xi$, although it is presumably similar to a known type in $\Xi$.

\subsection{Pre-training as Bayesian Stackelberg Markov Game}
\label{subsec:bsmg}
From the discussion above, the global model updates and the final output are jointly influenced by the defender (through aggregation) and the malicious clients (through corrupted gradients). Hence, the FL process in an adversarial environment can be formulated as a two-player discrete-time Bayesian Stackelberg Markov game (BSMG) defined by a tuple $\langle {S}, {A}_\D, {A}_{\xi},\mathcal{T}, r, \gamma, H \rangle$. Using discrete time index $t$ ({one step corresponds to one FL round}), we have the following.   
\begin{itemize}
    \item $S$ is the state space, and its elements represent the global model at each round $s^t=w_g^t$.  
    \item $A_{\D}$ is the defender's action set. Each action $a_\D^t$ represents a combination of the robust aggregation and post-training defenses: $a_\D^t=\{Aggr(\cdot), h(\cdot)\}$. 
    \item $A_{\xi}$ is the type-$\xi$ attacker's action set. Each action includes the joint model updates of all malicious clients: $a_\A^t=\{\widetilde{g}_i^t\}_{i=1}^{M_1}\cup \{\widetilde{g}_i^t\}_{i=M_1+1}^{M}$. 
    \item 
    {$\mathcal{T}(s^{t+1}|s^t, Aggr(\cdot), a^t_\A)$} {specifies the distribution of the next state given the current state and joint actions at $t$, which is determined by the global model update:}
    $w_g^{t+1}=w_g^t-\eta^t Aggr({\widetilde{g}_i^t}\cup {\widetilde{g}_j^t}\cup {g_k^t})$. 
    \item $r_\D, r_{\xi}$ are the defender's and the attacker's reward functions (to be maximized), respectively. The defender aims to minimize the loss after the post-training: $r_{\mathcal{D}}^t:= -F(\widehat{w}^{t}_g)$ where $\widehat{w}^{t}_g=h({w}^{t}_g)$. The attacker's $r_\xi^t$ is given by the joint attack objective: $-F'(\widehat{w}^{t}_g)$. 
\end{itemize}

\begin{remark}
Even though the defender's reward evaluation considers a post-training defense applied to each step, such a defense is actually executed only at the final round or to a client leaving the FL system. The key message is that the post-training defense $h(\cdot)$ in defense actions do not interfere with the model updates on $w_g^t$, since the transition function $\mathcal{T}$ does not involve $h(\cdot)$. Compared with existing reward designs that only focus on the last round model accuracy \cite{li2022learning}, our reward design prioritizes practical implementation and long-term defense performance, where clients can join and leave the FL system anytime before the final round. {This design enables us to combine a post-training defense along with techniques for modifying the model structure, e.g., NeuroClip~\cite{wang2022universal} and Prun~\cite{wu2020mitigating}.}
\end{remark}
The Stackelberg interactions among players are deferred to \Cref{sec:meta-sl}, while the rest of this section presents an overview of the pre-training and online adaptation stages. We summarize the frequently used notations in \Cref{tab:notations}.
\begin{table}[!ht]
\footnotesize
    \centering
    \begin{tabular}{ll}
    \toprule
      Notation(s)   & Description  \\
      \midrule
        $w_g^t$, $\hat{w}_g^t$ & Global model weights, post-training-defense weights \\
        $\mathcal{D}$, $\mathcal{A}$ & Defender, Attacker \\
        $\xi$, $\Xi$ & Attack type, attack domain \\
        $F$, $F'$, $F''$ & FL, attack, and approximated attack objectives  \\
        $a_\D^t$, $a_\A^t$ & Defense and attack actions \\
        $\mathcal{T}$ & Transition, i.e., global model update  \\
        $r_\D$, $r_\xi$ & Defender's and type-$\xi$ Attacker's rewards \\
        $\pi_\D$, $\pi_\xi$ & Defender's and type-$\xi$ attacker's policies \\
        $\theta$, $\Theta$ & Defender's policy parameter, and the domain  \\
        $\phi$, $\Phi$ & Generic Attacker's parameter, and the domain\\
        $\phi_\xi$, $\phi_\xi^*$ & Type-$\xi$ Attacker's parameter, and the optimal attack \\
        $Q(\Xi)$ & Prior distribution over the attack domain  \\
        $J_\D(\theta, \phi, \xi)$ & Expected cumulative defense rewards under type-$\xi$ attack \\ 
        $J_\A(\theta, \phi, \xi)$ & Expected cumulative type-$\xi$ attack rewards \\
        $\tau_\xi$ & FL system trajectory  under type-$\xi$ attack\\
        $q(\theta,\phi_\xi)$ & Trajectory distribution under type-$\xi$ attack  \\
        $d_i$ & Residue factors of $q(\theta,\xi_i)$ after removing $\pi_{\xi_i}$\\
        $\nabla_\theta J_\D(\tau)$ & Estimated gradient using trajectory $\tau$ \\
        $\mathcal{L}_{\D}(\theta,\phi,\xi)$ & Expected rewards after gradient adaptation on $\theta$ \\
        $\mathcal{L}_{\A}(\theta,\phi,\xi)$ & Expected type-$\xi$ rewards after gradient adaptation \\
        $V(\theta)$ & Worst-case expected defense rewards over all attack types  \\
        $\hat{V}(\theta)$ & Sample average of $V$ w.r.t. sampled attack types \\
        \bottomrule   
    \end{tabular}
    \caption{A summary of frequently used notations.}
    \label{tab:notations}
\end{table}

\subsection{Simulated Pre-training Environments}
\label{subsec:sim}
With the game model defined above, the defender (i.e., the server) can, in principle, identify a strong defense by solving the game (we discuss different solution concepts in Section~\ref{sec:meta-sl}). Due to efficiency and privacy concerns in FL, however, it is often infeasible to solve the game in real time when facing the actual attacker. Instead, the defender can create a simulated environment to approximate the actual FL system during the pre-training stage. The main challenge, however, is that the defender lacks information about the individual devices in FL. 

\paragraph{White-box simulation} We first consider the {\it white-box} setting where the defender is aware of the number of malicious devices in each category (i.e., $M_1$ and $M_2$) and their actual attack types, as well as the \textit{non-i.i.d.} level (to be defined in~\cref{sec:exp-setup}) of local data distributions across devices. 
However, it does not have access to individual devices' local data and random seeds, making it difficult to simulate clients' local training and evaluate rewards. To this end, we assume that the server has a small amount of root data randomly sampled from the collection of all client datasets $U$ as in previous work~\cite{cao2020FLTrust, miao2022privacy}. We then use generative model (e.g., conditional GAN model~\cite{mirza2014conditional} for MNIST and diffusion model~\cite{sohl2015deep} for CIFAR-10 in our experiments) to generate as much data as necessary to mimic the local training (see details in Appendix~\ref{app:self-data}). We give an ablation study (\cref{table:5}) in Appendix~\ref{app:add-exp} to evaluate the influence of limited/biased root data. We remark that the purpose of pre-training is to derive a defense policy rather than the model itself. Directly using the shifted data (root or generated) to train the FL model will result in low model accuracy (see~\cref{table:3} in Appendix~\ref{app:add-exp}).

\paragraph{Black-box simulation} We then consider the more realistic {\it black-box} setting, where the defender has no access to the number of malicious devices and their actual attack types, nor the \textit{non-i.i.d.} level of local data distributions. To obtain a robust defense, we assume the server considers the worst-case scenario based on a rough estimate of the missing information (see our ablation study in Appendix~\ref{app:add-exp}) and adopts the RL-based attacks to simulate the worst-case attacks (see Section~\ref{sec:RL-attack-defense}) when the attack is unknown or adaptive. In the face of an unknown backdoor attack, the defender does not know the backdoor triggers and targeted labels. To simulate a backdoor attacker's behavior, we first implement multiple GAN-based attack models as in~\cite{doan2021lira} to generate worst-case triggers (which maximizes attack performance given the backdoor objective) in the simulated environment. Since the defender does not know the {poisoning ratio $\rho_i$} and the target label of the attacker's poisoned dataset (
needed to determine the attack objective $F'$), we {approximate} the attacker's reward function by $r_{\mathcal{A}}^t=-F''(\widehat{w}^{t+1}_g)$, where $F''(w):=\min_{c\in C} [\frac{1}{M_1}\sum_{i=1}^{M_1}\frac{1}{|D_i'|}\sum_{j=1}^{|D_i'|}\ell(w,(\hat{x}_i^j,c))]-\frac{1}{M_2}\sum_{i=M_1+1}^M f(\omega, D_i)$. $F''$ differs $F'$ only in the first $M_1$ clients, where we use a strong target label (that minimizes the expected loss) as a surrogate to the true label $c^*$. We report the defense performance against white-box and black-box backdoor attacks in Fig.~\ref{fig:backdoors} in Appendix~\ref{app:add-exp}.

\subsection{Online Adaptation and Execution}
\label{sec:online}
When deploying the pre-trained defense policy online, the defender interacts with the FL system and collects online samples, including the states (global model weights), actions (clients' local updates), and rewards information. Since the defender cannot access clients' data, the exact reward evaluation is missing. Instead, it calculates estimated rewards using the self-generated data and simulated triggers from the pertaining stage, as well as new data, inferred online through methods such as inverting gradient~\cite{geiping2020inverting} and reverse engineering~\cite{wang2019neural}. Inferred data samples are blurred using data augmentation~\cite{shorten2019survey} to protect clients' privacy.  For a fixed {number} of FL rounds ({e.g., $50$ for MNIST and $100$ for CIFAR-10 in our experiments}), the defense policy will be updated using gradient ascents from the collected samples. Ideally, the defender's adaptation time  (including the time for collecting new samples and updating the policy) should be significantly less than the whole FL training period so that the defense execution will not be delayed. In real-world FL training, the server typically waits for up to $10$ minutes before receiving responses from the clients~\cite{bonawitz2019towards,kairouz2021advances}, enabling defense policy's online update with enough episodes.
\section{Meta Stackelberg Learning}
\label{sec:meta-sl}
Since the pre-training is modeled by a BSMG, solving the game efficiently is crucial to a successful defense. This work's main contribution includes the formulation of a new solution concept to the game, meta-Stackelberg equilibrium (meta-SE), and a learning algorithm to approximate such equilibrium in finite time. To motivate the proposed concept, we begin by addressing the defense against non-adaptive attacks. 

Consider the attacker employing a non-adaptive attack of type $\xi$; in other words, the attack action at each iteration is determined by a fixed attack strategy $\pi_{\xi}$, {where $\pi_{\xi}(a)$ gives the probability of taken action $a \in A_{\xi}$, independent of the FL training and the defense strategy}.
In this case, BSMG reduces to an MDP, where the transition kernel is $\mathcal{T}_{\xi}(\cdot|s,a_\D)\triangleq \int_{A_\xi} \mathcal{T}(\cdot|s,a_\mathcal{A}, a_\mathcal{D})d\pi_{\xi}(a_\A)$.  Parameterizing the defender's policy $\pi_\D(a_\D^t|s^t;\theta)$ by a neural network with model weights $\theta\in \Theta$, the solution to the following optimization problem $\max_{\theta\in \Theta} 
\E_{a_\D^t\sim \pi_\D, s^t\sim \mathcal{T}_\xi}[\sum_{t=1}^H \gamma^t r_\D^t]\triangleq J_\D(\theta, \xi)$ 
gives the optimal defense against the non-adaptive attack. When the actual attack in the online stage falls within $\Xi$, which the defender is uncertain of, one can consider the defense against the expected attack: $\max_{\theta}\E_{\xi\sim Q}J_\D(\theta, \xi)$, where $Q$ is a distribution over the attack domain to be designed by the defender. One intuitive design is to include all reported attack methods in history as the attack domain and their empirical frequency as the $Q$ distribution. 

In stark contrast to non-adaptive attacks, an adaptive attack can adjust attack actions to the FL environment and the defense mechanism \cite{li2022learning,li2023learning}. Most existing attacks are history-independent \cite{rodriguez2023survey,xia2023poisoning}. Hence, we assume that an adaptive attack takes the current state (global model) as input, i.e., the attack policy is a Markov policy denoted by $\pi_{\xi}(a^t_\A|s^t;\phi)$, which is parameterized by $\phi\in \Phi$.   An optimal adaptive attack policy is the best response to the existing defense $\pi_\D(\cdot|s^t;\theta)$: $\phi_\xi^*\in \argmax \E_{a_\A^t\sim \pi_\xi, a_\D^t\sim \pi_\D}[\sum_{t=1}^H \gamma^t r_\xi^t]\triangleq J_\A(\theta,\phi,\xi)$. Then, the defender's cumulative rewards under such attack is $J_\D(\theta, \phi_\xi^*, \xi)\triangleq \E_{a_\A^t\sim \pi_\xi, a_\D^t\sim \pi_\D}[\sum_{t=1}^H \gamma^t r_\D^t]$. 

\subsection{RL-based Attacks and Defenses}
\label{sec:RL-attack-defense}
The actual attack type (which could be either adaptive or non-adaptive) encountered in the online phase may be not in $\Xi$ and thus unknown to the defender. 
To prepare for these unknown attacks, we propose to use multiple RL-based attacks with different objectives, adapted from RL-based untargeted model poising attack~\cite{li2022learning} and RL-based backdoor attack~\cite{li2023learning}, as surrogates for unknown attacks, which are added to the attack domain for pre-training. The rationale behind the RL surrogates includes: (1) they achieve strong attack performance by optimizing long-term objectives, {which is typically more general than myopic attacks with short-term goals}; (2) they adopt the most general action space (i.e., model updates), which allows them to mimic any adaptive or non-adaptive attacks given the corresponding objectives; (3) they are flexible enough to incorporate multiple attack methods by using RL to tune the hyper-parameters of a mixture of attacks. A similar argument applies to RL-based defenses. We remark that in this paper, an RL-based attack (defense) is not a single attack (defense) as in~\cite{li2022learning,li2023learning} but a systematically synthesized combination of existing attacks (defenses). In the simulated environment, we train our defense against the strongest white-box RL attacks in~\cite{li2022learning,li2023learning} with different objectives (e.g., untargeted or targeted), which is considered the optimal attack strategy. The ``worst-case'' scenario is commonly used in security scenarios to ensure the associated defense has performance guarantees under ``weaker'' attacks with similar objectives. Such a robust defense policy gives us a good starting point to further adapt to uncertain or unknown attacks. Our defense is generalizable to other adaptive attacks (see \cref{table:6} in Appendix~\ref{app:add-exp}). The key novelty of our RL-based defense is that instead of using a fixed and hand-crafted algorithm as in existing approaches, we use RL to optimize the policy network $\pi_\D(a_\D^t|s^t;\theta)$. Similar to RL-based attacks, the most general action space could be the set of global model parameters. However, the high dimensional action space will lead to an extremely large search space that is prohibitive in terms of training time and memory space. Thus, we apply compression techniques (see Appendix~\ref{app:exp-setup}) to reduce the action from a high-dimensional space to a 3-dimensional space {incorporating robust aggregation and post-training defenses}. Note that the execution of our defense policy is lightweight, without using any extra data for evaluation/validation. See the discussion in Appendix~\ref{app:exp-setup} on how we apply our RL-based defense during online adaptation.

\subsection{Meta-Stackelberg Equilibrium}
\label{subsec:meta-se}
{As discussed in Section~\ref{subsec:bsmg}, one of the key challenges to solving the BSMG} is the defender's incomplete information on attack types. Prior works have explored a Bayesian equilibrium approach to address this issue \cite{sengupta20bsmg}. Given the 
{set of possible attacks $\Xi$ that the defender is aware of} and a prior distribution $Q$ over the domain, the Bayesian Stackelberg equilibrium (BSE) is given by the following bi-level optimization.
\begin{definition}[Bayesian Stackelberg equilibrium]
\label{def:bse}
A pair of the defender's policy $\theta$ and the attacker's type-dependent policy $(\phi_\xi)_{\xi\in\Xi}$ is a Bayesian Stackelberg equilibrium if it satisfies
    \begin{equation}
\label{eq:bse}
    \max_{\theta\in \Theta} \E_{\xi\sim Q}[J_\D(\theta, \phi_\xi^*, \xi)], \text{ s.t. } \phi_\xi^*\in \argmax J_\A(\theta, \phi, \xi). \tag{BSE}
\end{equation}
\end{definition}
In (\ref{eq:bse}), unaware of the exact attacker type, the defender {(the leader)} aims to maximize the defense performance against an average of all attack types, anticipating their best responses. 

 From a game-theoretic viewpoint, the Bayesian equilibrium in (\ref{eq:bse}) is of ex-ante. The defender determines its equilibrium strategy only knowing the type distribution $Q$. However, as the Markov game proceeds, the attacker's moves (e.g., malicious global model updates) during the interim stage (online stage) reveal additional information on the attacker's private type. This Bayesian equilibrium defense strategy fails to handle the emerging information on the attacker's hidden type in the interim stage, as the policy obtained from (\ref{eq:bse}) remains fixed throughout the online stage without adaptation.

To address the limitation of Bayesian equilibrium, we introduce the {novel} solution concept, meta-Stackelberg equilibrium (meta-SE),  to equip the defender with online responsive intelligence under incomplete information. As a synthesis of meta-learning and Stackelberg equilibrium, the meta-SE aims to pre-train a meta policy on a variety of attack types sampled from the attack domain $\Xi$ such that online gradient adaption applied to the base produces a decent defense against the actual attack in the online environment. Using mathematical terms, we denote by $\tau_\xi:=  (s^k, a^k_\D, a^k_{\xi})_{k=1}^H$ the trajectory of the FL system under type-$\xi$ attacker {up to round $H$}, which is subject to the distribution $q(\theta, \phi_\xi):=\prod_{t=1}^H \pi_\D(a_\D^t|s^t;\theta)\pi_\xi(a^t_\A|s^t, \phi_\xi)\mathcal{T}(s^{t+1}|s^t, a_\D^t, a_\A^t)$. Let $\hat{\nabla}_\theta J_\D(\tau)$ be the unbiased estimate of the policy gradient $\nabla_\theta J_\D$ using the sample trajectory $\tau_\xi$ (see Appendix~\ref{app:algo}). Then, a one-step gradient adaptation using the sample trajectory is given by $\theta+\eta \hat{\nabla}_\theta J_\D$. Incorporating this gradient adaptation into (\ref{eq:bse}) leads to the proposed meta-SE.   
\begin{definition}[Meta-Stackelberg Equilibrium]
\label{def:meta-se}
    A pair of the defender's policy $\theta$ and the attacker's type-dependent policy $(\phi_\xi)_{\xi\in\Xi}$ is a one-step gradient-based meta-Stackelberg equilibrium if it satisfies
    \begin{align}
 \label{eq:meta-se}
    \max_{\theta\in\Theta} & \mathbb{E}_{{\xi}\sim Q(\Xi)}\mathbb{E}_{\tau\sim q}[J_\D(\theta+\eta \hat{\nabla}_\theta J_\D(\tau),\phi^*_{\xi}, {\xi})], \tag{meta-SE}\\
    \text{s.t. }&  \phi^*_{ {\xi}}\in \argmax \mathbb{E}_{\tau\sim q}J_\A(\theta+\eta \hat{\nabla}_\theta J_\D(\tau), \phi,  {\xi}), \forall \xi\in \Xi. \nonumber
\end{align}
\end{definition}
\begin{remark}
    The meta-SE is open to various online adaptation schemes, such as multi-step gradient \cite{nichol2018first} recurrent neural network-based adaptation \cite{duan2016rl}. Our experiments implement multi-step gradient adaptation due to simplicity; see \Cref{algo:meta-rl} in Appendix~\ref{app:algo} and online adaptation setup in Appendix~\ref{app:exp-setup}. 
\end{remark}

The idea of adding the gradient adaptation to the equilibrium is inspired by the recent developments in gradient-based meta-learning \cite{finn2017model,nichol2018first}. When the attack is non-adaptive, the BSMG reduces to an MDP problem, as delineated at the beginning of this section. Consequently, (\ref{eq:meta-se}) turns into the standard form of meta-learning \cite{finn2017model}.   Unlike the conventional (\ref{eq:bse}), the solution to (\ref{eq:meta-se}) gives the defender a decent defense initialization after pre-training whose gradient adaptation in the online stage is tailored to type $\xi$, since the online trajectory follows the distribution $q(\theta,\phi_\xi)$ that contains information on the attack type. The novelty of (\ref{eq:meta-se}) lies in that the leader (defender) determines an optimal adaptation scheme rather than a policy, which is computed using an online trajectory without knowing the actual type, creating a data-driven strategic adaptation after the pre-training. 

\subsection{Meta-Stackelberg Learning}
\label{subsec:meta-sl}
Unlike finite Stackelberg Markov games that can be solved (approximately) using mixed-integer programming  \cite{Vorobeychik_Singh_2021}, two-stage bilinear programming \cite{tao23pot} or Q-learning \cite{sengupta20bsmg}, our BSMG admits high-dimensional continuous state and action spaces, posing a more challenging computation issue. Hence, we resort to a two-timescale policy gradient (PG) algorithm, referred to as meta-Stackelberg learning (meta-SL) presented in \Cref{algo:meta-sl}, to solve for (\ref{eq:meta-se}) in a similar vein to \cite{li2022sampling, hong23twotime}, which alleviates the nonstationarity caused by concurrent policy updates from both players \cite{bora23two-timescale, tao_info, tao22confluence}. As shown in the pseudo-code, meta-SL features a nested-loop structure, where the inner loop (line 13-15) learns the attacker's best response for each sampled type defined in the constraint in (\ref{eq:meta-se}) while fixing the current defense at the $t$-th outer loop. Once the inner loop terminates after $N_\A$ rounds, the returned attack policy $\phi_\xi^t(N_\A)$, as an approximate to $\phi_\xi^*$, is utilized to estimate the policy gradient of the defender's value function. Of particular note is that when evaluating the defender's policy gradient under a given type $\xi$, the gradient computation $\nabla_\theta \E_{\tau\sim q(\theta)}[J_\D(\theta+\eta \hat{\nabla}_\theta J_\D(\tau), \phi_\xi^*, \xi )]$ involves the Hessian computation due to $\hat{\nabla}_\theta J_\D(\tau)$. Even though \cite{fallah2021convergence} gives an unbiased sample estimate of the policy gradient, leading to debiased meta-learning, the sample complexity induced by the Hessian is prohibitive. To avoid Hessian estimation, we adopt another meta-learning scheme called Reptile \cite{nichol2018first} to update the defense policy. The key difference is that Reptile directly evaluates the policy gradient at the adapted policy $\theta_\xi^t$ (line 11) instead of the current meta policy $\theta^t$. We provide a step-by-step derivation of debiased meta-learning in Appendix~\ref{app:algo}. 
\begin{algorithm}[H]
\footnotesize
\caption{Meta-Stackelberg Learning}
\label{algo:meta-sl}
    \begin{algorithmic}[1]
    \STATE \textbf{Input: } the distribution $Q(\Xi)$, initial defense meta policy $\theta^0$, pre-defined attack methods $\{\pi_\xi\}_{\xi\in \Xi}$,  pre-trained RL attack policies $\{\phi_\xi^{0}\}_{\xi\in {\Xi}}$, step size parameters $\kappa_\D$, $\kappa_\A$, $\eta$, and iterations numbers $N_\A, N_\D$;
    \STATE \textbf{Output: }$\theta^{N_\D}$;
     \FOR{iteration $t=0$ to $N_\D-1$} 
     \IF{\texttt{meta-RL} (for non-adaptive)}  
     \STATE Sample a batch of $K$ attack types $\xi$ from  ${\Xi}$;
     \STATE Estimate $\hat{\nabla}J_D(\xi):=\hat{\nabla}_\theta J_\D(\theta, \pi_\xi, \xi)|_{\theta=\theta_\xi^t}$;
     \ENDIF
     \IF{\texttt{meta-SG} } 
     \STATE Sample a batch of $K$ attack types $\xi\in {\Xi}$;
     \FOR{each sampled attack $\xi$}
     \STATE Apply one-step adaptation \\
     $\theta_\xi^t\gets \theta^t+\eta \hat{\nabla}_\theta J_\D(\theta^t, \phi^t_\xi, \xi)$;
     \STATE $\phi^t_\xi(0)\gets \phi^t_\xi$;
     \FOR{iteration $k=0,\ldots, N_\A-1$}
     \STATE $\phi^{t}_\xi(k+1)\gets \phi^{t}_\xi(k)+\kappa_\A \hat{\nabla}_{\phi} J_\A(\theta_\xi^t, \phi_\xi^t(k), \xi)$;
     \ENDFOR
     \IF{\texttt{Reptile}}
     \STATE $\hat{\nabla}J_\D(\xi)\gets\hat{\nabla}_\theta J_\D(\theta, \phi^t_\xi(N_\A), \xi)|_{\theta=\theta_\xi^t}$;
     \ENDIF
     \IF{\texttt{Debiased}}
     \STATE $\hat{\nabla}J_\D(\xi)\gets \hat{\nabla}_\theta J_\D(\theta+\eta \hat{\nabla}_\theta J_\D, \phi^t_\xi(N_\A), \xi)|_{\theta=\theta^t}$;
     \ENDIF
     \ENDFOR
    \ENDIF
     \STATE $\theta^{t+1}\gets \theta^t+ \kappa_\D/K \sum_{\xi}\hat{\nabla}J_\D(\xi)$
     \ENDFOR
\end{algorithmic}
\end{algorithm}

The rest of this subsection addresses the computational expense of the proposed meta-SL under debiased meta-learning from a theoretical perspective. 
We begin with the definition of two quantities, $\LL_\D (\theta, \phi, \xi) \triangleq \E_{\tau \sim q} J_{\D}(\theta+\eta \hat{\nabla}_{\theta} J_\D (\tau), \phi, \xi) $, and $\LL_\A (\theta, \phi, \xi) \triangleq  \E_{\tau \sim q} J_\A ( \theta + \hat{\nabla}_\theta J_\D(\tau), \phi, \xi )$, for any fixed type $\xi \in \Xi$. 
We highlight the strict competitiveness (\Cref{ass:sc}) and continuity/smoothness (\Cref{asslip}) of these two quantities. These properties allow us to formalize a slightly weaker solution concept in \Cref{def:meta-fose}.
\begin{assumption}[Strict-Competitiveness]
\label{ass:sc}
    The BSMG is strictly competitive, i.e., there exist constants $c<0$, $d$ such that $\forall \xi \in \Xi$, $s \in S$, $a_\D, a_\A \in A_\D \times A_\xi$, $r_\D(s,a_\D, a_\A)=c \cdot r_\A(s,a_\D, a_\A)+d$.
\end{assumption}
The notion of strict competitiveness (SC) can be treated as a generalization of zero-sum games: if one joint action $(a_\D, a_\A)$ leads to payoff increases for one player, it must decrease the other's payoff.
In adversarial FL, the untargeted attack naturally makes the game zero-sum (hence, SC). {The purpose of introducing \Cref{ass:sc} is to establish the Danskin-type result \cite{danskin-type} for the Stackelberg game with nonconvex value functions (see \Cref{liplemma} in Appendix~\ref{app:theory}), which spares us from the Hessian inversion appeared in implicit function theorem (see \Cref{lemma:ift} in Appendix~\ref{app:theory}). More specifically, it enables us to estimate the gradients of value function $ V(\theta) :=  \mathbb{E}_{\xi \sim Q, \tau \sim q} J_\D (\theta + \eta \hat{\nabla}_{\theta} J_{\D}(\tau), \phi_\xi, \xi) $, where $\{ \phi_\xi: \phi_\xi \in \arg\max_{\phi}  \LL_{\A} (\theta, \phi, \xi) \}_{\xi \in \Xi}$, without considering the second-order information.

 \begin{assumption}[type-wise Lipschitz] \label{asslip}
 The functions $\mathcal{L}_{\D}$ and $\LL_\A$ are continuously diffrentiable in both $\theta$ and $\phi$. Furthermore, there exists constants $L_{11}$, $L_{12}, L_{21}$, and $L_{22}$ such that
 for all $\theta, \theta_1, \theta_2 \in \Theta$ and $\phi, \phi_1, \phi_2 \in \Phi$, and for any $\xi \in \Xi$,
\begin{align*}
    \left\|\nabla_{\theta} \mathcal{L}_\D  \left(\theta_{1}, \phi, \xi\right)-\nabla_{\theta} \mathcal{L}_\D\left(\theta_{2}, \phi, \xi\right)\right\| &  \leq L_{11}\left\|\theta_{1}-\theta_{2}\right\|,  
    \\ 
    \left\|\nabla_{\phi} \mathcal{L}_\D \left(\theta, \phi_{1}, \xi\right)-\nabla_{\phi} \mathcal{L}_\D \left(\theta, \phi_{2}, \xi\right)\right\| & \leq L_{22}\left\|\phi_{1}-\phi_{2}\right\|, 
     \\ 
     \left\|\nabla_{\theta} \mathcal{L}_\D \left(\theta, \phi_{1}, \xi\right)-\nabla_{\theta} \mathcal{L}_\D \left(\theta, \phi_{2}, \xi\right)\right\| & \leq L_{12}\left\|\phi_{1}-\phi_{2}\right\|, 
    \\ 
    \left\|\nabla_{\phi} \mathcal{L}_\D \left(\theta_{1}, \phi, \xi\right)-\nabla_{\phi} \mathcal{L}_\D \left(\theta_{2}, \phi, \xi\right)\right\|  & \leq L_{21}\left\|\theta_{1}-\theta_{2}\right\|,   
    \\ 
    \| \nabla_{\phi} \LL_\A (\theta, \phi_1, \xi ) - \nabla_{\phi} \LL_\A (\theta, \phi_2, \xi)\| & \leq  L_{21} \| \phi_1 - \phi_2\|,  
    \\ 
    \| \nabla_{\phi} \LL_\A (\theta_1, \phi, \xi ) - \nabla_{\phi} \LL_\A (\theta_2, \phi, \xi)\| & \leq  L_{21} \| \theta_1 - \theta_2\|.
\end{align*}
\end{assumption}

\begin{definition}[First-order Equilibrium]
 \label{def:meta-fose}
For $\varepsilon \in [0,1)$, a pair $(\theta^*,\{\phi^*_\xi\}_{\xi \in \Xi}) \in \Theta \times \Phi^{|\Xi|}$ is a $\varepsilon$-\textit{meta First-Order Stackelbeg Equilibrium} ($\varepsilon$-meta-FOSE) if it satisfies that for  $\xi \in \Xi$, $\max_{\theta \in B(\theta^*)} \langle \nabla_{\theta} \LL_\D (\theta^*, \phi^*_\xi, \xi), \theta - \theta^*\rangle \leq \varepsilon$, $ \max_{\phi \in  B(\phi^*_\xi )} \langle \nabla_{\phi} \LL_\A (\theta^*, \phi^*_\xi, \xi) , \phi - \phi^*_\xi \rangle \leq \varepsilon$, $B( \theta^* )=  \{ \theta \in \Theta : \| \theta - \theta^*\| \leq 1\}$, and  $B( \phi^*_\xi ) =  \{ \phi \in \Phi : \| \phi - \phi^*_\xi \| \leq 1\} $, when $\varepsilon = 0$, it is called a meta-FOSE.
 \end{definition}
\Cref{def:meta-fose} constitutes a necessary equilibrium condition for \ref{eq:meta-se}), which can be reduced to $\| \nabla_{\theta} \LL_\D (\theta^*, \phi_\xi, \xi)\| \leq \varepsilon$ and $\|\nabla_{\phi} \LL_\A (\theta^*, \phi_\xi, \xi)\| \leq \varepsilon$ in the unconstraint settings since the ball radius is set to 1.  While omitting the second-order conditions, in the strictly competitive setting, $\varepsilon$-meta-FOSE is a more reasonable focal point, (see references \cite{nouiehed2019solving,pang2016unified}.) as its existence is guaranteed by \Cref{thm:existence}.
\begin{theorem}
\label{thm:existence}
 When $\Theta$ and $\Phi$ are compact and convex, there exists at least one meta-FOSE.
\end{theorem}

Our convergence analysis is based on a regularity assumption adapted from the Polyak-Łojasiewicz (PL) condition \cite{karimi2016linear}. PL condition is a much weaker alternative to convexity conditions (e.g., essential/weak/restricted convexity) \cite{karimi2016linear}, which is customary in nonconvex analysis. 
Despite the lack of theoretical justifications for the PL condition in the literature, \cite{li2022sampling} empirically demonstrates that the cumulative rewards in meta-reinforcement learning satisfy the PL condition.

\begin{assumption}[Stackelberg Polyak-Łojasiewicz condition] \label{plass}
There exists a positive constant $\mu$ such that for any $(\theta, \phi) \in \Theta\times \Phi$ and $\xi \in \Xi$, the following inequalities hold: $\frac{1}{2\mu} \|\nabla_{\phi } \mathcal{L}_{\D} (\theta, \phi, \xi)\|^2 \geq  \max_{\phi} \LL_\D (\theta, \phi, \xi) -  \mathcal{L}_{\D}(\theta, \phi, \xi)$, $\
    \frac{1}{2\mu} \|\nabla_{\phi } \LL_{\A} (\theta, \phi, \xi)\|^2 \geq 
   \max_{\phi} \LL_\A (\theta, \phi, \xi) -  \LL_\A(\theta, \phi, \xi) $.
\end{assumption}

To analyze the algorithmic performance, we require some standard assumptions on batch reinforcement learning, along with some additional information about the parameter space and function structure, which ensures that the approximation error induced by inner loops is decreasing. These assumptions, commonly used in the literature \cite{fallah2021convergence}, are all stated in \Cref{ass:grad}.

\begin{assumption}
\label{ass:grad}
 The following holds true throughout the progression of \Cref{algo:meta-sl}:
\begin{enumerate}
    \item The compact  space $\Theta$ has diameter bounded by $D_\Theta \geq \sup_{\theta_1, \theta_2  \in \Theta} \|\theta_1 - \theta_2\|$; the initialization $\theta^0$ admits at most $D_V$ function gap, i.e., $D_V:= \max_{\theta \in \Theta} V(\theta) - V(\theta^0)$.
     \item The following relation holds: $  0 < \mu <  -cL_{22}$. 
  \item For any $\theta, \phi$ and attacker type $\xi \in \Xi$, the stochastic policy gradient estimators are bounded, unbiased (for attacker), with $\frac{\sigma^2}{N_b}$ bounded variances, i.e., 
  \begin{equation*}
  \begin{aligned}
         \|\nabla_\theta J_\D (\theta, \phi, \xi) \|^2 \leq G^2, & \quad \|\nabla_\phi J_\A (\theta, \phi, \xi) \|^2 \leq G^2, \\
          \E [\hat{\nabla}_{\phi} J_\A(\theta^t, \phi^t_\xi, \xi) & - \nabla_{\phi} J_\A(\theta^t, \phi^t_\xi, \xi)] = 0, \\
         \E [   \| \hat{\nabla}_{\phi} J_\A(\theta^t, \phi^t_\xi, \xi)  - & \nabla_{\phi}  J_\A(\theta^t, \phi^t_\xi, \xi) \|^2 ] \leq  \frac{\sigma^2}{N_b},  \\
     \E [   \| \hat{\nabla}_{\theta} J_\D(\theta^t, \phi^t_\xi, \xi) - & \nabla_{\theta} J_\D(\theta^t, \phi^t_\xi, \xi) \|^2 ] \leq  \frac{\sigma^2}{N_b}.
  \end{aligned}
  \end{equation*}

\end{enumerate}
\end{assumption}

\begin{theorem}\label{thm:main}
     Under assumptions~\ref{ass:sc}, \ref{asslip}, \ref{plass}, and \ref{ass:grad} for any given $\varepsilon \in (0,1)$, let the learning rates $\kappa_\A = \frac{1}{L_{22}}$ and $\kappa_\D = \frac{1}{L}$, $\rho = 1 + \frac{\mu }{c L_{22}} \in (0, 1)$, $L = L_{11} + \frac{L_{12}L_{21}}{\mu}$, $ \bar{L} = \max \{ L_{11}, L_{12}, L_{22}, L_{21}, V_{\infty} \}$ where $V_{\infty} :=  \max\{ \max\|\nabla V(\theta)\|, 1 \}$; let the batch size and inner-loop iteration size be properly chosen,  
     \begin{equation*}
        \begin{aligned}
             N_\A & \geq \frac{1}{\log \rho^{-1}}\log \frac{32 D_V^2 (2V_{\infty} + LD_{\Theta})^4 \bar{L} |c|G^2    }{ L^2 \mu^2\varepsilon^4},  \\
             N_b  & \geq \frac{32 \mu L_{21}^2 D_V^2 ( 2 V_{\infty} +  L D_{\Theta} )^4}{ |c| L_{22}^2 \sigma^2  \bar{L} L\varepsilon^4} ;
        \end{aligned}
     \end{equation*}
      then, \Cref{algo:meta-sl} finds a $\varepsilon$-meta-FOSE within $N_\D$ iterations in expectation, where explicitly, 
     \begin{equation*}
          N_\D \geq \frac{ 4D_V (2 V_{\infty} + LD_{\Theta})^2}{L\varepsilon^2 } .
     \end{equation*}
    which leads to the sample complexity $N_\A\sim \mathcal{O}(\log\epsilon^{-1})$, $N_b\sim \mathcal{O}(\epsilon^{-4})$, and $N_\D \sim \mathcal{O}(\varepsilon^{-2})$.
\end{theorem}
Finally, we conclude this section by analyzing the meta-SG defense's generalization ability when the learned meta policy is exposed to attacks unseen in the pre-training. \Cref{main_prop:generalization} asserts that meta-SG is generalizable to the unseen attacks, given that the unseen is not distant from those seen. {To formalize the generalization error, let the fixed attack policies $\phi_i$, $i=1, \ldots, m+1$ corresponding to each attack type $\{\xi_i\}^{m+1}_{i=1}$. For each $\theta \in \Theta$, we define 
\begin{align*}
     \hat{V} (\theta) & :=   \frac{1}{m} \sum_{i=1}^m  \E_{\tau \sim q^{\theta}_i} J_\D ( \theta  +   \eta \hat{\nabla}_{\theta} J_\D ( \tau ) , \phi_{i}, \xi_i), \\
      \hat{V}_{m+1} (\theta) & := \E_{\tau \sim q^{\theta}_{m+1}} J_\D ( \theta  +   \eta \hat{\nabla}_{\theta} J_\D ( \tau ) , \phi_{m+1}, \xi_{m+1}), 
\end{align*}
where $q^{\theta}_i (\cdot) \triangleq   q(\theta, \phi_i)  $ is the trajectory distribution determined by state dependent policies $\pi_\D(\cdot|s; \theta)$, $\pi_{\xi_i}( \cdot |s; {\phi_i} )$ and transition kernel $\mathcal{T}$. 
Let $\|\cdot\|_{TV}$ be the total variation, $d_i$ be the residue marginal factors of $q^\theta_i(\cdot)$ after removing $\pi_\D$, i.e., $d_i = \prod_{t=1}^{H-1}\pi_{\xi_i}(a_\A^t | s^t, \phi_i) \prod_{t=1}^{H-1} \mathcal{T}(s^{t+1}|s^t, a_\D^t, a_\A^t)$, we have generalization characterization in \Cref{main_prop:generalization}.} 
\begin{proposition}
    \label{main_prop:generalization}
  Under assumptions~ \ref{ass:sc}, \ref{asslip}, \ref{plass}, and \ref{ass:grad}, fixing a policy $\theta \in \Theta$, 
 \begin{equation*}
      | \hat{V}_{m+1} (\theta) - \hat{V} (\theta) | \leq C(d_{m+1},  \{d_{i}\}_{i=1}^m ), 
 \end{equation*}
where the distance function $C$ depends on the total variation between $d_{m+1}$ and $\{ d_i\}_{i=1}^m$:
\begin{align*}
     C( d_{m+1}, \{d_i\}_{i=1}^m )  & : = \frac{2\eta G^2}{m} \sum_{i=1}^m\| d_{m+1} -  d_i \|_{TV}\\ & \quad  + \frac{1 - \gamma^H}{ 1 - \gamma} \| d_{m+1} - \frac{1}{m} \sum_{i=1}^m d_i\|_{TV} .
\end{align*}
\end{proposition}

\section{Experiments}
\label{sec:exp}

\subsection{Experiment Settings}
\label{sec:exp-setup}
\paragraph{Dataset}
Our experiments are conducted on MNIST~\cite{lecun1998gradient} and CIFAR-10~\cite{krizhevsky2009learning} datasets with a CNN classifier and ResNet-18 model respectively (see Appendix~\ref{app:exp-setup} for details). 
We consider horizontal FL and adopt 
the approach introduced in~\cite{fang2020local} to measure the diversity of local data distributions among clients. Let the dataset encompass $C$ classes, such as $C = 10$ for datasets like MNIST and CIFAR-10. Client devices are divided into $C$ groups (with $M$ attackers evenly distributed among these groups). Each group is allocated $1/C$ of the training samples in the following manner: a training instance labeled as $c$ is assigned to the $c$-th group with a probability of $q \geq 1/C$, while being assigned to every other group with a probability of $(1-q)/(C-1)$. Within each group, instances are evenly distributed among clients. A higher value of $q$ signifies a greater \textit{non-i.i.d.} level. By default, we set $q=0.5$ as the standard \textit{non-i.i.d.} level. We assume the server holds a small amount of root data randomly sampled from the  the collection of all client datasets $U$ ($100$ for MNIST and $200$ for CIFAR-10).

\paragraph{Baselines}  We evaluate our meta-RL and meta-SG defenses under the following untargeted model poisoning attacks including IPM~\cite{xie2020fall} (with scaling factor $2$), LMP~\cite{fang2020local}, RL~\cite{li2022learning}, and backdoor attacks including BFL~\cite{bagdasaryan2020backdoor} (with poisoning ratio $1$), DBA~\cite{xie2019dba} (with $4$ sub-triggers evenly distributed to 
malicious clients and poisoning ratio $0.5$), BRL~\cite{li2023learning}, and a mix of attacks from the two categories (see~\cref{table:attacks} for all attacks' categories in Appendix~\ref{app:exp-setup}). 
We consider various strong defenses as baselines, including training-stage defenses such as 
Coordinate-wise trimmed mean/median~\cite{yin2018byzantine}, Norm bounding~\cite{sun2019can}, FLTrust~\cite{cao2020FLTrust}, Krum~\cite{blanchard2017machine}, and post-training stage defenses such as NeuroClip~\cite{wang2023mm} and Prun~\cite{wu2020mitigating} and the selected combination of them. We utilize the Twin Delayed DDPG (TD3)~\cite{fujimoto2018addressing} algorithm to train both attacker's and defender's policies. 
We use the following default parameters: number of devices $=100$, number of malicious clients for untargeted model poisoning attack $=10$, number of malicious clients for backdoor attack $=5$ ($20$ for DBA), client subsampling rate $=10\%$, number of FL epochs $=500$ ($1000$) for MNIST (CIFAR-10). We fix the initial model and the random seeds for client subsampling and local data sampling for fair comparisons. 
The details of the experiment setup and additional results are provided in Appendix~\ref{app:exp-setup} and \ref{app:add-exp}, respectively.

\subsection{Experiment Results}
\begin{table*}[!h]
\centering
\begin{tabular}{@{\extracolsep{1pt}}lccccccc}
            \toprule 
        Acc/Bac & \multicolumn{1}{c}{FedAvg} & \multicolumn{1}{c}{Trimed Mean} & \multicolumn{1}{c}{FLTrust}& \multicolumn{1}{c}{ClipMed} & \multicolumn{1}{c}{FLTrust+NC} & \multicolumn{1}{c}{Meta-RL (ours)}\\
        \midrule
                NA & $0.7082/0.1$ & $0.7093/0.1078$  & $0.7139/0.1066$ & $0.5280/0.1212$ & $0.7100/0.1061$  & $\mathbf{0.7053/0.0999}$  \\
                IPM & $0.1369/\mathbf{0.0312}$ & $0.6542/0.1174$ & $0.6828/0.1054$ & $0.5172/0.1220$ & $0.6656/0.0971$ & $\mathbf{0.6862}/0.0637$ \\
                LMP & $0.1115/0.1174$ & $0.6224/0.1033$  & $0.7071/0.099$ & $ 0.5144/0.121$ & $0.7075/0.104$ & $\mathbf{0.7109/0.037}$ \\
                BFL & $0.7137/1.0$ & $0.7034/1.0$ & $\mathbf{0.7145}/1.0$ & $0.5198/0.5337$ & $0.7100/0.1061$ & $0.7106/\mathbf{0.0143}$ \\
                DBA & $0.7007/0.7815$ & $0.6904/0.7737$ & $\mathbf{0.7010}/0.8048$ & $0.4935/0.6261$ & $0.6618/0.9946$ & $0.6699/\mathbf{0.2838}$ \\
                IPM+BFL & $0.3104/0.8222$ & $0.6415/1.0$ & $0.6911/1.0$ & $0.5097/0.5776$ & $0.6817/0.0267$ & $\mathbf{0.6949/0.0025}$  \\
                LMP+DBA & $0.1124/\mathbf{0.1817}$ & $0.6444/0.7311$ & $\mathbf{0.7007}/0.7620$ & $0.4841/0.6342$ & $0.6032/0.8422$ & $0.6934/0.2136$  \\
        \bottomrule 
\end{tabular}
\caption{Comparisons of average global model accuracy (acc: higher the better) and backdoor accuracy (bac: lower the better) after 500  rounds under single/multiple type attacks on CIFAR-10. All parameters are set as default, and random seeds are fixed. Boldfaced numbers indicate the best performance.}
\label{table:1}
\end{table*}
\paragraph{Effectiveness against single/multiple types of attacks.}
We examine the defense performance of our meta-RL compared with other defense combinations in \Cref{table:1} based on average global model accuracy after 500 FL rounds on CIFAR-10, which measures the success of defense and learning speed ignoring the randomness influence (corner-case updates, bias data, etc.) at the bargaining stage of FL. The meta-RL first learns a meta-defense policy from the attack domain involving $\{$NA, IPM, LMP, BFL, DBA$\}$, then adapts it to the real single/mixed attack. 
We observe that multiple types of attacks may intervene with each other (e.g., IPM+BFL, LMP+DBA), which makes it impossible to manually address the entangled attacks.
It is not surprising to see FedAvg~\cite{mcmahan2017communication} and defenses specifically designed for untargeted attacks (i.e., Trimmed mean, FLTrust) fail to defend backdoor attacks (i.e., BFL, DBA) due to the huge deviation of defense objective from the optimum.
For a fair comparison, we further manually tune the norm threshold (more results in Appendix~\ref{app:add-exp}) from $[0.01, 0.02, 0.05, 0.1, 0.2, 0.5, 1]$ for ClipMed (i.e., Norm bounding + Coordinate-wise Median) and clipping range from $[2:2:10]$ for FLTrust + NeuroClip to achieve the best performance to balance the global model and backdoor accuracy in linear form (i.e., Acc - Bac). Intuitively, a tight threshold/range has better performance in defending against backdoor attacks, yet will hinder or even damage the FL progress. On the other hand, a loose threshold/range fails to defend backdoor injection.  Nevertheless, manually tuning in real-world FL scenarios is nearly impossible due to the limited knowledge of the ongoing environment and the presence of asymmetric adversarial information. Instead of suffering from the above concerns and exponential growth of parameter combination possibilities, our data-driven meta-RL approach can automatically tune multiple parameters at each round. Targeting the cumulative defense rewards, the RL approach naturally holds more flexibility than myopic optimization.
\begin{figure*}[t]
 \vspace{-5pt}
  \centering
  \begin{subfigure}{0.24\textwidth}
      \centering
          \includegraphics[width=\textwidth]{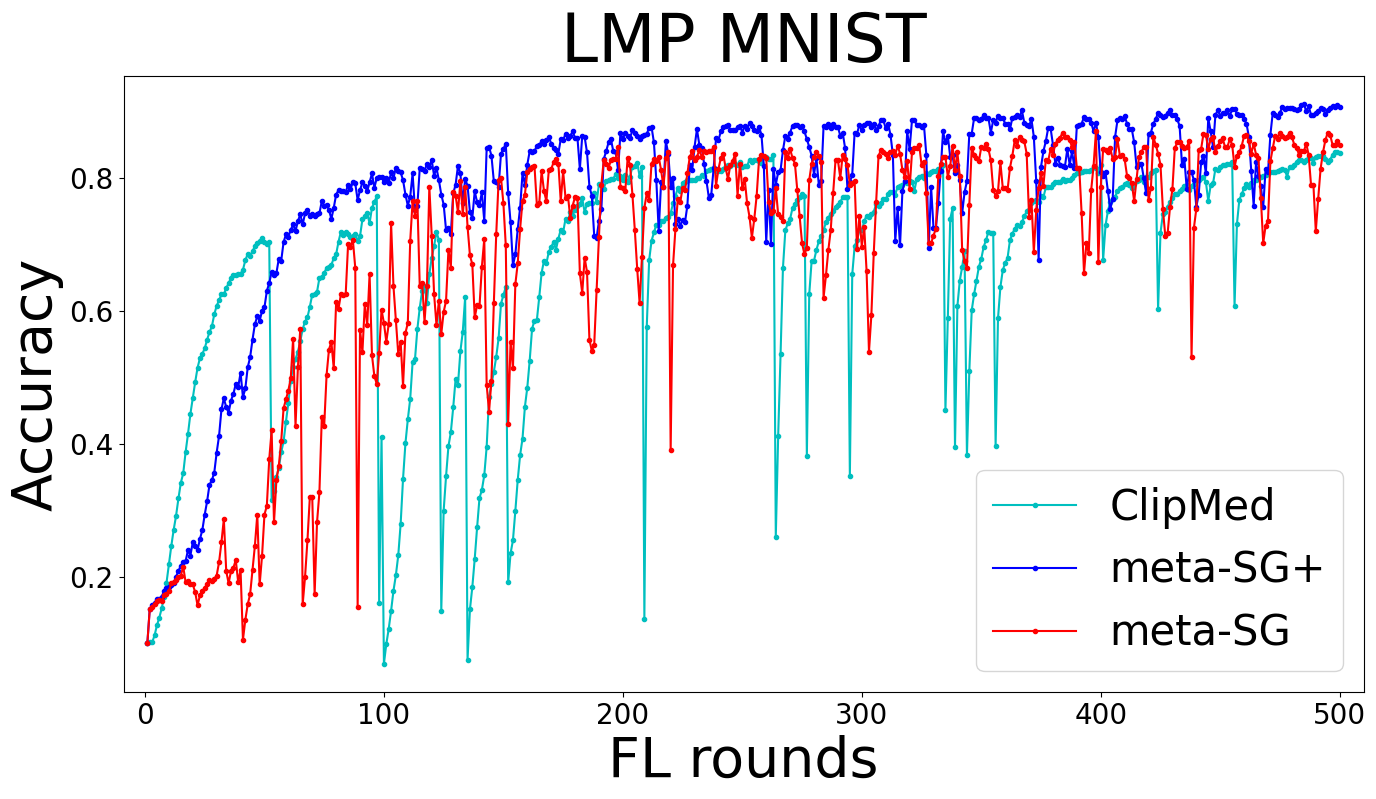}
  \end{subfigure}
  \hfill
    \begin{subfigure}{0.24\textwidth}
      \centering
          \includegraphics[width=\textwidth]{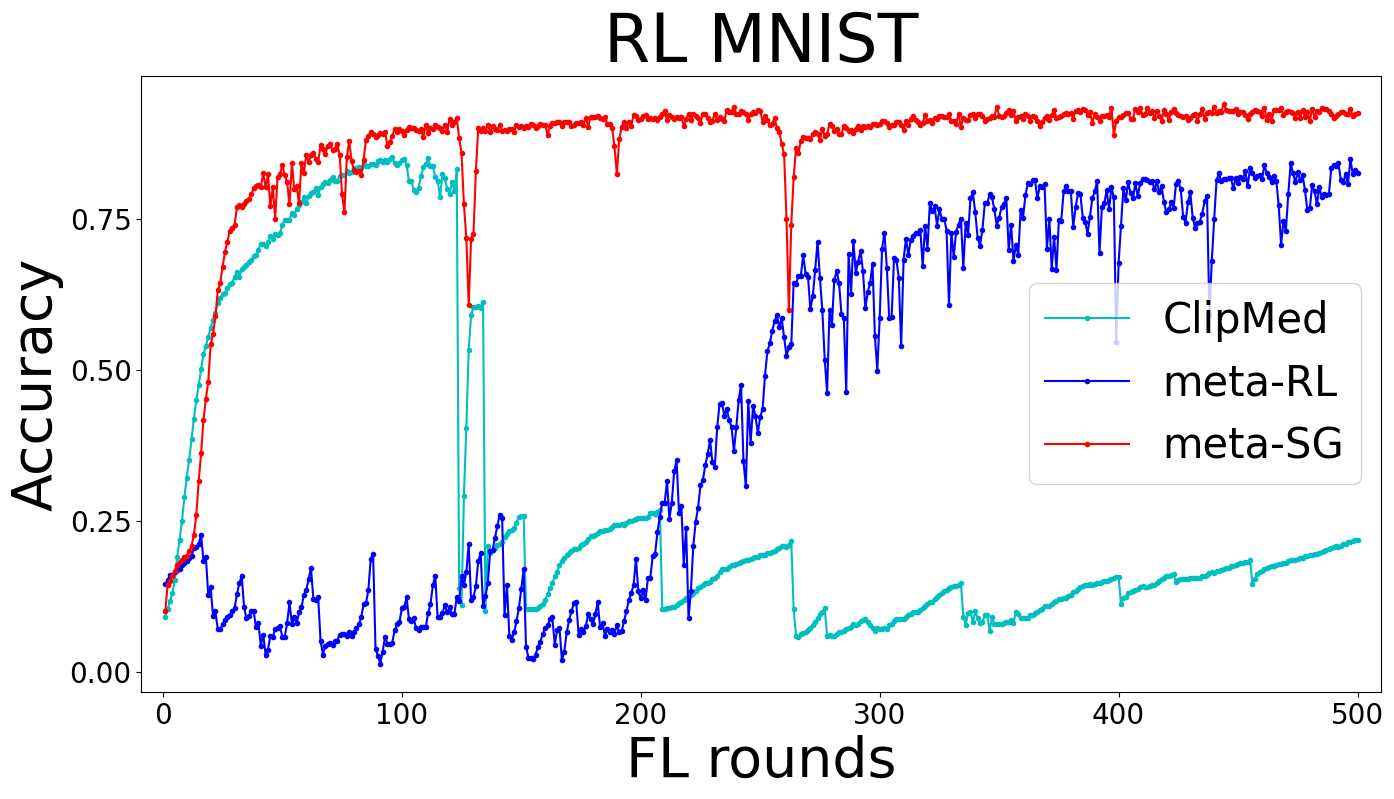}
    \end{subfigure}
  \hfill
    \begin{subfigure}{0.24\textwidth}
      \centering
          \includegraphics[width=\textwidth]{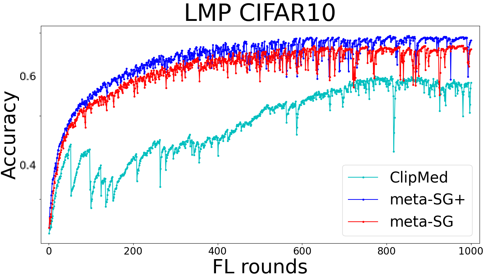}
    \end{subfigure}
  \hfill
    \begin{subfigure}{0.24\textwidth}
      \centering
          \includegraphics[width=\textwidth]{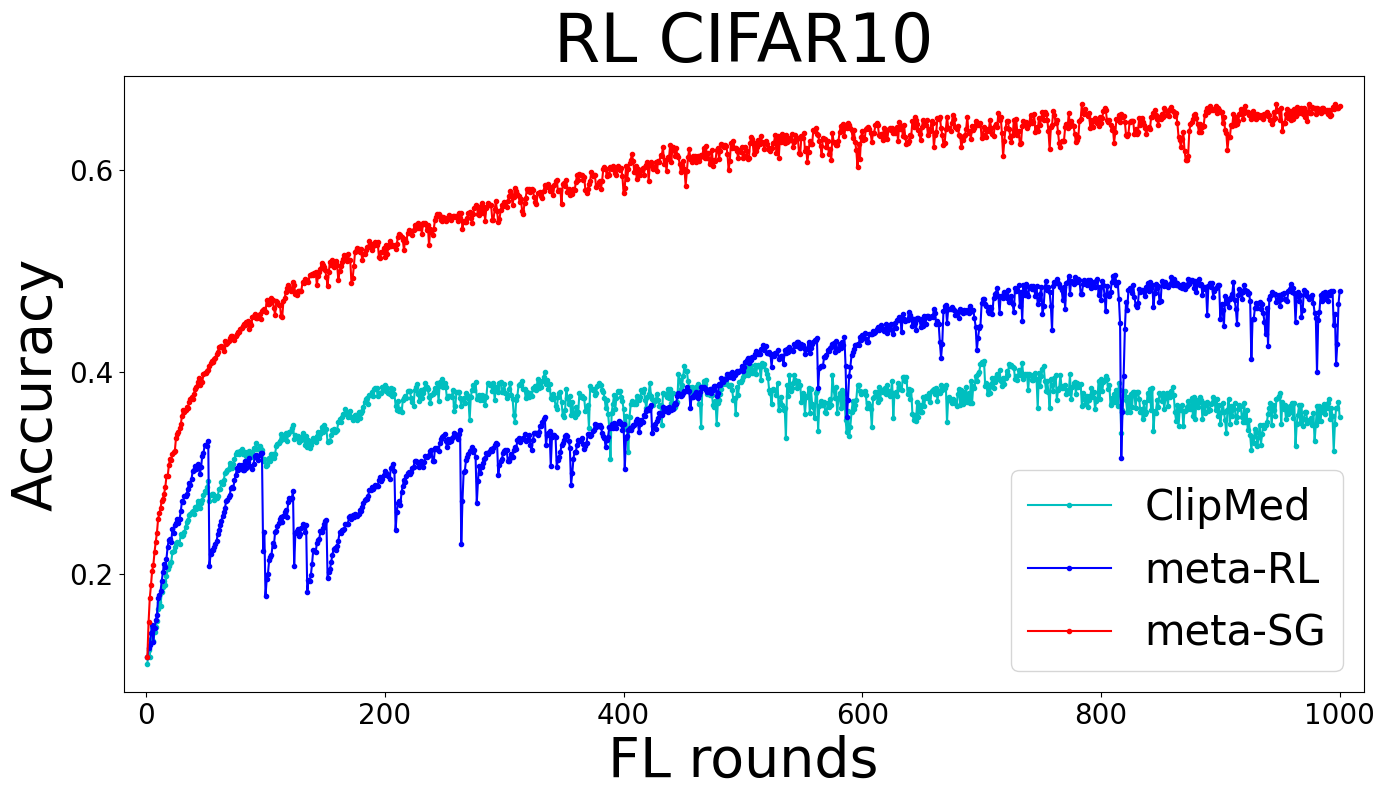}
    \end{subfigure}

 	\caption{\small Comparisons of defenses against untargeted model poisoning attacks (i.e., LMP and RL) on MNIST and CIFAR-10. All parameters are set as default and random seeds are fixed.}
 	\vspace{-0.4cm}
 	\label{fig:untargeted}
\end{figure*}

\paragraph{Adaptation to uncertain/unknown attacks.}
To evaluate the necessity and efficiency of adaptation from the meta-SG policy in the face of unknown attacks, we plot the global model accuracy graph over FL epochs. The meta-RL pre-trained from non-adaptive attack domain $\{$NA, IPM, LMP, BFL, DBA$\}$ (RL attack is unknown), while meta-SG pre-train from interacting with a group of RL attacks initially target on $\{$FedAvg, Coordinate-wise Median, Norm bounding, Krum, FLTrust $\}$ (LMP is unknown). The meta-SG plus (i.e., meta-SG+) is a pre-trained model from the combined attack domain of the above two. All three defenses then adapt to the real FL environments under LMP or RL attacks.
As shown in Fig.~\ref{fig:untargeted}, the meta-SG can quickly adapt to both uncertain RL-based adaptive attacks (attack action is time-varying during FL) and unknown LMP attacks, while meta-RL can only slowly adapt to or fail to adapt to the unseen RL-based adaptive attacks on MNIST and CIFAT-10 respectively. In addition, the first and the third figures in Fig.~\ref{fig:untargeted} demonstrate the power of meta-SG against unknown LMP attacks, even if LMP is not directly used during its pre-training stage. The results are only slightly worse than meta-SG plus, where LMP is seen during pre-training. 
Similar observations are given under IPM in Appendix~\ref{app:add-exp}. 

\paragraph{Defender's knowledge of backdoor attacks.} 
We consider two settings: 1) the server knows the backdoor trigger but is uncertain about the target label, and 2) the server knows the target label but not the backdoor trigger. 
In the former case, the meta-SG first pre-trains the defense policy with RL attacks using a known fixed global pattern (see Fig.~\ref{fig:cifar10_dba}) targeting all 10 classes in CIFAR-10, then adapts with an RL-based backdoor attack using the same trigger targeting class 0 (airplane), with results shown in the third figure of Fig.~\ref{fig:backdoors}.  
In the latter case where the defender does not know the true backdoor trigger used by the attacker, we implement the GAN-based model~\cite{doan2021lira} to generate the worst-case triggers (see Fig.~\ref{fig:gan_trigger}) targeting one known label (truck). The meta-SG will train a defense policy with the RL-based backdoor attacks using the worst-case triggers targeting the known label,
then adapt with a RL-based backdoor attack using a fixed global pattern (see Fig.~\ref{fig:cifar10_dba}) targeting the known label in the real FL environment (results shown in the fourth graph in Fig.~\ref{fig:backdoors}.
We call the two above cases \textbf{blackbox} settings since the defender misses key backdoor information and solely depends on their own generated data/triggers w/o inverting/reversing during online adaptation.   
In the \textbf{whitebox} setting, the server knows the backdoor trigger pattern (global) and the targeted label (truck), and is trained by true clients' data. The corresponding results are in the first two figures of Fig.~\ref{fig:backdoors}, which show the upper bound performance of meta-SG and may not be practical in a real FL environment. Post-training defenses alone (i.e., NeuroClip and Prun) and combined defenses (i.e., ClipMed and FLTrust+NC) are susceptible to RL-based attacks once the defense mechanism is known. On the other hand, as depicted in Fig.~\ref{fig:backdoors}, we demonstrate that our whitebox meta-SG approach is capable of effectively eliminating the backdoor influence while preserving high main task accuracy simultaneously, while blackbox meta-SG against uncertain labels is unstable since the meta-policy will occasionally target a wrong label, even with adaptation and blackbox meta-SG against unknown trigger is not robust enough as its backdoor accuracy still reaches nearly $50\%$ at the end of FL training.

\begin{figure*}[t]
    \centering
    \centering
  \begin{subfigure}{0.24\textwidth}
      \centering
          \includegraphics[width=\textwidth]{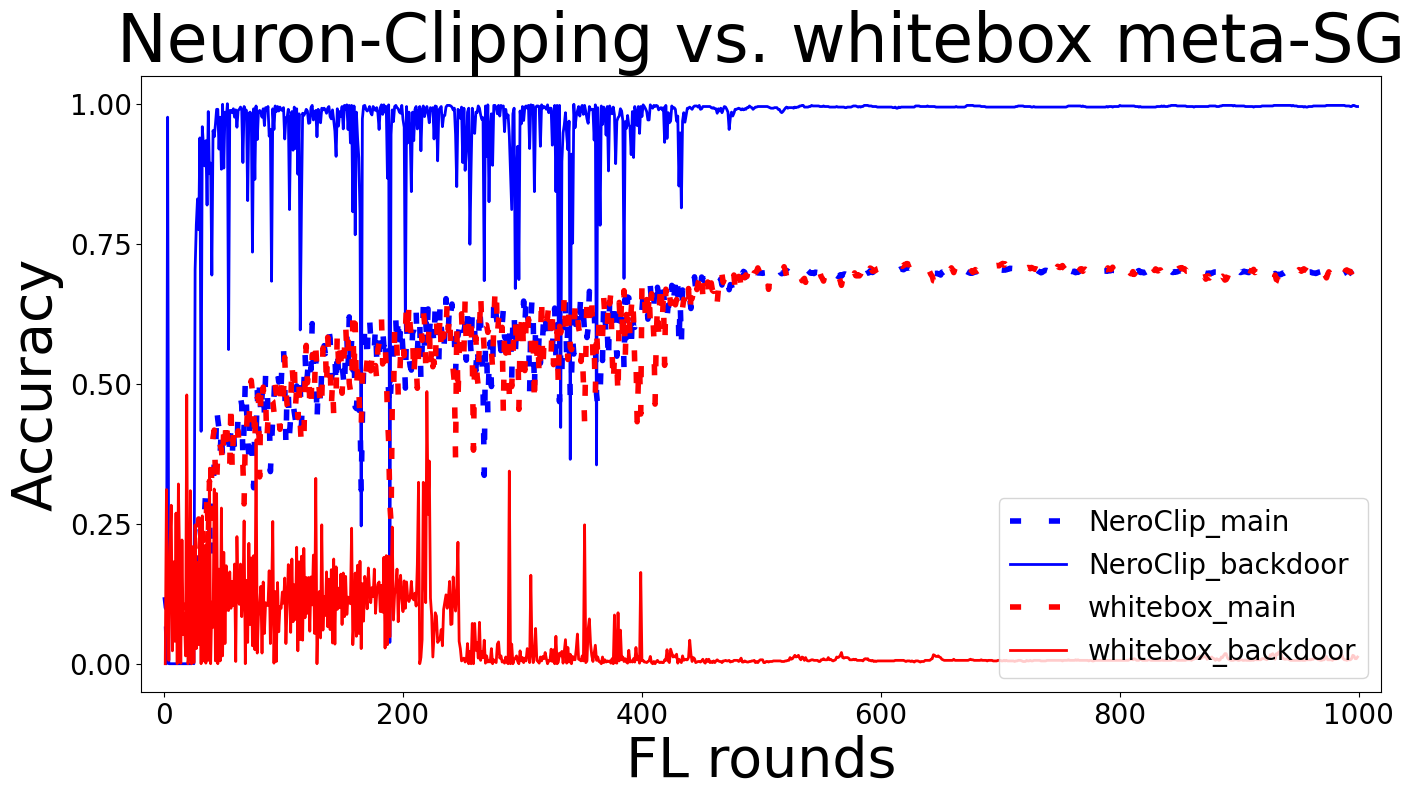}
  \end{subfigure}
  \hfill
    \begin{subfigure}{0.24\textwidth}
      \centering
          \includegraphics[width=\textwidth]{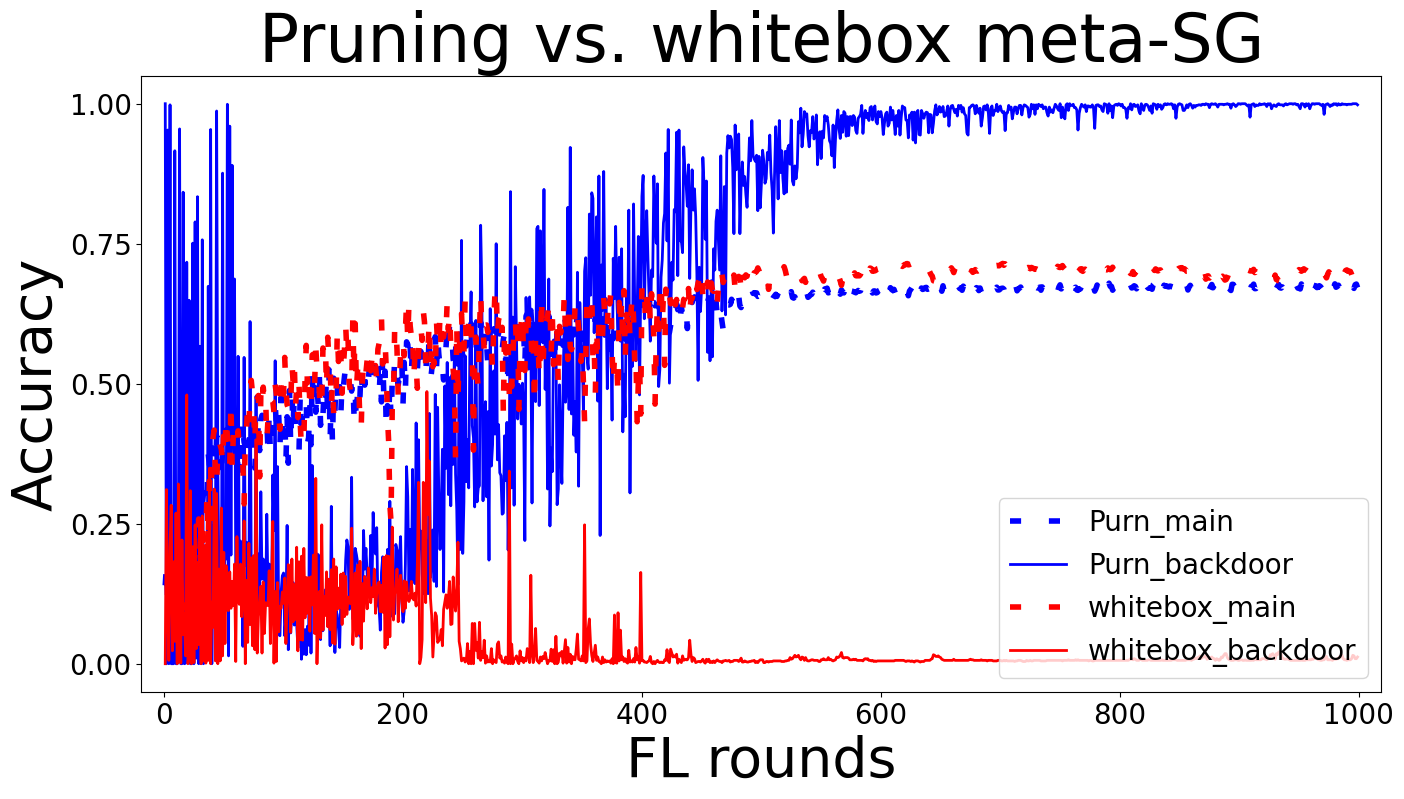}
    \end{subfigure}
  \hfill
    \begin{subfigure}{0.24\textwidth}
      \centering
          \includegraphics[width=\textwidth]
          {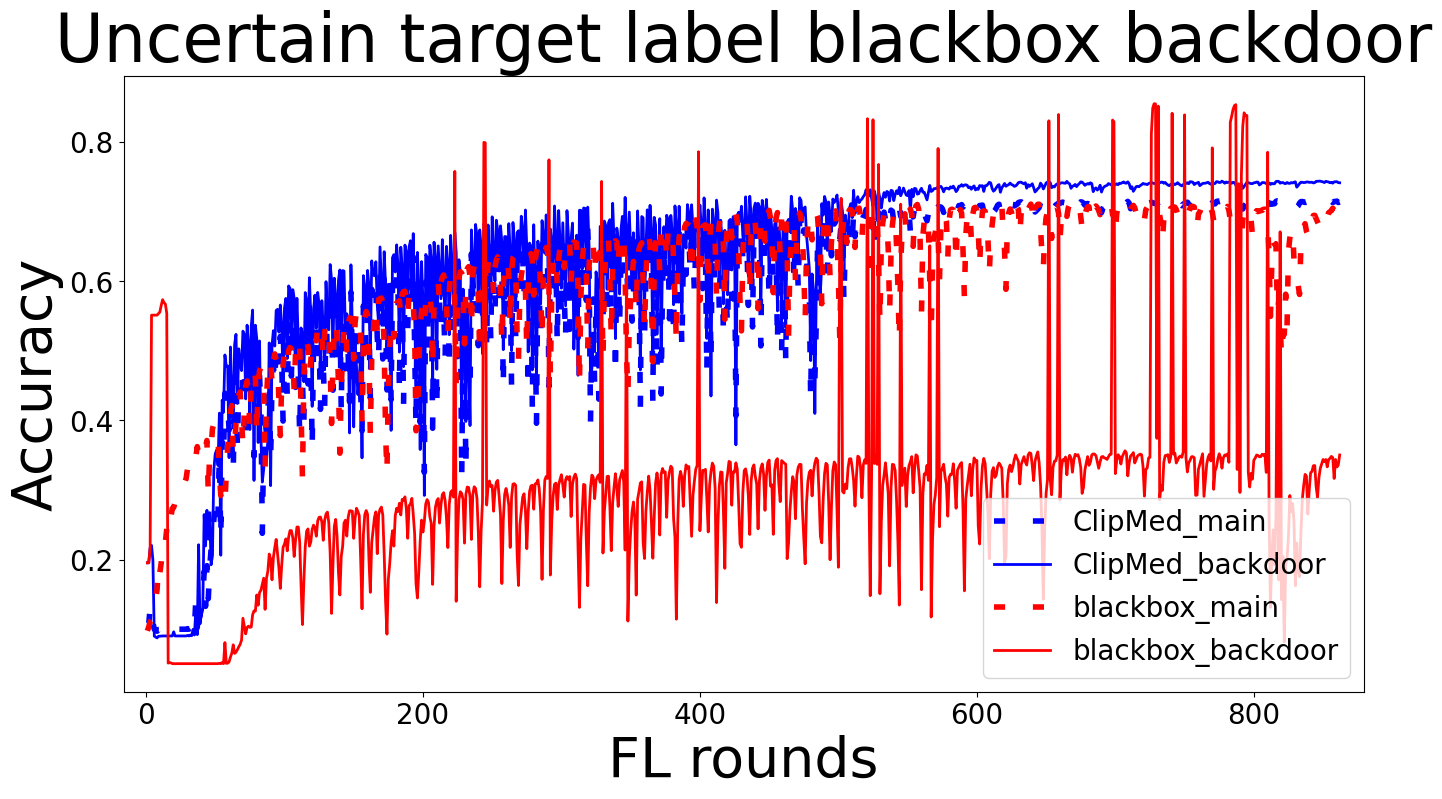}
    \end{subfigure}
  \hfill
    \begin{subfigure}{0.25\textwidth}
      \centering
          \includegraphics[width=\textwidth]
          {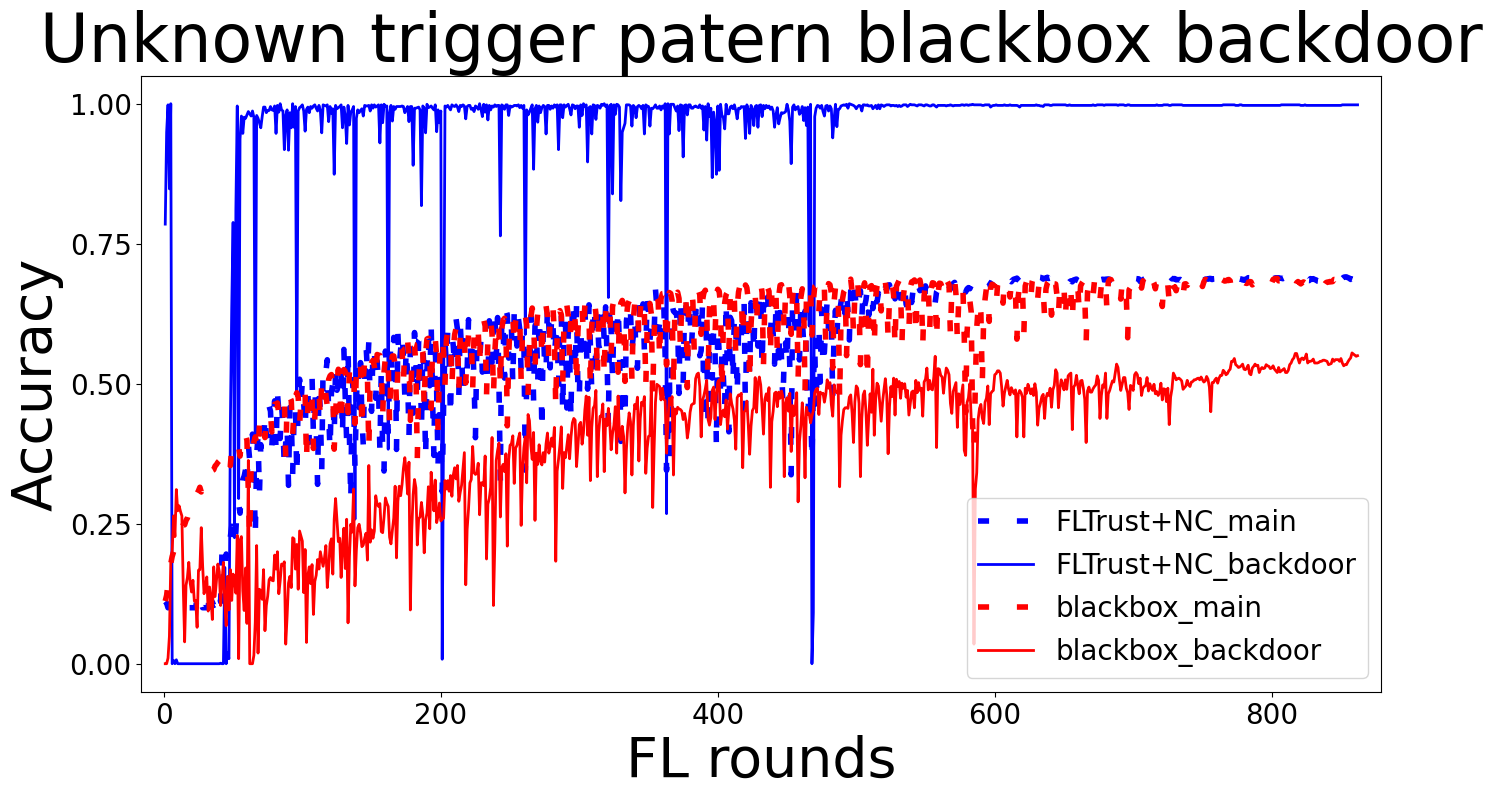}
    \end{subfigure}
  
  \caption{\small{Comparisons of baseline defenses, i.e., NeuroClip, Prun, ClipMed, FLTrust+NeuroClip (from left to right) and whitebox/blackbox meta-SG under RL-based backdoor attack (BRL) on CIFAR-10. The BRLs are trained before FL round 0 against the associate defenses (i.e., NeuroClip, Prun, ClipMed, FLTrust+NC and meta-policy of meta-SG). Other parameters are set as default and all random seeds are fixed.}}\label{fig:backdoors}
\end{figure*}

\begin{figure*}[t]
 	\vspace{-5pt}
 	\centering
 		\subfloat[]{%
 		\includegraphics[width=0.25\textwidth]
                {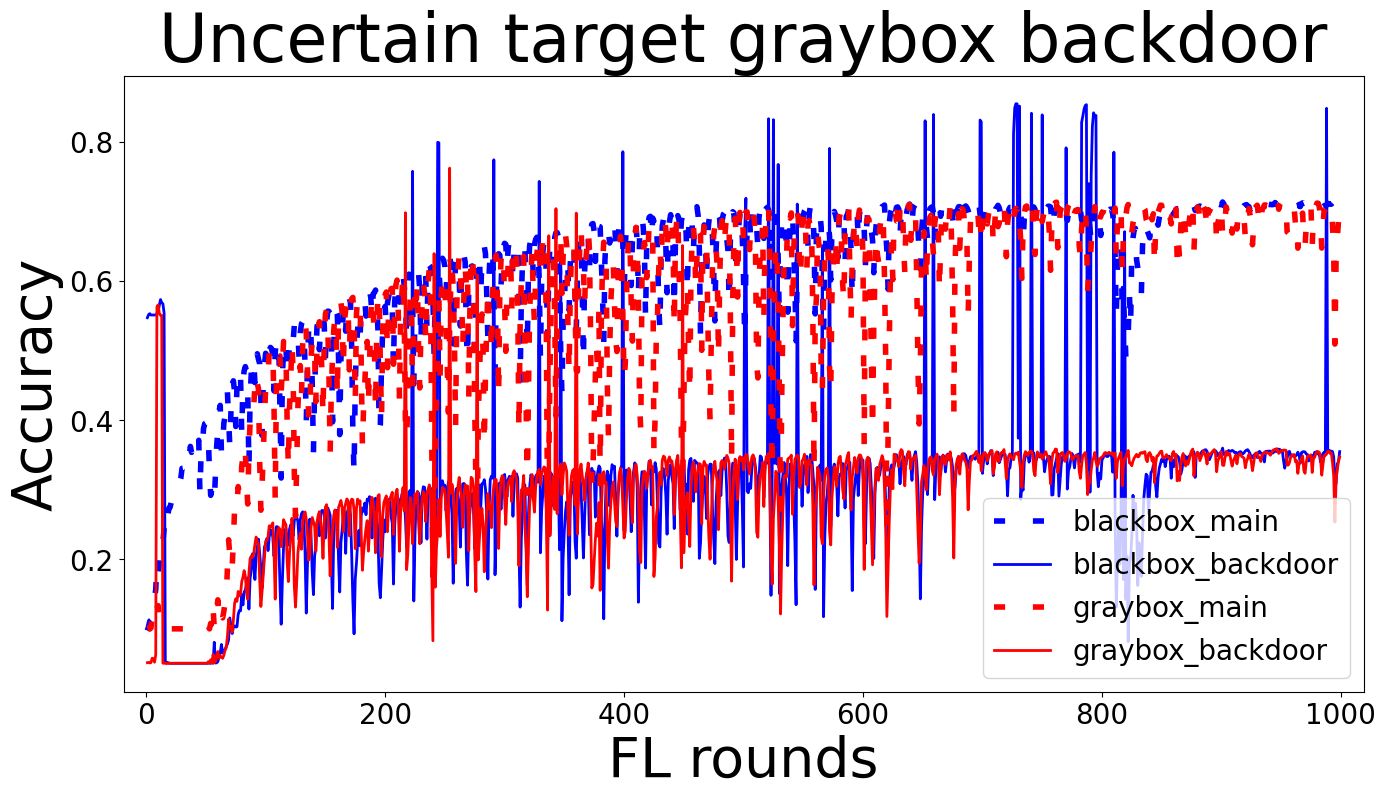}%
                
 	  }
 	      \subfloat[]{%
 		\includegraphics[width=0.2525\textwidth]
                {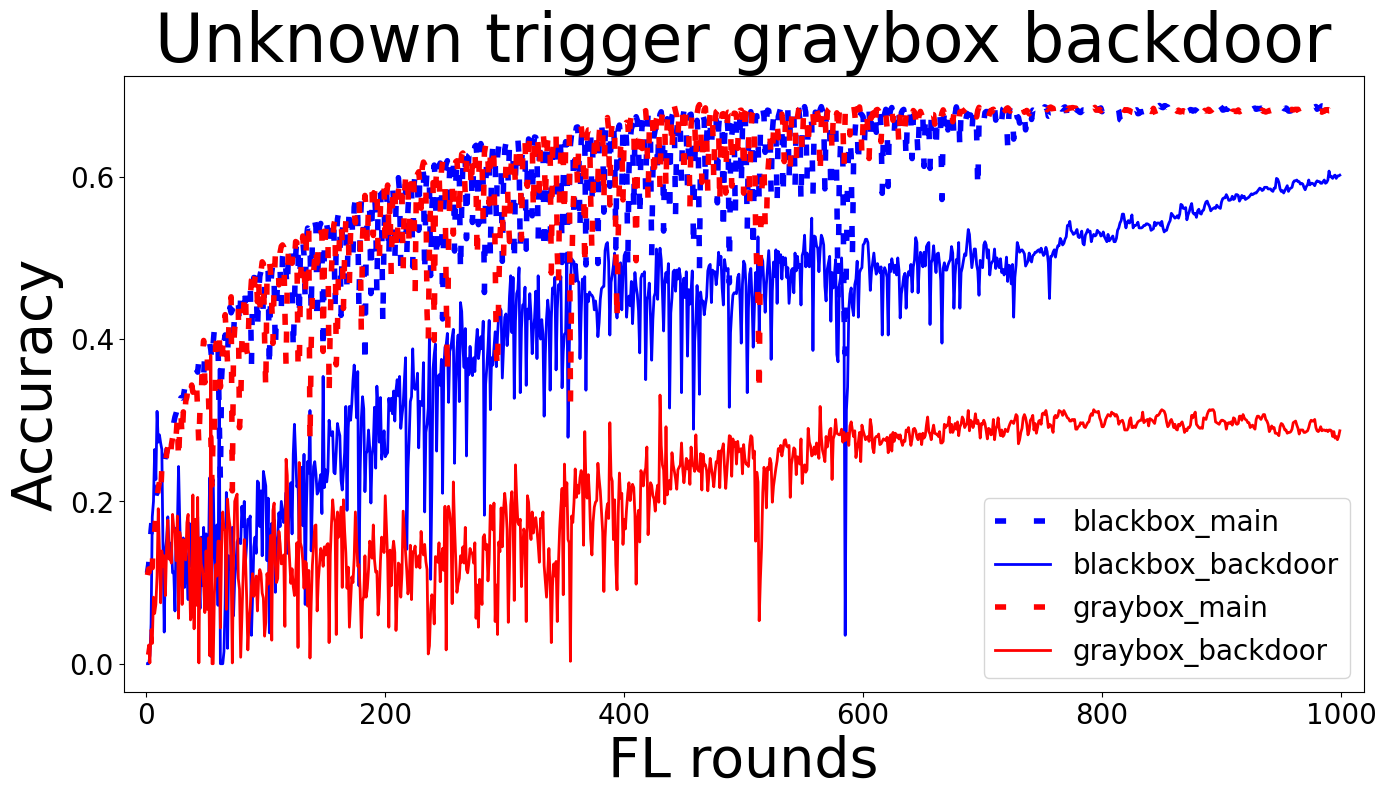}%
        
 	}
 	      \subfloat[]{%
 		\includegraphics[width=0.235\textwidth]{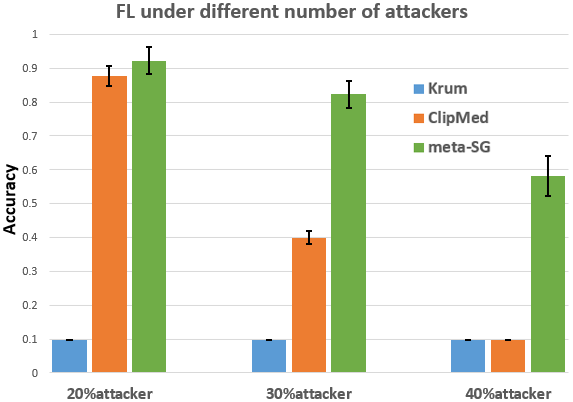}%
 	}
 	      \subfloat[]{%
 		\includegraphics[width=0.25\textwidth]{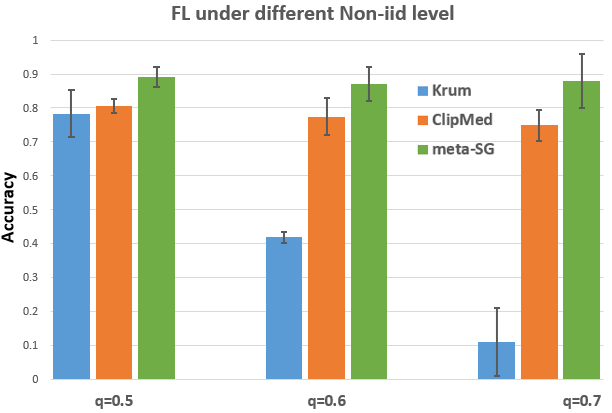}%
 	}

 	 \caption{\small Ablation studies. (a)-(b): uncertain backdoor target and unknown backdoor triggers, where the meta-policies are trained by worst-case triggers generated from GAN-based models~\cite{doan2021lira} or targeting multiple labels on CIFAR-10 during pre-training and utilizing inverting gradient~\cite{geiping2020inverting} and reverse engineering~\cite{wang2019neural} during online adaptation. 
     (c)-(d): meta-RL tested by the number of malicious clients in $[20\%, 30\%, 40\%]$ and non-$i.i.d.$ level in $q=[0.5, 0.6, 0.7]$ on MNIST compared with Krum and ClipMed under LMP attack. Other parameters are set as default.}
 	\vspace{-0.4cm}
 	\label{fig:ablation}
 \end{figure*}

\paragraph{Importance of inverting/reversing methods.}
In the ablation study, we examine a practical and relatively well-performed \textbf{graybox} meta-SG. The graybox meta-SG has the same setting as \textbf{blackbox} meta-SG during pre-training as describe in \Cref{subsec:bsmg}, but utilizes inverting gradient~\cite{geiping2020inverting} and reverse engineering~\cite{wang2019neural} during online adaptation to learn clients' data and backdoor trigger in a way without breaking the privacy condition in FL. The graybox approach only learns ambiguous data from clients, then applies data augmentation (e.g., noise, distortion) and combines them with previously generated data before using. 
Fig.~\ref{fig:ablation}(a) illustrates that graybox meta-SG exhibits a more stable and robust mitigation of the backdoor attack compared to blackbox meta-SG. Furthermore, in Fig.~\ref{fig:ablation}(b), graybox meta-SG demonstrates a significant reduction in the impact of the backdoor attack, achieving nearly a $70\%$ mitigation, outperforming blackbox meta-SG.

\paragraph{Number of malicious clients/Non-i.i.d. level.} 
Here we apply our meta-RL to study the impact of inaccurate knowledge of the number of malicious clients and the non-$i.i.d.$ level of clients' local data distribution. With rough knowledge that the number of malicious clients is in the range of $5\%$-$50\%$, the meta-SG will pre-train on LMP attacks with malicious clients $[5:5:50]$, and adapt to three cases with $20\%$, $30\%$, and $40\%$ malicious clients in online adaptation, respectively. Similarly, when the \textit{non-i.i.d.} level is between $0.1$-$1$, the meta-SG will pre-train on LMP attacks with \textit{non-i.i.d.} level $[0.1:0.1:1]$ and adapt to q$=0.5,0.6,0.7$ in online adaptation. As illustrated in Fig.~\ref{fig:ablation}(c) and~\ref{fig:ablation}(d), meta-SG reaches the highest model accuracy for all numbers of malicious clients and non-$i.i.d.$ levels under LMP.

\section{Related Works}
\label{sec:related}
\subsection{Poisoning/Backdoor Attacks and Defenses in FL} Several defensive strategies against model poisoning attacks broadly fall into two categories. The first category includes
robust-aggregation-based defenses encompassing techniques such as dimension-wise filtering. These methods treat each dimension of local updates individually, as explored in studies by~\cite{bernstein2018signsgd,yin2018byzantine}. Another strategy is client-wise filtering, aiming to limit or entirely eliminate the influence of clients who might harbor malicious intent. This approach has been examined in the works of~\cite{blanchard2017machine,pillutla2022robust,sun2019can}. Some defensive methods necessitate the server having access to a minimal amount of root data, as detailed in the study by~\cite{cao2020FLTrust}. 
Naive backdoor attacks are limited by even simple defenses like norm-bounding ~\cite{sun2019can} and weak differential private ~\cite{geyer2017differentially} defenses.
Despite the sophisticated design of state-of-the-art non-adaptive backdoor attacks against federated learning, post-training stage defenses ~\cite{wu2020mitigating,nguyen2021flame,rieger2022deepsight} can still effectively erase suspicious neurons/parameters in the backdoored model.

\subsection{Defenses Against Unknown Attacks} Prior works have attempted to tackle the challenge of incomplete information on attack types through two distinct approaches. The first approach is the ``infer-then-counter'' approach, where the hidden information regarding the attacks is first inferred through observations. For example, one can infer the backdoor triggers through reverse engineering using model weights \cite{bolun19neural-cleanse}, based on which the backdoor attacks can be mitigated \cite{chen21fedreverse}. The inference helps adapt the defense to the present malicious attacks. However, this inference-based adaptation requires prior knowledge of the potential attacks (i.e., backdoor attacks) and does not directly lend itself to mixed/adaptive attacks. Moreover, the inference and adaptation are offline, unable to counter online adaptive backdoor attack \cite{li2022learning}. Even though some concurrent efforts have attempted online inference \cite{li24col-aisg, hammer24col}, they mainly target a small set of attack types and do not scale.  The other approach has explored the notion of robustness that prepares the defender for the worst case \cite{sinha2018certifying, sengupta20bsmg}, which often leads to a Stackelberg game (SG) between the defender and the attacker. Yet, such a Stackelberg approach often leads to conservative defense, lacking adaptability. Most relevant to our meta-RL-based defense is \cite{tao23ztd}, where meta-learning-based zero-trust network defense is proposed to combat unknown attacks. However, the attack setup is relatively simple and does not consider adaptive attacks as in our work.

\subsection{Usage of Public Dataset in FL}
In FL, it is a common practice to use a small globally shared dataset to enhance robustness (see Section 3.1.1 of~\cite{kairouz2021advances}). This dataset could come from a publicly available proxy source, a separate non-sensitive dataset, or a processed version of the raw data as suggested by~\cite{wang2018dataset}. The use of such public datasets is widely accepted in the FL community ~\cite{kairouz2021advances, fang2020local, huang2022learn, yoshida2020hybrid}. For example, systems like Sageflow~\cite{park2021sageflow}, Zeno~\cite{xie2019zeno}, and Zeno++~\cite{xie2020zeno++} leverage public data at the server to address adversarial threats. Additionally, having public data available on the server supports collaborative model training with formal differential or hybrid differential privacy guarantees~\cite{kairouz2021advances, avent2017blender}.~\cite{avent2017blender} introduces hybrid differential privacy, where some users voluntarily share their data. Many companies, such as Mozilla and Google, utilize testers with high mutual trust who opt into less stringent privacy models compared to the average end-user. 

\section{Conclusion and Future Work}

We have proposed a meta-Stackelberg framework to tackle attacks of uncertain or unknown types in federated learning through data-driven adaptation. 
The proposed meta-Stackelberg equilibrium (\Cref{def:meta-se}) approach is computationally tractable and strategically adaptable, targeting mixed and adaptive attacks under incomplete information. We have developed meta-Stackelberg learning (\Cref{algo:meta-sl}) to approximate the $\varepsilon$-meta-equilibrium, which avoids second-order Hessian computation and matches the state-of-the-art sample complexity: $O(\varepsilon^{-2})$ gradient iterations with $O(\varepsilon^{-4})$ samples per iteration.
 
This paper opens up several directions for future work. One direction is to incorporate additional state-of-the-art defense algorithms to counter more potent attacks, such as edge-case attacks~\cite{wang2020attack}, as well as other attack types, such as privacy-leakage attacks~\cite{lyu2022privacy}. It is also worth exploring more sophisticated application scenarios, including NLP and large generative models, since our meta-Stackelberg framework essentially addresses incomplete information in defense design, which is ubiquitous in adversarial machine learning. Our framework could be further improved by including a client-side defense mechanism that closely mirrors real-world scenarios, replacing the current processes of self-data generation.

\bibliographystyle{ieeetr}
\bibliography{IEEEabrv, reference}
\clearpage
\appendices
\section{Theoretical Proofs}
\label{app:theory}
\setcounter{equation}{0}
\renewcommand{\theequation}{\Alph{section}.\arabic{equation}}
\subsection{Existence of Meta-SE}
We first establish the existence of the first-order meta-SE in the Markov game defined in \Cref{subsec:bsmg}. The proof idea here is we can always augment the original utility function to be strongly concave and then leverage the fixed point theorem to prove the existence of the auxiliary game $( \tilde{\ell}_{\mathcal{D}},   \{\tilde{\ell}_{\xi}\}_{\xi \in \Xi})$. Our proof is inspired by a similar idea in \cite[Proposition 4.2]{pang2016unified}.  
Note that this proof technique does not help us to investigate a second-order equilibrium condition since the Hessians $ \nabla^2 \tilde{\ell}_{\mathcal{D}}$ and $ \nabla^2 \ell_{\mathcal{D}}$ are not equal.

\begin{proof}
 Denote by $\Phi^{|\Xi|}$ the product space of $\Phi$ up to $|\Xi|$ times. It is clear that $\Theta \times \Phi^{|\Xi|}$ is compact and convex as both $\Theta$ and $\Phi$ are compact and convex. Let $\phi \in \Phi^{|\Xi|}, \phi_\xi \in \Phi$ be the type-aggregated and type $\xi$ attacker's strategy, respectively.  Consider twice continuously differentiable utility functions $\ell_\D(\theta, \phi) := \E_{\xi \sim Q} \LL_\D (\theta, \phi_\xi, \xi)$ and $\ell_\xi (\theta, \phi) := \LL_\A(\theta, \phi_\xi, \xi)$ for all $\xi \in \Xi$. 
 Then, there exists a constant $\gamma_c > 0$, such that the auxiliary utility functions $\forall \xi \in \Xi$:
 \begin{equation}\label{auxiliaryutility}
\begin{aligned}
      \tilde{\ell}_\D (\theta; (\theta^{\prime}, \phi^{\prime}))  & :=  \ell_\D (\theta, \phi^{\prime})  - \frac{\gamma_c}{2} \| \theta -  \theta^{\prime}\|^2
      \\  \tilde{\ell}_\xi (\phi_\xi ; ( \theta^{\prime}, \phi^{\prime})) & := \ell_\xi(\theta^{\prime}, (\phi_{\xi }, \phi^{\prime}_{-\xi}) )   - \frac{\gamma_c}{2} \| \phi_\xi -  \phi^{\prime}_\xi\|^2, 
\end{aligned}
\end{equation}
Define the self-map $h: \Theta \times \Phi^{|\Xi|} \to \Theta \times \Phi^{|\Xi|}$ with $h(\theta^{\prime},\phi^{\prime}) := ( \bar{\theta}, \bar{\phi})$, where $\bar{\theta}$ and $\bar{\phi}$ are functions of $(\theta^{\prime},\phi^{\prime})$ given by
\begin{align*}
      \bar{\theta}(\theta^{\prime},\phi^{\prime})  & = \argmax_{\theta \in \Theta}  \tilde{\ell}_\D (\theta; (\theta^{\prime}, \phi^{\prime}) ), 
     \\  
     \bar{\phi}_\xi( \theta^{\prime},\phi^{\prime})  & = \argmax_{\phi_\xi \in \Phi}  \tilde{\ell}_\xi (\phi_\xi ; ( \theta^{\prime}, \phi^{\prime}) ) . 
\end{align*}
Due to compactness of $\Theta \times \Phi^{|\Xi}$, $h$ is well-defined. 
By strong concavity of $ \tilde{\ell}_\D (\cdot; (\theta^{\prime},\phi^{\prime}))$ and $\tilde{\ell}_\xi (\cdot; (\theta^{\prime},\phi^{\prime}))$, it follows from Berge's maximum theorem \cite[Thm 17.31]{aliprantis06} that $h$ is a upper semi-continuous self-mapping from $\Theta \times \Phi^{|\Xi|}$ to itself. By Kakutani's fixed point theorem \cite{glicksberg52kakutani}, there exists at least one $(\theta^*, \phi^*) \in \Theta \times \Phi^{|\Xi|}$ such that $h(\theta^*, \phi^*) = (\theta^*, \phi^*)$, which satisfies the following inequalities (due to the $\argmax$):
\begin{align*}
      \langle \nabla_{\theta} \tilde{\ell}_\D (\theta^*; (\theta^*, \phi^*)) ,  \theta  - \theta^* \rangle & \leq 0 \\ 
      \langle \nabla_{\phi_\xi} \tilde{\ell}_\xi (\theta^*; (\theta^*, \phi^*)) ,  \phi_\xi  - \phi^*_\xi \rangle & \leq 0
\end{align*}

Then, one can verify that $(\theta^*, \phi^*)$ is a meta-FOSE of the meta-SG with utility function $\ell_\D$ and $\ell_\xi$, $\xi \in \Xi$, in view of the following equalities: 
\begin{align*}
   \langle \nabla_{\theta} \tilde{\ell}_\D (\theta^*; (\theta^*, \phi^*)) ,  \theta  - \theta^* \rangle & = \langle \nabla_{\theta} \ell_\D ( \theta^*, \phi^* )  ,  \theta  - \theta^* \rangle   \\
 \langle \nabla_{\phi_\xi} \tilde{\ell}_\xi (\theta^*; (\theta^*, \phi^*)) ,  \phi_\xi  - \phi^*_\xi \rangle & = \langle \nabla_{\phi_\xi } \ell_\xi ( \theta^*, \phi^* )  ,  \phi_\xi  - \phi^*_\xi \rangle .
\end{align*}

the conditions of meta-FOSE are satisfied.
Therefore, the equilibrium conditions for meta-SG with utility functions $\tilde{\ell}_\D$ and $\{\tilde{\ell}_\xi\}_{\xi \in \Xi}$ are the same as with utility functions $\ell_\D$ and $\{\ell_\xi\}_{\xi \in \Xi}$, hence the claim follows.

\end{proof}

\subsection{Proofs: Non-Asymptotic Analysis}

In the sequel, we prove the sample complexity results in \Cref{thm:main}. In addition, we assume, for analytical simplicity, that all types of attackers are unconstrained, i.e., $\Phi$ is the Euclidean space with proper finite dimension. We first recall the following Lipschitz conditions in \cref{asslip}: the functions $\mathcal{L}_{\D}$ and $\LL_\A$ are continuously diffrentiable in both $\theta$ and $\phi$. Furthermore, there exists constants $L_{11}$, $L_{12}, L_{21}$, and $L_{22}$ such that
 for all $\theta, \theta_1, \theta_2 \in \Theta$ and $\phi, \phi_1, \phi_2 \in \Phi$, we have, for any $\xi \in \Xi$,
\begin{flalign}
    \left\|\nabla_{\theta} \mathcal{L}_\D  \left(\theta_{1}, \phi, \xi\right)-\nabla_{\theta} \mathcal{L}_\D\left(\theta_{2}, \phi, \xi\right)\right\| &  \leq L_{11}\left\|\theta_{1}-\theta_{2}\right\|  \label{lip1}
    \\ 
    \left\|\nabla_{\phi} \mathcal{L}_\D \left(\theta, \phi_{1}, \xi\right)-\nabla_{\phi} \mathcal{L}_\D \left(\theta, \phi_{2}, \xi\right)\right\| & \leq L_{22}\left\|\phi_{1}-\phi_{2}\right\|  \label{lip2}
     \\ 
     \left\|\nabla_{\theta} \mathcal{L}_\D \left(\theta, \phi_{1}, \xi\right)-\nabla_{\theta} \mathcal{L}_\D \left(\theta, \phi_{2}, \xi\right)\right\| & \leq L_{12}\left\|\phi_{1}-\phi_{2}\right\|  \label{lip3}
    \\ 
    \left\|\nabla_{\phi} \mathcal{L}_\D \left(\theta_{1}, \phi, \xi\right)-\nabla_{\phi} \mathcal{L}_\D \left(\theta_{2}, \phi, \xi\right)\right\|  & \leq L_{21}\left\|\theta_{1}-\theta_{2}\right\|  \label{lip4}  
    \\ 
    \| \nabla_{\phi} \LL_\A (\theta, \phi_1, \xi ) - \nabla_{\phi} \LL_\A (\theta, \phi_2, \xi)\| & \leq  L_{21} \| \phi_1 - \phi_2\|  \label{lip5}
    \\ 
    \| \nabla_{\phi} \LL_\A (\theta_1, \phi, \xi ) - \nabla_{\phi} \LL_\A (\theta_2, \phi, \xi)\| & \leq  L_{21} \| \theta_1 - \theta_2\| . \label{lip6}
\end{flalign}

\begin{lemma}[Implicit Function Theorem (IFT) for Meta-SG adapted from \cite{still2018lectures}] \label{lemma:ift}
 Suppose for $(\bar{\theta}, \bar{\phi} ) \in \Theta \times \Phi$, $\xi \in \Xi$, we have $\nabla_{\phi} \LL_\A (\bar{\theta}, \bar{\phi}, \xi)= 0 $, and the Hessian $\nabla^2_{\phi} \LL_\A (\bar{\theta}, \bar{\phi}, \xi)$ is non-singular. Then, there exists a neighborhood $B_{\varepsilon}(\bar{\theta}), \varepsilon > 0$ centered around $\bar{\theta}$ and a $C^1$-function $\phi(\cdot): B_{\varepsilon}(\bar{\theta}) \to \Phi$ such that near $(\bar{\theta}, \bar{\phi})$ the solution set $\{ (\theta, \phi) \in \Theta \times \Phi: 
 \nabla_{\phi} \LL_\A (\theta, \phi, \xi)= 0 \}$ is a $C^1$-manifold locally near $(\bar{\theta}, \bar{\phi})$. The gradient $\nabla_{\theta} \phi(\theta)$ is given by $- ( \nabla^2_{ \phi}\LL_\A (\theta, \phi, \xi))^{-1}  \nabla^2_{ \phi \theta} \LL_\A (\theta, \phi, \xi)$.

\end{lemma}
{Recall that in (\ref{eq:meta-se}), $\phi_\xi^*\in\argmax \mathbb{E}_{\tau\sim q}J_\A(\theta+\eta \hat{\nabla}_\theta J_\D(\tau), \phi,  {\xi})$ is a function of $\theta$. Hence, the defender's value function is given by $V(\theta):=\mathbb{E}_{{\xi}\sim Q}\mathbb{E}_{\tau\sim q}[J_\D(\theta+\eta \hat{\nabla}_\theta J_\D(\tau),\phi^*_{\xi}(\theta), {\xi})]$. The computation of $\nabla_\theta V(\theta)$ naturally involves $\nabla_\theta \phi_\xi^*(\theta)$, which brings in the Hessian terms as stated in the lemma above. However, thanks to the strict competitiveness assumption, the following lemma implies that $\nabla_\theta V$ can be calculated using the evaluation of $\phi_\xi$ without Hessian.}

\begin{lemma}\label{liplemma}
    Under assumptions \ref{asslip}, \ref{plass},  there exists $\{ \phi_\xi: \phi_\xi \in \arg\max_{\phi}  \LL_{\A} (\theta, \phi, \xi) \}_{\xi \in \Xi}$, such that the gradient of value function $V(\theta)$  can be written as:
    \begin{equation}\label{eq:firstorderv}
        \nabla_{\theta}  V(\theta) =  \nabla_{\theta} \mathbb{E}_{\xi \sim Q, \tau \sim q} J_\D (\theta + \eta \hat{\nabla}_{\theta} J_{\D}(\tau), \phi_\xi, \xi) .
    \end{equation}
    Moreover, the function $V(\theta)$ is $L$-Lipschitz-smooth, where $L = L_{11} + \frac{L_{12}L_{21}}{\mu}$
    \begin{equation*}
        \|  \nabla_{\theta} V(\theta_1) -  \nabla_{\theta} V(\theta_2) \| \leq L \|\theta_1 - \theta_2 \|.
    \end{equation*}
\end{lemma}

\begin{proof}[Proof of Lemma \ref{liplemma}]
    First, we show that for any $\theta_1, \theta_2 \in \Theta, \xi \in \Xi$, and $\phi_1 \in \argmax_{\phi} \LL_\A (\theta_1, \phi, \xi) $, there exists $\phi_2 \in \argmax_{\phi} \LL_\A (\theta_2 , \phi, \xi) $ such that $\|\phi_1 - \phi_2\| \leq \frac{L_{12}}{\mu} \|\theta_1 - \theta_2\|$.
    Indeed, based on smoothness assumption \eqref{lip6} and \eqref{lip4},
    \begin{equation*}
    \begin{aligned}
        \|\nabla_{\phi} \LL_\A (\theta_1, \phi_1, \xi) - \nabla_{\phi} \LL_\A(\theta_2, \phi_1, \xi )\| \leq  L_{21} \| \theta_1 - \theta_2\| ,  
    \\  \|\nabla_{\phi} \LL_\D (\theta_1, \phi_1, \xi) - \nabla_{\phi} \LL_\D(\theta_2, \phi_1, \xi)\| \leq  L_{12} \| \theta_1 - \theta_2\| .
        \end{aligned}
    \end{equation*}
    Since $\phi_2 \in \argmax_{\phi} \LL_\A (\theta_2 , \phi, \xi)$,  $\nabla_{\phi} \LL_\A (\theta_2, \phi_2, \xi ) = 0$. 
    Apply PL condition to $\nabla_{\phi} \LL_\A (\theta, \phi_2, \xi)$, 
    \begin{equation*}
     \begin{aligned}
        & \quad \max_{\phi} \LL_\A (\theta_1, \phi, \xi) -  \LL_\A (\theta_1, \phi_2, \xi) \\ & \leq \frac{1}{2 \mu} \| \nabla_{\phi} \LL_\A (\theta_1, \phi_2, \xi)\|^2 
         \\ & = \frac{1}{2 \mu} \| \nabla_{\phi} \LL_\A (\theta_1, \phi_2, \xi) - \nabla_{\phi} \LL_\A (\theta_2, \phi_2, \xi ) \|^2 
         \\ & \leq  \frac{L^2_{21}}{2\mu} \|\theta_1 - \theta_2\|^2 \quad \quad \text{ by \eqref{lip6}.}
     \end{aligned}
    \end{equation*}
    Since PL condition implies quadratic growth, we also have 
    \begin{equation*}
         \LL_\A (\theta_1, \phi_1, \xi) - \LL_\A (\theta_1, \phi_2, \xi) \geq \frac{\mu}{2} \| \phi_1 - \phi_2\|^2.
    \end{equation*}
    Combining the two inequalities above we obtain the Lipschitz stability for $\phi^*_\xi(\cdot)$, i.e., $$ \| \phi_1 - \phi_2 \|\leq \frac{L_{21}}{\mu}\| \theta_1 - \theta_2\|. $$
    Second, show that $\nabla_{\theta} V(\theta)$ can be directly evaluated at $\{\phi^*_\xi \}_{\xi \in \Xi}$. Inspired by Danskin's theorem, we first made the following argument, consider the definition of directional derivative. Let $\ell(\theta,\phi):=\nabla_\theta \mathbb{E}_{\xi,\tau} J_\D(\theta+\eta \hat{\nabla} J_\D(\tau), \xi)$. For a constant $\tau$ and an arbitrary direction $d$, 
\begin{align*}
    & \quad \ell(\theta+\tau d, \phi^*(\theta+\tau d))- \ell(\theta, \phi^*(\theta)))\\
    &=  \ell(\theta+\tau d, \phi^*(\theta+\tau d))- \ell(\theta+\tau d, \phi^*(\theta)) \\
    & \ \
    +\ell(\theta+\tau d, \phi^*(\theta))- \ell(\theta, \phi^*(\theta))\\
    &= \nabla_\phi \ell(\theta+\tau d, \phi^*(\theta))^\top\underbrace{[\phi^*(\theta+\tau d)-\phi^*(\theta))]}_{\Delta \phi} + o(\Delta \phi^2)\\
    &+ \tau \nabla_\theta \ell(\theta, \phi^*(\theta))^T d + o(d^2).
\end{align*}
Hence, a sufficient condition for the first equation is $\nabla_\phi \ell(\theta+\tau d, \phi^*(\theta))=0$, meaning that $\ell_D(\theta, \phi)$ and $\LL_\A(\theta, \phi, \xi)$ share the first-order stationarity at every $\phi$ when fixing $\theta$. Indeed, by \Cref{lemma:ift}, we have, the gradient is locally determined by 
 \begin{equation*}
     \begin{aligned}
        \nabla_{\theta}V &  = \E_{\xi \sim Q } [\nabla_{\theta} \LL_\D (\theta, \phi_\xi, \xi) + (\nabla_{\theta} \phi_\xi(\theta))^{\top}\nabla_{\phi}\LL_\D(\theta, \phi_\xi, \xi)] \\
        & = \E_{\xi \sim Q } \bigg[\nabla_{\theta} \LL_\D (\theta, \phi_\xi, \xi) - [( \nabla^2_{ \phi}\LL_\A (\theta, \phi, \xi))^{-1} \\
        &  \quad  \nabla^2_{ \phi \theta} \LL_\A (\theta, \phi, \xi)]^{\top}\nabla_{\phi}\LL_\D(\theta, \phi_\xi, \xi) \bigg] . 
     \end{aligned}
 \end{equation*}

 Given a trajectory $\tau:=(s^1,a_\D^t, a_\A^t, \ldots, a_\D^{H}, a_\A^H, s^{H+1})$, let $R_\D(\tau, \xi):=\sum_{t=1}^{H} \gamma^{t-1}r_\D(s_t, a_t, \xi)$ and $R_\D(\tau, \xi):=\sum_{t=1}^{H} \gamma^{t-1}r_\D(s_t, a_t, \xi)$. Denote by $\mu(\tau; \theta, \phi)$ the trajectory distribution,  that the log probability of $\mu$ is given by 
\begin{align*}
 \log \mu(\tau;\theta,\phi) & = \sum_{t=1}^H (\log\pi_\D(a^t_\D|s^t;\theta +\eta \hat{\nabla}_{\theta} J_\D (\tau) ) \\
  & +\log\pi_\A(a^t_\A|s^t;\phi)+\log P(s^{t+1}|a^t_\D, a^t_\A,s^t)
\end{align*}
According to the policy gradient theorem, we have
\begin{align*}
    &\nabla_\phi \LL_\D(\theta, \phi, \xi)= \mathbb{E}_{\mu}[R_\D(\tau, \xi)\sum_{t=1}^H \nabla_\phi \log(\pi_\A(a_\A^t|s^t;\phi))],\\
    &\nabla_\phi \LL_\A(\theta, \phi, \xi )=\mathbb{E}_{\mu}[R_\A(\tau, \xi)\sum_{t=1}^H \nabla_\phi \log(\pi_\A(a_\A^t|s^t;\phi))].
\end{align*}
By SC \Cref{ass:sc}, when $\nabla_\phi \LL_\A(\theta, \phi, \xi ) = 0$, there exists $c < 0$, $d$, such that $\nabla_\phi \LL_\D(\theta, \phi, \xi) = \mathbb{E}_{\mu}[c R_\A(\tau, \xi)\sum_{t=1}^H \nabla_\phi \log(\pi_\A(a_\A^t|s^t;\phi))] + \mathbb{E}_{\mu}[ \sum_{t=1}^H \gamma^{t-1}d \sum_{t=1}^H \nabla_\phi \log(\pi_\A(a_\A^t|s^t;\phi))] = 0$. Hence $ \nabla_{\theta}V  = \E_{\xi \sim Q } [\nabla_{\theta} \LL_\D (\theta, \phi_\xi, \xi)]$.
  
    Third, $V(\theta)$ is also Lipschitz smooth. As we notice that, $\ell_\D$ is Lipschitz smooth since $\E_{\xi\sim Q}$ is a linear operator, we have,  
    \begin{equation*}
    \begin{aligned}
       & \quad \ \| \nabla_{\theta}V (\theta_1) - \nabla_{\theta}V (\theta_2) \|  \\
        &\leq \| \nabla_{\theta} \E_{\xi \sim Q}\LL_\D (\theta_1, \phi_1, \xi) - \nabla_{\theta} \E_{\xi \sim Q}\LL_\D (\theta_2, \phi_2, \xi)\| \\ 
       & = \| \nabla_{\theta}\ell_\D (\theta_1, \phi_1) - \nabla_{\theta}\ell_\D (\theta_2, \phi_1) \\ & \quad  + \nabla_{\theta}\ell_\D (\theta_2, \phi_1) - \nabla_{\theta}\ell_\D (\theta_2, \phi_2)\| \\
       & \leq \| \nabla_{\theta}\ell_\D (\theta_1, \phi_1) - \nabla_{\theta}\ell_\D (\theta_2, \phi_1)\| 
       \\ & \quad + \|\nabla_{\theta}\ell_\D(\theta_2, \phi_1) - \nabla_{\theta}\ell_\D(\theta_2, \phi_2)\| \\ 
        & \leq L_{11} \|\theta_1 - \theta_2 \| + L_{12} \|\phi_1 - \phi_2\|  \\
        & \leq (L_{11} +\frac{L_{12}L_{21}}{\mu}) \|\theta_1 - \theta_2 \|, 
        \end{aligned}
    \end{equation*}
     which implies the Lipschitz constant $L = L_{11} + \frac{L_{12}L_{21}}{\mu}$.
\end{proof}

{Equipped with \Cref{ass:grad} we are able to unfold our main result \Cref{thm:main}, before which we show in \Cref{lemma:approxgradv} that $\phi^{*}_\xi$ can be efficiently approximated by the inner loop in the sense that $\nabla_\theta \E_{\xi \sim Q}\LL_\D (\theta^t, \phi^t_\xi(N_\A), \xi) \approx \nabla_\theta V(\theta^t)$, where $\phi^t_\xi(N_\A)$ is the last iterate output of the attacker policy.}

\begin{lemma}\label{lemma:approxgradv}
 Under assumptions   \ref{ass:sc}, \ref{asslip}, \ref{plass}, and \ref{ass:grad}, let $\rho : = 1 + \frac{\mu }{c L_{22}} \in (0, 1)$, $ \bar{L} = \max \{ L_{11}, L_{12}, L_{22}, L_{21}, V_{\infty} \}$ where $V_{\infty} :=  \max\{ \max\|\nabla V(\theta)\|, 1 \}$. 
 For all $\varepsilon > 0$, if the attacker learning iteration $N_\A$ and batch size $N_b$ are large enough such that 
 \begin{equation*}
     \begin{aligned}
         N_\A & \geq \frac{1}{\log \rho^{-1}}\log \frac{32 D_V^2 (2V_{\infty} + LD_{\Theta})^4 \bar{L} |c|G^2    }{ L^2 \mu^2\varepsilon^4} \\
        N_b & \geq \frac{32 \mu L_{21}^2 D_V^2 ( 2 V_{\infty} +  L D_{\Theta} )^4}{ |c| L_{22}^2 \sigma^2  \bar{L} L\varepsilon^4} , 
     \end{aligned}
 \end{equation*}
 then, for $z_t := \nabla_{\theta} \E_{\xi \sim Q}\LL_\D (\theta^t, \phi^t_\xi(N_\A), \xi) - \nabla_{\theta} V(\theta^t)$, 
 \begin{equation*}
      \mathbb{E} [\|z_t\|]  \leq \frac{ L\varepsilon^2}{ 4D_V ( 2 V_{\infty} +  L D_{\Theta} )^2}  ,
 \end{equation*}
 and 
 \begin{equation*}
      \mathbb{E} [\| \nabla_{\phi} \LL_\A (\theta^t, \phi^t_\xi (N), \xi)\|] \leq\varepsilon. 
 \end{equation*}
\end{lemma}

\begin{proof}[Proof of \Cref{lemma:approxgradv}]
     Fixing a $\xi \in \Xi$, due to Lipschitz smoothness,
     \begin{align*}
        & \quad \ \LL_\D (\theta^t, \phi^t_\xi(N), \xi) - \LL_\D(\theta^t, \phi^t_\xi(N-1), \xi )
        \\ & \leq \langle \nabla_{\phi} \LL_\D (\theta^t, \phi^t_\xi(N-1), \xi), \phi^t_\xi(N) - \phi^t_\xi(N-1)\rangle 
        \\ & + \frac{L_{22}}{2} \|\phi^t_\xi(N) - \phi^t_\xi(N-1)\|^2 .
     \end{align*}
     The inner loop updating rule ensures that when $\kappa_\A = \frac{1}{L_{21}}$, $ \phi^t_\xi(N) - \phi^t_\xi(N-1) = \frac{1}{L_{21}} \hat{\nabla}_{\phi} J_\A (\theta^t_\xi, \phi^t_\xi(N-1), \xi)$. 
     Plugging it into the inequality, we arrive at
     \begin{align*}
           & \ \quad \LL_\D (\theta^t, \phi^t_\xi(N), \xi) - \LL_\D(\theta^t, \phi^t_\xi(N-1), \xi ) 
          \\  & \leq   \frac{1}{L_{21}}\langle \nabla_{\phi} \LL_\D (\theta^t, \phi^t_\xi(N-1), \xi), \hat{\nabla}_{\phi}J_\A (\theta^t_\xi, \phi^t_\xi(N-1), \xi)\rangle
          \\ & 
          + \frac{L_{22}}{2L^2_{21}} \|\hat{\nabla}_{\phi} J_\A (\theta^t_\xi, \phi^t_\xi(N-1), \xi)\|^2.
     \end{align*}
     Therefore, we let $(\mathcal{F}^t_n)_{ 0\leq n \leq N}$ be the filtration generated by $\sigma( \{ \phi^t_\xi (\tau) \}_{\xi \in \Xi}|\tau \leq n)$ and take conditional expectations on $\mathcal{F}^t_n$:
     \begin{align*}
         & \quad \E [ V(\theta^t) - \ell_\D(\theta^t, \phi^t(N)) |  \mathcal{F}^t_{N-1}] 
         \\ & \leq V(\theta^t ) -  \ell_\D (\theta^t,\phi^t (N-1)) \\
         & \leq  \E_\xi \bigg[ \frac{1}{L_{21}}\langle \nabla_{\phi} \LL_\D , \nabla_{\phi} J_\A (\theta^t_\xi, \phi^t_\xi(N-1), \xi)\rangle 
         \\ &  \quad \quad  + \frac{L_{22}} {2L^2_{21}} \|\hat{\nabla}_{\phi} J_\A (\theta^t_\xi, \phi^t_\xi(N-1), \xi)\|^2 \bigg].
     \end{align*}
    By variance-bias decomposition, and \Cref{ass:grad} (b) and (c),
    \begin{align*}
    & \quad \ \E[ \| \hat{\nabla}_{\phi} J_\A (\theta^t_\xi, \phi^t_\xi(N-1), \xi) \|^2| \mathcal{F}^t_{N-1}] \\ & =  \E[ \| \hat{\nabla}_{\phi} J_\A (\theta^t_\xi, \phi^t_\xi(N-1), \xi) -  \nabla_{\phi} J_\A (\theta^t_\xi, \phi^t_\xi(N-1), \xi) \\ & \quad + \nabla_{\phi}J_\A (\theta^t_\xi, \phi^t_\xi(N-1), \xi) \|^2| \mathcal{F}^t_{N-1}] 
     \\ & =  \E[ \| (\hat{\nabla}_{\phi} -  \nabla_{\phi}) J_\A (\theta^t_\xi, \phi^t_\xi(N-1), \xi)\|^2 | \mathcal{F}^t_{N-1} ]  \\
     & \quad +  \E [\| \nabla_{\phi} J_\A (\theta^t_\xi, \phi^t_\xi(N-1), \xi) \|^2 | \mathcal{F}^t_{N-1}] 
      \\ & \quad  + \E [ 2 \langle (\hat{\nabla}_{\phi} -  \nabla_{\phi}) J_\A (\theta^t_\xi, \phi^t_\xi(N-1), \xi) , \\   & \quad \quad    \nabla_{\phi} J_\A (\theta^t_\xi, \phi^t_\xi(N-1), \xi) \rangle |\mathcal{F}^t_{N-1} ]
     \\ & \leq \frac{\sigma^2}{N_b} + \| \nabla_{\phi} J_\A (\theta^t_\xi, \phi^t_\xi(N-1), \xi) \|^2    .
    \end{align*}
    Applying the PL condition (\Cref{plass}), and \Cref{ass:grad} (a) we obtain
      \begin{align*}
        & \quad  \E[ V(\theta^t) - \ell_\D(\theta, \phi^t(N))| \phi^{N-1}] \\
        & \quad 
        - V(\theta^t) -  \ell_\D(\theta, \phi^t(N-1))  
        \\ & \leq  \E_\xi \bigg[\frac{1}{ L_{21}}\langle \nabla_{\phi} \LL_\D , \nabla_{\phi}\LL_\A (\theta^t, \phi^t_\xi(N-1), \xi)\rangle 
        \\ & + \frac{L_{22}}{2 L_{21}^2} ( \frac{\sigma^2}{ N_b}+ \|\nabla_{\phi}\LL_\A (\theta^t, \phi^t_\xi(N-1), \xi) \|^2  ) \bigg]
        \\ & =  \E_\xi \bigg[  -\frac{1}{2L_{22}} \| \nabla_{\phi } \LL_\D \|^2 + \\ & \quad  \frac{1}{2 L_{22}}\|\nabla_{\phi} (\LL_\D    +  \frac{L_{22}}{L_{21}} \LL_\A) (\theta^t, \phi^t_\xi(N-1), \xi) \|^2 + \frac{L_{22}\sigma^2 }{2 L_{21}^2 N_b} \bigg]
        \\ & \leq  \frac{\mu}{c L_{21}} ( \max_{\phi} \ell_\D (\theta^t, \phi)  - \ell_\D (\theta^t, \phi^t (N-1)) ) + \frac{L_{22} \sigma^2  }{2 L_{21}^2 N_b} , 
       \end{align*}
    
   rearranging the terms yields 
    \begin{align*}
        & \quad  \E [ V(\theta^t) - \ell_\D(\theta^t,\phi^t (N)) | \mathcal{F}^t_n] \\ & \leq \rho ( V(\theta^t) - \ell_\D(\theta^t,\phi^t (N-1)))  + 
 \frac{L_{22}\sigma^2  }{2 L_{21}^2 N_b} , 
    \end{align*}
    where we use the fact that $ - \max_{\phi} \ell_\D (\theta^t, \phi) \leq -V(\theta^t)$.
    Telescoping the inequalities from $\tau = 0$ to $\tau = N$, we arrive at 
    \begin{align*}
        & \quad  \E [ V(\theta^t) - \ell_\D(\theta^t, \phi^t(N))] \\
         & \leq \rho^N (V(\theta^t) - \ell_\D(\theta^t, \phi^t(0))) + \frac{1 - \rho^N }{1 - \rho} \left(\frac{L_{22}\sigma^2  }{2 L_{21}^2 N_b}\right).
    \end{align*}   

PL-condition implies quadratic growth, we also know that $ V(\theta^t) -  \ell_\D (\theta^t, \phi^t(N)) \leq \E_\xi \frac{1}{2\mu} \| \nabla_{\phi} \LL_\D (\theta^t, \phi^t_\xi(N), \xi)\|^2 \leq \frac{1}{2\mu} G^2 $, by  \Cref{ass:sc}, 
 \begin{align*}
  & \quad  \|   \phi^*_\xi (\theta^t) - \phi^t_\xi (N) \|^2 \\
    &\leq \frac{2}{\mu} (\LL_\A(\theta^t, \phi^*_\xi, \xi) - \LL_\A(\theta^t, \phi^t_\xi (N), \xi)) 
    \\ & \leq  \frac{2|c| }{\mu} \big\vert \LL_\D (\theta^t, \phi^*_\xi, \xi) - \LL_\D (\theta^t, \phi^t_\xi(N), \xi ) \big\vert
 \end{align*}
 Hence, with Jensen inequality and choice of $N_\A$ and $N_b$, 
     \begin{align*}
         \E [ \|z_t \|] & = \E [ \| \nabla_{\theta} V(\theta^t) -  \E_\xi \nabla_{\theta} \LL_\D (\theta^t, \phi^t_\xi (N_\A), \xi )\|] \\
         & \leq L_{12} \E [ \|\phi^t_\xi (N_\A) -  \phi^*_\xi \|]  \\
        & \leq  L_{12} \sqrt{\frac{2 |c|}{\mu}\E [ V(\theta^t) - \ell_\D(\theta^t, \phi^t(N_\A))] } \\
        & \leq   L_{12} \sqrt{\frac{ |c|}{\mu^2} \rho^{N_\A} G^2 + (1 - \rho^{N_\A}) \frac{|c| L_{22}^2 \sigma^2 }{\mu L_{21}^2 N_b}} .  
     \end{align*}
Now we adjust the size of $N_\A$ and $N_b$ to make $\E [ \|z_t \|]$ small enough, to this end, we set
\begin{align*}
  \rho^{N_\A} \frac{|c| G^2}{ \mu^2} & \leq \frac{\varepsilon^4 L^2  }{32 D_V^2 (2V_{\infty} + L D_{\Theta})^4 \bar{L} } \\
   \frac{|c| L_{22}^2 \sigma^2 }{  L_{21}^2 N_b} & \leq \frac{\varepsilon^4 L^2 \mu^2 }{32 D_V^2 (2V_{\infty} + LD_{\Theta})^4 \bar{L} }, 
\end{align*} 
which further indicates that 
\begin{align*}
    N_\A &   \geq \frac{1}{\log \rho^{-1}}\log \frac{32 D_V^2 (2V_{\infty} + LD_{\Theta})^4 \bar{L} |c|G^2    }{ L^2 \mu^2\varepsilon^4} \\
    N_b &  \geq \frac{32 \mu L_{21}^2 D_V^2 ( 2 V_{\infty} +  L D_{\Theta} )^4}{ |c| L_{22}^2 \sigma^2  \bar{L} L\varepsilon^4} .
\end{align*}
In the setting above, it is not hard to verify that
\begin{align*}
 \E [\|z_t \|] \leq \frac{ L\varepsilon^2}{ 4D_V ( 2 V_{\infty} +  L D_{\Theta} )^2} \leq\varepsilon. 
\end{align*}
Also note that $\| \nabla_{\phi} \LL_\A (\theta^t, \phi^t_\xi(N_\A), \xi) \| = \| \nabla_{\phi}  \LL_\A (\theta^t, \phi^t_\xi(N_\A), \xi) - \nabla_{\phi} \LL_\A (\theta^t, \phi^*_\xi, \xi)\|$, given the proper choice  of $N_\A$ and $N_b$, one has
      \begin{align*}
           & \quad \E \| \nabla_{\phi}  \LL_\A (\theta^t, \phi^t_\xi(N_\A), \xi) - \nabla_{\phi} \LL_\A (\theta^t, \phi^*_\xi, \xi)\|
          \\ & \leq  L_{21} \E [ \|\phi^t_\xi (N_\A) -  \phi^*_\xi \| ]  \leq  \frac{ L\varepsilon^2}{ 4D_V ( 2 V_{\infty} +  L D_{\Theta} )^2}  \leq \varepsilon ,
      \end{align*}
which implies $\xi$-wise inner loop stability for algorithm \ref{algo:meta-sl}.
\end{proof}

{Now we are ready to provide the convergence guarantee of the first-order outer loop in \Cref{thm:outer}, as well as the complexity estimates of the numbers of inner loops, outer loops, and sampled trajectory batch sizes with respect to the error $\varepsilon$.
Essentially, we aim to show that the first condition for $\varepsilon$-meta-FOSE holds for a small number $\varepsilon$; when analyzing the first-order iterations, a key step is to take care of the residue error introduced by imperfect policy gradient and best-response attacks, we omit the variance of sampling from the prior $Q(\cdot)$; this will introduce a sample complexity on the sample size of attack types, but does not affect the eventual order.
}

\begin{theorem}\label{thm:outer}
 Under assumptions  \ref{ass:sc}, \Cref{asslip}, and \ref{ass:grad}, let  $\kappa_\A = \frac{1}{L_{22}}$ and $\kappa_\D = \frac{1}{L}$, if $N_\D, N_\A,$ and $N_b$ are large enough,
 \begin{align*}
     &  N_\D \geq N_\D(\varepsilon) \sim  \mathcal{O}(\varepsilon^{-2}) \quad N_\A  \geq N_\A (\varepsilon) \sim \mathcal{O} ( \log\varepsilon^{-1}), \\
      & \quad N_b \geq N_b (\varepsilon)\sim \mathcal{O} (\varepsilon^{-4})
 \end{align*}
 then there exists $t \in \mathbb{N}$ such that $(\theta^t, \{\phi^{t}_\xi(N_\A) \}_{\xi \in \Xi})$ is $\varepsilon$-meta-FOSE.
\end{theorem}

\begin{proof} 
According to the update rule of the outer loop, (here we omit the projection analysis for ease of exposition)
\begin{align*} 
 \theta^{t+1} - \theta^t = \frac{1}{L}\hat{\nabla}_{\theta} \ell_\D (\theta^t, \phi^t(N_\A)), 
\end{align*}
one has, due to unbiasedness assumption, let $(\mathcal{F}_t)_{ 0 \leq t\leq N_\D}$ be the filtration generated by $\sigma(\theta^t| k \leq t)$ 
\begin{align*}
  \\  & \quad    \E[ \langle  \nabla_{\theta} \ell_\D (\theta^t, \phi^t(N_\A)),  \theta^{t+1} - \theta^t \rangle | \mathcal{F}_t ] \\ & =  \frac{1}{L}\E[ \| 
 \nabla_{\theta} \ell_\D(\theta^t, \phi^t(N_\A))\|^2 |\mathcal{F}_t] 
 \\ &  = L  \E \| \theta^{t+1} - \theta^t\|^2 | \mathcal{F}_t ], 
\end{align*}
which leads to
\begin{align*}
  \\ & \quad   \E[ \langle \nabla_{\theta} \ell_\D(\theta^t, \phi^*) , \theta^{t+1} - \theta^t \rangle | \mathcal{F}_t ]  \\ & = \E[ \langle z_t, \theta^t - \theta^{t+1}\rangle| \mathcal{F}_t ] + L \E[ \| \theta^{t+1} - \theta^t\|^2\|]  . 
\end{align*}
 Since $V(\cdot)$ is $L$-Lipschitz smooth,
 \begin{equation} \label{telev}
 \begin{aligned}
  & \quad  \E[ V( \theta^{t}) -  V(\theta^{t+1} )]  \\ & \leq  \E[ \langle  \nabla_{\theta} V(\theta^t ), \theta^{t} - \theta^{t+1} \rangle] + \frac{L}{2} \E[ \| \theta^{t+1} - \theta^t\|^2 ]
   \\ & \leq  \E[\langle z_t, \theta^{t+1}  - \theta^{t}\rangle] - \E[\langle \nabla_{\theta} \ell_\D(\theta^t, \phi^t(N_\A) ),  \theta^{t+1} - \theta^t\rangle] 
   \\ & + \frac{L}{2} \E[ \| \theta^{t+1} - \theta^t\|^2 ] 
    \\ &  \leq \E[\langle z_t, \theta^{t+1}  - \theta^{t}\rangle] - \frac{L}{2} \E[ \| \theta^{t+1} - \theta^t\|^2 ] .
 \end{aligned}
 \end{equation}

Fixing a $\theta \in \Theta$, let $e_t :=  \langle \nabla_{\theta} \ell_\D(\theta^t, \phi^t(N_\A) ), \theta - \theta^t \rangle$, we have
\begin{equation} \label{boundet}
 \begin{aligned}
     \E[e_t | \mathcal{F}_t ] & =  L \E[  \langle \theta^{t+1} - \theta^t   , \theta - \theta^{t} \rangle | \mathcal{F}_t] 
     \\ &  =  \E[ \langle  \nabla_{\theta} \ell_\D (\theta^t, \phi^t(N_\A)) - \nabla_{\theta} V(\theta^t), \theta^{t+1} - \theta^{t}  \rangle
      \\  + & \langle  \nabla_{\theta} V(\theta^t),  \theta^{t+1} - \theta^t \rangle]  + L \E[  \langle \theta^{t+1} - \theta^t,  \theta - \theta^{t+1} \rangle ]
    \\ & \leq   \E [ (  \|z_t\| + V_{\infty} + LD_{\Theta}) \| \theta^{t+1} - \theta^t \|]
 \end{aligned}
 \end{equation}

  By the choice of $N_b$, we have, since $V_{\infty} = \max\{ \max_{\theta} \|\nabla V(\theta)\|, 1\}$,
 \begin{align*}
      \E [ \|z_t\|] \leq L_{12} \E [ \| \phi^N - \phi^*\|] \leq \frac{ L\varepsilon^2}{ 4D_V ( 2 V_{\infty} +  L D_{\Theta} )} \leq V_{\infty} .
 \end{align*}
Thus, the relation \eqref{boundet} can be reduced to 
\begin{equation*}
\E [ e_t] \leq (2 V_{\infty}+ LD_{\Theta} ) \E [\|\theta^{t+1} - \theta^t\|]. 
\end{equation*}
Telescoping \eqref{telev} yields 
\begin{align*}
& \quad - D_V \leq \E[ V(\theta^0) - V(\theta^{N_\D})] \\ & \leq  D_{\Theta} \sum_{t=0}^{T-1} \E [ \| z_t \|] -  \frac{L}{2(2 V_{\infty} + L D_{\Theta})^2} \E[\sum_{t=0}^{T-1} \E[ e^2_t | \mathcal{F}_t ].
\end{align*}

Thus, setting $N_\D \geq \frac{ 4D_V (2 V_{\infty} + LD_{\Theta})^2}{L\varepsilon^2 }$, and then by Lemma \ref{lemma:approxgradv}, we obtain that, 
\begin{align*}
     \frac{1}{N_\D } \sum_{t=0}^{N_\D -1} \E [e^2_t ] 
     \leq \frac{\varepsilon^2 }{2}  + \frac{ 2 D_V(2V_{\infty}+ LD_{\Theta})^2  }{L N_\D }  \leq \varepsilon^2
\end{align*}
which implies there exists $t \in \{0, \ldots, N_\D - 1\}$ such that $\E [ e^2_t] \leq\varepsilon^2$.


\end{proof}

\subsection{Proof of \Cref{main_prop:generalization}}

 For two distributions $P$ and $Q$, defined over the sample space $\Omega$ and $\sigma$-field $\mathcal{F}$, the total variation between $P$ and $Q$ is $\|P- Q\|_{TV} := \sup_{ U \in \mathcal{F}} |P(U) - Q(U)|$. The celebrated result shows the following characterization of total variation,
 \begin{equation*}
    \|P-Q\|_{T V}=\sup _{f: 0 \leq f \leq 1} \mathbb{E}_{x \sim P}[f(x)]-\mathbb{E}_{x \sim Q}[f(x)].
 \end{equation*}

Since $q^{\theta}_i$ is factorizable, we have Lemma \ref{lem:marginal} to eliminate $\| q^{\theta}_i - q^{\theta}_{m+1}\|_{TV}$ dependence on $\theta$ by upper bounding it using another pair of mariginal distributions. 

 \begin{lemma}\label{lem:marginal}
      For any $\theta \in \Theta$, there exist marginals $d_i, d_{m+1}:  (S  \times  A_\A \times S)^{H-1} \times S \to [0,1]$ total variation $\|q^{\theta}_i - q^{\theta}_{m+1}\|_{TV}$ can be bounded by  $\|d_i - d_{m+1}\|_{TV}$.
 \end{lemma} 

 \begin{proof}
     By factorization, for a trajectory $\tau$, any $\theta \in \Theta$, and any type index $i = 1, \ldots, m+1$: 
     \begin{align*} 
       & q^{\theta}_i ( \tau ) =  \prod_{t=1}^{H-1} \pi_\D ( a_\D^t| s^t; \theta )  \\ & \quad \quad  \prod_{t=1}^{H-1} \pi_{\xi_i}( a_\A^t | s^t, \phi_i) \prod_{t=1}^{H-1}  \mathcal{T}(s^{t+1}|s^t, a_\D^t, a_\A^t), 
     \end{align*}
     thus, by the inequality of product measure, 
     \begin{align*}
         &  \| q^{\theta}_i - q^{\theta}_{m+1}\|_{TV} \leq  \sum_{t=1}^{H-1} \underbrace{\|\pi_\D (\cdot | s_t; \theta )  - \pi_\D (\cdot | s_t; \theta ) \|_{TV}}_{0} \\ & \quad  + \| d_i - d_{m+1} \|_{TV},
     \end{align*}
     where $d_i$ and $d_{m+1}$ are the residue factors of $q^\theta_i$ and $q^{\theta}_{m+1}$ after removing $\pi_\D(\cdot | s^t; \theta)$. 
 \end{proof}


{The upper bound on the total variation of trajectory distributions, which determines the gradient adaptation, leads to an upper bound on the difference $| \hat{V}_{m+1} (\theta) - \hat{V} (\theta)|$, characterizing the generalization error.} 

\begin{proof}[Proof of \Cref{main_prop:generalization}]
    We start with the decomposition of the generalization error, for an arbitrary attack type $\xi_i$, $i = 1, \ldots, m$, fixing a policy $\theta \in \Theta$ determines jointly with each $\phi_i$ the trajectory distribution $q^{\theta}_i$.
    Denoting the one-step adaptation policy $\theta^{\prime}(\tau) = \theta + \eta \nabla J_\D (\tau)$ as a function of trajectory $\tau$, we have the following decomposition,
    \begin{align*} 
         & \quad \hat{V}_{m+1}  ( \theta) - \hat{V} (\theta) 
            = \E_{\tau_{m+1} \sim q^{\theta}_{m+1}} J_\D (\theta^{\prime}(\tau_{m+1}), \phi_{m+1}, \xi_{m+1}) \\
            & \quad - \frac{1}{m} \sum_{i=1}^m \E_{\tau_{i} \sim q^{\theta}_{i}} J_\D (\theta^{\prime}(\tau_{i}), \phi_{i}, \xi_{i})  \\
        &  =   \left. \begin{array}{c}
            \E_{\tau_{m+1} \sim q^{\theta}_{m+1}}  J_\D (\theta^{\prime}(\tau_{m+1}), \phi_{m+1}, \xi_{m+1})    \\
            -  \frac{1}{m} \sum_{i=1}^m \E_{\tau_{m+1} \sim q^{\theta}_{m+1}} J_\D (\theta^{\prime}(\tau_{m+1}), \phi_i, \xi_i)  
        \end{array} \right \} (i) \\   
        & \quad  \left.\begin{array}{c}
            +  \frac{1}{m} \sum_{i=1}^m \E_{\tau_{m+1} \sim q^{\theta}_{m+1}} J_\D (\theta^{\prime}(\tau_{m+1}), \phi_i, \xi_i)  \\
               -  \frac{1}{m} \sum_{i=1}^m \E_{\tau_{i} \sim q^{\theta}_{i}} J_\D (\theta^{\prime}(\tau_{i}), \phi_{i}, \xi_{i}) .
        \end{array}  \right\} (ii)
    \end{align*}
  We assume $(\tau_{m+1}, \tau_i)$ is drawn from a joint distribution which has marginals $q^{\theta}_{m+1}$ and $q^{\theta}_i$ and is corresponding to the maximal coupling of these two. Then,
    \begin{equation*}
        \tau_{m+1} \sim q^{\theta}_{m+1}, \quad  \tau_i \sim q^{\theta}_i, \quad  \mathbb{P} (\tau_{m+1} \neq \tau_i) = \| q^{\theta}_i - q^{\theta}_{m+1}\|_{TV}, 
    \end{equation*}
    if $\tau_{m+1}$ disagrees with $\tau_i$, for $(ii)$, we have, since $J_\D^{\theta} $ is Lipschitz with respect to $\theta$ (\Cref{ass:grad}(a)),
    \begin{align*}
         & \quad  \| J_\D ( \theta^{\prime}(\tau_{m+1}) , \phi_i, \xi_i) - J_\D ( \theta^{\prime}(\tau_i), \phi_i,  \xi_{i}) \|  \\
        & \leq  \eta G \| \hat{\nabla}_{\theta} J_\D( \tau_{m+1}) - \hat{\nabla}_{\theta} J_\D( \tau_i)  \|   \\
        & \leq 2 \eta G^2 ,
    \end{align*}
    as a result, denoting  the maximal coupling of $q^{\theta}_{m+1}$ and  $q^{\theta}_i$ as $\prod$ gives,
    \begin{align*}
      &    \E_{\tau_{m+1} \sim q^{\theta}_{m+1}} J_\D ( \theta^{\prime}(\tau_{m+1}), \phi_i, \xi_i) - \E_{\tau_{i} \sim q^{\theta}_{i}} J_\D ( \theta^{\prime}(\tau_i), \phi, \xi_{i})   \\
    & = \E_{( \tau_{m+1}, \tau_i) \sim \prod  }[J_\D ( \theta^{\prime}(\tau_{m+1}), \phi_i, \xi_i)  -  J_\D ( \theta^{\prime}(\tau_i), \phi, \xi_{i}) ] \\
     & \leq 2 \eta G^2 \| q^{\theta}_{m+1} -  q^{\theta}_i \|_{TV} \leq 2 \eta G^2  \|   d_i - d_{m+1}\|_{TV} ,
    \end{align*}
   where the last inequality is due to Lemma \ref{lem:marginal}.
    Averaging the $m$ empirical $\xi_i$'s yeilds the result:
    \begin{align*}
        (ii) \leq  \frac{2 \eta G^2}{m} \sum_{i=1}^m \| d_i - d_{m+1}\|_{TV} .
    \end{align*}


Since the trajectory distribution is a product measure, the difference between $q^{\theta}_i$ and $q^{\theta}_{m+1}$ only lies by attacker's type, $\|q^{\theta^{\prime}(\tau_{m+1})}_{m+1} -  q^{\theta^{\prime}(\tau_{m+1})}_i\|_{TV} =  \| q^{\theta}_{m+1} -  q^{\theta}_i\|_{TV} \leq \| d_{m+1} - d_{i}\|_{TV}$. 
      
Now we bound $(i)$, for ease of exposition we let $q^{\prime\prime} = q^{\theta^{\prime}(\tau_{m+1})}_{m+1}$ and $q^{\prime}_i := q^{\theta^{\prime}(\tau_{m+1})}_{i}$.
By the finiteness of total trajectory reward $R(\tau)$ for any trajectory $\tau$, $R(\tau) \leq \frac{1 - \gamma^{H}}{1 - \gamma }$, hence,

\begin{align*}
(i)  & =  \E_{\tau_{m+1} \sim q^{\theta}_{m+1}}  J_\D (\theta^{\prime}(\tau_{m+1}), \phi_{m+1}, \xi_{m+1}) \\
& \quad -  \frac{1}{m} \sum_{i=1}^m \E_{\tau_{m+1} \sim q^{\theta}_{m+1}} J_\D (\theta^{\prime}(\tau_{m+1}), \phi_i, \xi_i) \\
    &  = \E_{\tau_{m+1} \sim q^{\theta}_{m+1} } \left[\E_{ \tau^{\prime\prime} \sim q^{\prime\prime}} R_\D (  \tau^{\prime \prime}) -  \frac{1}{m} \sum_{i=1}^m \E_{\tau^{\prime}_i \sim q^{\prime}_i} R_\D (\tau^{\prime }_i) \right]   \\
    & \leq  \E_{\tau_{m+1} \sim q^{\theta}_{m+1}}  \frac{1 - \gamma^H}{ 1 - \gamma} \| q^{\prime\prime}_{m+1} - \frac{1}{m} \sum_{i=1}^m q^{\prime}_i \|_{TV}    \\
    & \leq  \frac{1 - \gamma^H}{ 1 - \gamma } \| d_{m+1} - \frac{1}{m} \sum_{i=1}^m d_i \|_{TV} . 
\end{align*}

\end{proof}

\section{Algorithm}
\label{app:algo}
\newcommand{\tp}{\mathsf{T}}
\setcounter{equation}{0}
\renewcommand{\theequation}{\Alph{section}.\arabic{equation}}
This section elaborates on the algorithmic details behind the proposed meta-Stackelberg learning.  To begin with, we first review the policy gradient method \cite{sutton_PG} in RL and its Monte-Carlo estimation. To simplify our exposition, we fix the attacker's policy $\phi$, and then the Markov game reduces to a single-agent MDP, where the optimal policy to be learned is the defender's $\theta$. 
\subsection{Policy Gradient}
The idea of the policy gradient method is to apply gradient ascent to the value function $J_\D$. Following \cite{sutton_PG}, we obtain $\nabla_\theta J_\D:=\E_{\tau\sim q(\theta)}[g(\tau;\theta)]$, where $g(\tau;\theta)=\sum_{t=1}^H \nabla_\theta \log \pi(a_\D^t|s^t;\theta)R(\tau)$ and $R(\tau)=\sum_{t=1}^H \gamma^t r(s^t, a_{\D}^t)$. Note that for simplicity, we suppress the parameter $\phi, \xi$ in the trajectory distribution $q$, and instead view it as a function of $\theta$. In numerical implementations, the policy gradient $\nabla_\theta J_\D$ is replaced by its Monte-Carlo (MC) estimation using sample trajectory. Suppose a batch of trajectories $\{\tau_i\}_{i=1}^{N_b}$, and $N_b$ denotes the batch size, then the MC estimation is 
\begin{equation}
    \hat{\nabla}_\theta J_\D(\theta,\tau):=1/{N_b} \sum_{\tau_i} g(\tau_i;\theta).
    \label{eq:mc-pg}
\end{equation}
 The same deduction also holds for the attacker's problem when fixing the defense $\theta$. 

\subsection{Debiased Meta-Learning and Reptile} 
Fixing the attacker's policy $\phi$, the defender's problem under one-step gradient adaptation reduces to the following.
\begin{equation}
    \max_{\theta} \E_{\xi \sim Q(\cdot)}\E_{\tau\sim q(\theta)}[J_\D(\theta+\eta \hat{\nabla}_\theta J_\D(\tau), \phi, \xi)].
    \label{eq:meta-def-sgd}
\end{equation}
 To apply the policy gradient method to (\ref{eq:meta-def-sgd}), one needs an unbiased estimation of the gradient of the objective function in (\ref{eq:meta-def-sgd}). Consider the gradient computation using the chain rule:
\begin{equation}
   \begin{aligned}
    &\nabla_\theta \E_{\tau\sim q(\theta)}[J_\D(\theta+\eta \hat{\nabla}_\theta J_\D(\tau), \phi, \xi )]\\
    &=\E_{\tau\sim q(\theta)}\{\underbrace{\nabla_\theta J_\D(\theta+\eta\hat{\nabla}_\theta J_\D(\tau), \phi, \xi)(I+\eta \hat{\nabla}^2_\theta J_D(\tau))}_{\text{\ding{172}}}\\
    &\quad  +\underbrace{J_\D(\theta+\eta \hat{\nabla}_\theta J_\D(\tau))\nabla_\theta \sum_{t=1}^H \log\pi(a^t|s^t;\theta)}_{\text{\ding{173}}}\}.
\end{aligned} 
\label{eq:chain}
\end{equation}
The first term results from differentiating the integrand $J_\D(\theta+\eta \hat{\nabla}_\theta J_\D(\tau), \phi, \xi )$ (the expectation is taken as integration), while the second term is due to the differentiation of $q(\theta)$. One can see from the first term that the above gradient involves a Hessian $\hat{\nabla}^2 J_\D$, and its sample estimate is given by the following. For more details on this Hessian estimation, we refer the reader to \cite{fallah2021convergence}.
\begin{align}
    \hat{\nabla}^2 J_\D(\tau)=\frac{1}{N_b}\sum_{i=1}^{N_b} [ g(\tau_i;\theta)\nabla_\theta \log q(\tau_i;\theta)^\tp+\nabla_\theta g(\tau_i;\theta)]
    \label{eq:hessian-est}
\end{align}
Finally, to complete the sample estimate of $\nabla_\theta \E_{\tau\sim q(\theta)}[J_\D(\theta+\eta \hat{\nabla}_\theta J_\D(\tau), \phi, \xi )]$, one still needs to estimate $\nabla_\theta J_\D(\theta+\eta\hat{\nabla}_\theta J_\D(\tau), \phi, \xi)$ in the first term. To this end, we need to first collect a batch of sample trajectories ${\tau'}$ using the adapted policy $\theta'=\theta+\eta \hat{\nabla}_\theta J_D(\tau)$. Then, the policy gradient estimate of $\hat{\nabla}_\theta J_\D(\theta')$ proceeds as in (\ref{eq:mc-pg}). To sum up, constructing an unbiased estimate of (\ref{eq:chain}) takes two rounds of sampling. The first round is under the meta policy $\theta$, which is used to estimate the Hessian (\ref{eq:hessian-est}) and to adapt the policy to $\theta'$. The second round aims to estimate the policy gradient $\nabla_\theta J_\D(\theta+\eta\hat{\nabla}_\theta J_\D(\tau), \phi, \xi)$ in the first term in (\ref{eq:chain}).

To avoid Hessian estimation in implementation, we employ a first-order meta-learning algorithm called Reptile~\cite{nichol2018first}. The gist is to simply ignore the chain rule and update the policy using the gradient $\nabla_\theta J_\D(\theta',\phi,\xi)|_{\theta'=\theta+\eta\hat{\nabla}_\theta J_\D(\tau)}$. Naturally, without the Hessian term, the gradient in this update is biased, yet it still points to the ascent direction as argued in \cite{nichol2018first}, leading to effective meta policy. The advantage of Reptile is more evident in multi-step gradient adaptation. Consider a $l$-step gradient adaptation, the chain rule computation inevitably involves multiple Hessian terms (each gradient step brings a Hessian term) as shown in \cite{fallah2021convergence}. In contrast, Reptile only requires first-order information, and the meta-learning algorithm ($l$-step adaptation) is given by \Cref{algo:meta-rl}. 
\begin{algorithm}[ht]
\begin{algorithmic}[1]
\STATE \textbf{Input: } the type distribution $Q$, step size parameters $\kappa,\eta$
\STATE \textbf{Output: }$\theta^T$
\STATE randomly initialize $\theta^0$
\FOR{iteration $t=1$ to $T$}
\STATE Sample a batch $\hat{\Xi}$ of $K$ attack types from $Q(\xi)$;
\FOR{each $\xi\in \hat{\Xi}$}
\STATE $\theta_\xi^{t}(0) \leftarrow \theta^{t}$
\FOR{$k=0$ to $l-1$}
\STATE Sample a batch trajectories ${\tau}$ of the horizon length $H$ under $\theta^{t}_\xi(k)$;
\STATE Evaluate $\hat{\nabla}_\theta J_\D(\tau)$ using MC in (\ref{eq:mc-pg});
\STATE $\theta_\xi^{t}(k+1) \leftarrow \theta_\xi^{t}(k) + \kappa \hat{\nabla}_\theta J_\D(\tau)$
\ENDFOR
\ENDFOR
\STATE Update $\theta^{t+1} \leftarrow \theta^{t}+1/K \sum_{\xi\in \hat{\Xi}}(\theta^t_\xi(l)-\theta^t)$;
\ENDFOR
\end{algorithmic}
 \caption{Reptile Meta-Reinforcement Learning ($l$-step adaptation)}
 \label{algo:meta-rl}
\end{algorithm}
\section{Experiment Setup}
\label{app:exp-setup}

\paragraph{Datasets}We consider two datasets: MNIST \cite{lecun1998gradient} and CIFAR-10 \cite{krizhevsky2009learning}, and default $i.i.d.$ local data distributions, where we randomly split each dataset into $n$ groups, each with the same number of training samples. MNIST includes 60,000 training examples and 10, 000 testing examples, where each example is a 28$\times$28 grayscale image, associated with a label from 10 classes. CIFAR-10 consists of 60,000 color images in 10 classes of which there are 50, 000 training examples and 10,000 testing examples. 
For the {\em non-i.i.d.} setting (see Fig.~\ref{fig:ablation}(d) in Appendix~\ref{app:add-exp}), we follow the method of~\cite{fang2020local} to quantify the heterogeneity of the data. We split the workers into $C=10$ (for both MNIST and CIFAR-10) groups and model the \textit{non-i.i.d.} federated learning by assigning a training instance with label $c$ to the $c$-th group with probability $q$ and to all the groups with probability $1-q$. A higher $q$ indicates a higher level of heterogeneity.

\paragraph{Federated learning setting}We use the following default parameters for the FL environment: local minibatch size = 128, local iteration number = 1, learning rate = 0.05, number of workers = 100, number of backdoor attackers = 5, number of untargeted model poisoning attackers = 20, subsampling rate = $10\%$, and the number of FL training rounds = 500 (resp. 1000) for MNIST (resp. CIFAR-10). For MNIST, we train a neural network classifier of 8×8, 6×6, and 5×5 convolutional filter layers with ReLU activations followed by a fully connected layer and softmax output. For CIFAR-10, we use the ResNet-18 model~\cite{he2016deep}.
We implement the FL model with PyTorch~\cite{paszke2019pytorch} and run all the experiments on the same 2.30GHz Linux machine with 16GB NVIDIA Tesla P100 GPU. 
We use the cross-entropy loss as the default loss function and stochastic gradient descent (SGD) as the default optimizer. For all the experiments except Fig.~\ref{fig:ablation}(c) and~\ref{fig:ablation}(d), we fix the initial model and random seeds of subsampling for fair comparisons.

\paragraph{Baselines} We evaluate our defense method against various state-of-the-art attacks, including non-adaptive and adaptive untargeted model poison attacks (i.e., IPM~\cite{xie2020fall}, LMP~\cite{fang2020local}, RL~\cite{li2022learning}), as well as backdoor attacks (BFL~\cite{bagdasaryan2020backdoor} without model replacement, BRL~\cite{li2023learning}, with tradeoff parameter  $\lambda=0.5$, DBA~\cite{xie2019dba} where each selected attacker randomly chooses a sub-trigger as shown in Fig.~\ref{fig:cifar10_dba}, PGD attack~\cite{wang2020attack} with a projection norm of 0.05), and a combination of both types. To establish the effectiveness of our defense, we compare it with several strong defense techniques. These baselines include defenses implemented during the training stage, such as Krum~\cite{blanchard2017machine}, ClipMed~\cite{yin2018byzantine,sun2019can,li2022learning} (with norm bound 1), FLTrust~\cite{cao2020FLTrust} with 100 root data samples and bias $q=0.5$, training stage CRFL~\cite{xie2021crfl} with norm bound of 0.02 and noise level $1e-3$ as well as post-training defenses like NeuroClip~\cite{wang2023mm} and Prun~\cite{wu2020mitigating}. We use the  original clipping thresholds 7 in~\cite{wang2023mm} and set the default Prun number to 256.

\begin{table}[!h]
\centering
\begin{tabular}{@{\extracolsep{1pt}}lccc}
    \toprule
        Attack type & \multicolumn{1}{c}{Category} & \multicolumn{1}{c}{Adaptivity} \\
        \midrule
        IPM \cite{xie2020fall} & untargeted model poisoning & non-adaptive \\
        LMP \cite{fang2020local} & untargeted model poisoning & non-adaptive \\
        BFL \cite{bagdasaryan2020backdoor}& backdoor & non-adaptive \\
        DBA \cite{xie2019dba}& backdoor & non-adaptive \\
        RL \cite{li2022learning}& untargeted model poisoning & adaptive \\
        BRL \cite{li2023learning}& backdoor & adaptive \\
    \bottomrule 
\end{tabular}
\vspace{0.2cm}
\caption{A summary of all attacks in the experiments, with their corresponding categories and adaptivities.}
\label{table:attacks}
\vspace{-0.2cm}
\end{table}

\begin{table*}[h!]
\centering
\scriptsize
\begin{tabular}{@{\extracolsep{1pt}}lcccc}
    \toprule
        Settings & \multicolumn{1}{c}{Pre-training} & \multicolumn{1}{c}{Online-adaptation} & \multicolumn{1}{c}{Related figures/tables} \\
        \midrule
        meta-RL & \{NA, IPM, LMP, BFL, DBA\} & \{IPM, LMP, BFL, DBA, IPM+BFL, LMP+DBA\}  & \cref{table:1}, \cref{fig:untargeted,fig:additional,fig:ablation} \\
        meta-SG & \{RL, BRL\} & \{IPM, LMP, RL, BRL\} & \cref{table:2,table:6}, \cref{fig:untargeted,fig:additional,fig:ablation,fig:backdoors} \\
        meta-SG+ & \{NA, IPM, LMP, BFL, DBA, RL, BRL\} & \{IPM, LMP, RL, BRL\} & \cref{fig:untargeted,fig:additional} \\
    \bottomrule 
\end{tabular}
\vspace{0.2cm}
\caption{A table showcasing the attacks and defenses employed during pre-training and online-adaptation, with links to the relevant figures or tables. RL and BRL are initially target on \{FedAvg, ClipMed, Krum, FLTrust+NC\} during pre-training.}
\label{table:setting}
\vspace{-0.2cm}
\end{table*}

\subsection{Meta Reinforcement Learning Setups}
\paragraph{Reinforcement learning setting.}In our RL-based defense, since both the action space and state space are continuous, we choose the state-of-the-art Twin Delayed DDPG (TD3)~\cite{fujimoto2018addressing} algorithm to individually train the untargeted defense policy and the backdoor defense policy. We implement our simulated environment with OpenAI Gym~\cite{1606.01540} and adopt OpenAI Stable Baseline3~\cite{stable-baselines3} to implement TD3.
The RL training parameters are described as follows: 
the number of FL rounds = 300 rounds, policy learning rate = 0.001, the policy model is MultiInput Policy, batch size = 256, and $\gamma$ = 0.99 for updating the target networks.
The default $\lambda=0.5$ when calculating the backdoor rewards.

\paragraph{Meta-learning setting} 
The attack domains (i.e., potential attack sets) are built as follows:
For meta-RL, we consider IPM~\cite{xie2020fall}, LMP~\cite{fang2020local}, EB~\cite{bhagoji2019analyzing} as three possible attack types. For meta-SG against untargeted model poisoning attack, we consider RL-based attacks~\cite{li2022learning} trained against Krum~\cite{blanchard2017machine} and ClipMed~\cite{li2022learning, yin2018byzantine, sun2019can} as initial attacks. For meta-SG against backdoor attack, we consider RL-based backdoor attacks~\cite{li2023learning} trained against Norm-bounding~\cite{sun2019can} and NeuroClip~\cite{wang2023mm} (Prun~\cite{wu2020mitigating}) as initial attacks. For meta-SG against mix type of attacks, we consider both RL-based attacks~\cite{li2022learning} and RL-based backdoor attacks~\cite{li2023learning} described above as initial attacks.

At the pre-training stage, we set the number of iterations $T=100$. In each iteration, we uniformly sample $K=10$ attacks from the attack type domain (see Algorithm~\ref{algo:meta-rl} and Algorithm~\ref{algo:meta-sl}). For each attack, we generate a trajectory of length $H=200$ for MNIST ($H=500$ for CIFAR-10), and update both attacker's and defender's policies for 10 steps using TD3 (i.e., $l=N_{\A}=N_{\D}=10$). At the online adaptation stage, the meta-policy is adapted for $100$ steps using TD3 with $T=10$, $H=100$ for MNIST ($H=200$ for CIFAR-10), and $l=10$. Other parameters are described as follows: single task step size $\kappa=\kappa_{\A}=\kappa_{\D}= 0.001$, meta-optimization step size $=1$, adaptation step size $=0.01$.

\paragraph{Space compression}
Following the BSMG model, it is natural to use $w_g^t$ as the state, and $\{\widetilde{g}_k^t\}_{k=1}^{M_1+ M_2}$ or $w_g^{t+1}$ as the action for the attacker and the defender, respectively, if the federated learning model is small.
However, when we use federated learning to train a high-dimensional model (i.e., a large neural network), the original state/action space will lead to an extremely large search space that is prohibitive in terms of training time and memory space. We adopt the RL-based attack in~\cite{li2022learning} to simulate an adaptive model poisoning attack and the RL-based local search in~\cite{li2023learning} to simulate an adaptive backdoor attack, both having a $3$-dimensioanl real action spaces after space comparison (see ). We further restrict all malicious devices controlled by the same attacker to take the same action. To compress the state space, we reduce $w_g^t$ to only include its last two hidden layers for both attacker and defender.

Our approach rests on an RL-based synthesis of existing specialized defense methods against mixed attacks, where multiple defenses can be selected at the same time and combined with dynamically tuned hyperparameters. The following specialized defenses are selected for our implementation. For training stage aggregation-based defenses, we first normalize {the magnitude of} all gradients to a threshold $\alpha\in(0,\max_{i\in \mathcal{S}^t}\{\|g_i^t\|\}]$, then apply coordinate-wise trimmed mean~\cite{yin2018byzantine} with trimmed rate $\beta\in [0,1)$. For post-training defense, NeuroClip~\cite{wang2023mm} with clip range $\varepsilon$ or Prun~\cite{wu2020mitigating} with mask rate $\sigma$ is applied. The concrete approach used in each of the above defenses can be replaced by other defense methods. The key novelty of our approach is that instead of using a fixed and hand-crafted algorithm as in existing approaches, we use RL to optimize the policy network $\pi_\D(a_\D^t|s^t;\theta)$. Similar to RL-based attacks, the most general action space could be the set of global model parameters. However, the high dimensional action space will lead to an extremely large search space that is prohibitive in terms of training time and memory space. Thus, we limit the action space to $a^t_\D:=(\alpha^t, \beta^t, \varepsilon^t \slash \sigma^t)$. Note that the execution of our defense policy is lightweight, without using any extra data for evaluation/validation.

\begin{figure*}
\centering
    \includegraphics[width=0.6\textwidth]{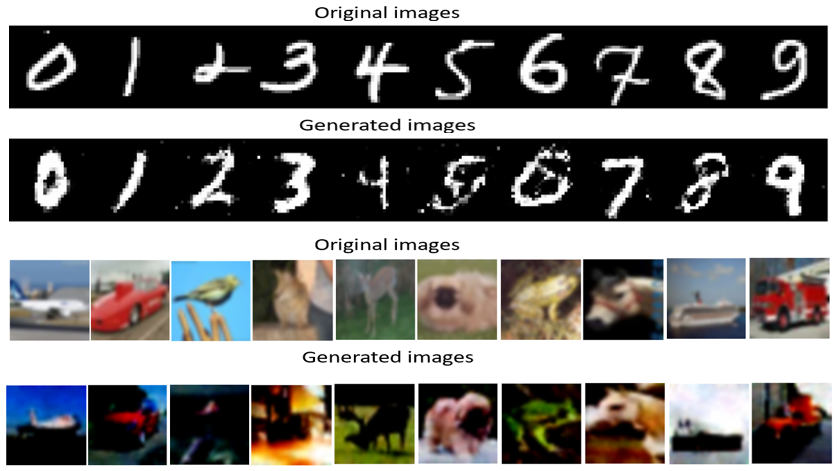}
  \caption{Self-generated MNIST images using conditional GAN~\cite{mirza2014conditional} (second row) and CIFAR-10 images using a diffusion model~\cite{sohl2015deep} (fourth row). 
  }
  \label{fig:generated}
\end{figure*}

\begin{figure*}
\centering
    \includegraphics[width=0.7\textwidth]{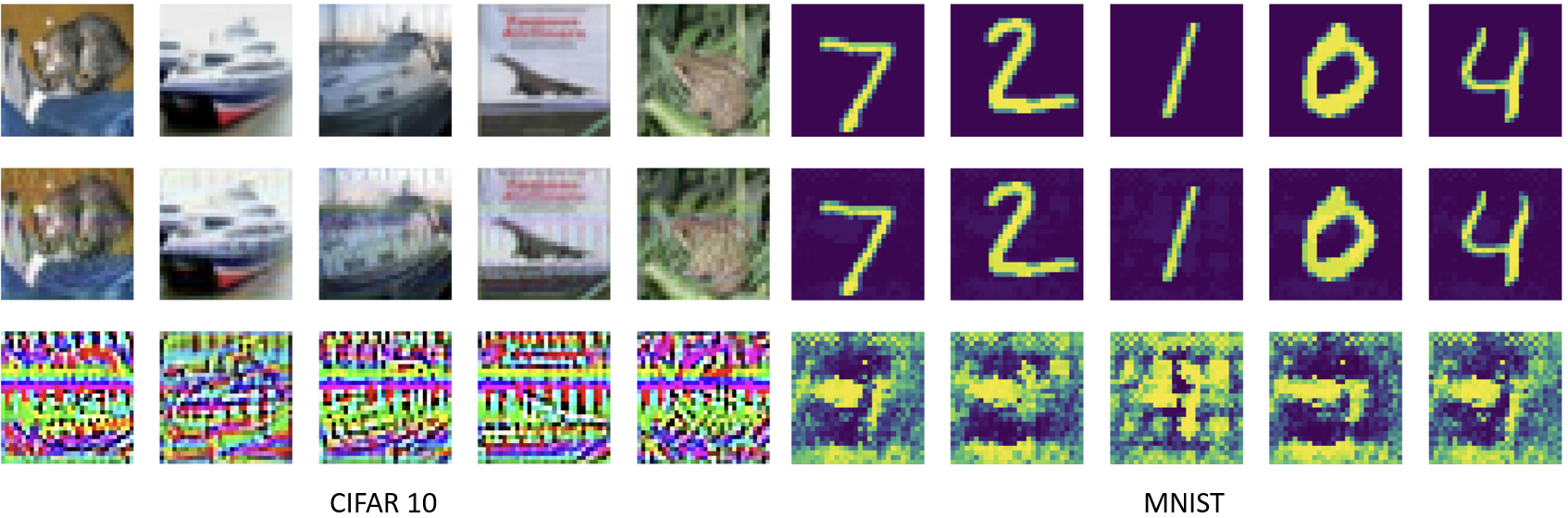}
  \caption{Generated backdoor triggers using GAN-based models~\cite{doan2021lira}. Original image (first row). Backdoor image (second row). Residual (third row).}
  \label{fig:gan_trigger}
\end{figure*}

\begin{figure*}
    \centering
    \includegraphics[width=.99\textwidth]{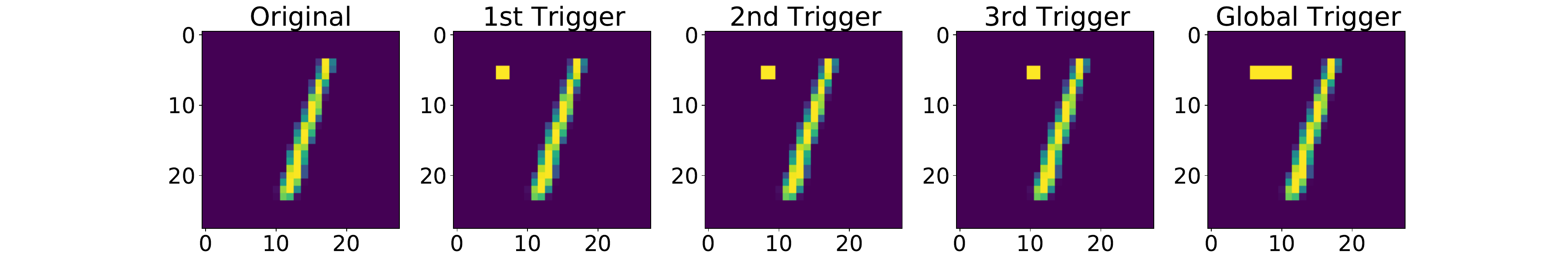}
     \caption{\small{MNIST backdoor trigger patterns. The global trigger is considered the default poison pattern and is used for backdoor accuracy evaluation. The sub-triggers are used by pre-training and DBA only.}}
    \label{fig:mnist_dba}
\end{figure*}

\begin{figure*}
    \centering
    \includegraphics[width=1\textwidth]{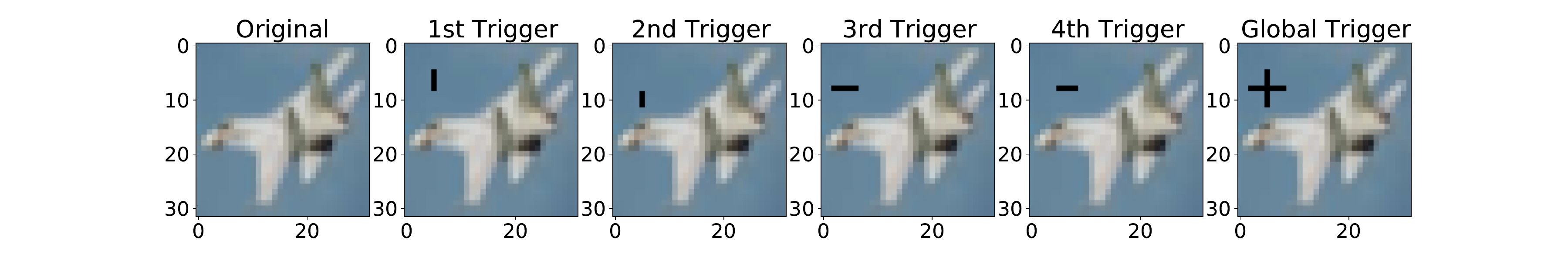}
    \caption{\small{CIFAR-10 fixed backdoor trigger patterns. The global trigger is considered the default poison pattern and is used for online adaptation stage backdoor accuracy evaluation. The sub-triggers are used by pre-training and DBA only.}}
    \label{fig:cifar10_dba}
\end{figure*}

\subsection{Self-generated data}
\label{app:self-data}
We begin by acknowledging that the server only holds a small amount of initial data (200 samples with $q=0.1$ in this work) learned from first 20 FL rounds using inverting gradient~\cite{geiping2020inverting}, to simulate training set with 60,000 images (for both MNIST and CIFAR-10) for FL. This limited data is augmented using several techniques, such as normalization, random rotation, and color jittering, to create a more extensive and varied dataset, which will be used as an input for generative models. 

For MNIST, we use the augmented dataset to train a Conditional Generative Adversarial Network (cGAN) model~\cite{mirza2014conditional,odena2017conditional} built upon the codebase in~\cite{Pytorch-cGAN}. The cGAN model for the MNIST dataset comprises two main components - a generator and a discriminator, both of which are neural networks. Specifically, we use a dataset with 5,000 augmented data as the input to train cGAN, keep the network parameters as default, and set the training epoch as 100.

For CIFAR-10, we leverage a diffusion model implemented in~\cite{cifar-diffusion} that integrates several recent techniques, including a Denoising Diffusion Probabilistic Model (DDPM)~\cite{ho2020denoising}, DDIM-style deterministic sampling~\cite{song2020denoising}, continuous timesteps parameterized by the log SNR at each timestep~\cite{kingma2021variational} to enable different noise schedules during sampling. The model also employs the `v' objective, derived from Progressive Distillation for Fast Sampling of Diffusion Models \cite{salimans2022progressive}, enhancing the conditioning of denoised images at high noise levels. During the training process, we use a dataset with 50,000 augmented data samples as the input to train this model, keep the parameters as default, and set the training epoch as 30.

\subsection{Simulated Environment}
\label{app:sim-env}
To further improve efficiency and privacy, the defender simulates a smaller FL system when solving the game. In our experiments, we include 10 clients in pre-training while using 100 clients in the online FL system. The simulation relies on a smaller dataset (generated from root data) and endures a shorter training time (100 (500) FL rounds for MINST (CIFAR-10) v.s. 1000 rounds in online FL experiments). Although the offline simulated Markov game deviates from the ground truth, the learned meta-defense policy can quickly adapt to the real FL during the online adaptation, as shown in our experiment section.

\paragraph{Backdoor attacks}
We consider the trigger patterns shown in Fig.~\ref{fig:gan_trigger} and Fig.~\ref{fig:cifar10_dba} for backdoor attacks. For triggers generated using GAN (see Fig.~\ref{fig:gan_trigger}), the goal is to classify all images of different classes to the same target class (all-to-one). For fixed patterns (see Fig.~\ref{fig:cifar10_dba}), the goal is to classify images of the airplane class to the truck class (one-to-one). The default poisoning ratio is 0.5 in both cases. The global trigger in Fig.~\ref{fig:cifar10_dba} is considered the default poison pattern and is used for the online adaptation stage for backdoor accuracy evaluation.
In practice, the defender (i.e., the server) does not know the backdoor triggers and targeted labels. To simulate a backdoor attacker's behavior, we first implement multiple GAN-based attack models as in~\cite{doan2021lira} to generate worst-case triggers {(which maximizes attack performance given backdoor objective)} in the simulated environment. 
Since the defender does not know the {poisoning ratio $\rho_i$} and target label of the attacker's poisoned dataset (involved in the attack objective $F'$), we {approximate} the attacker's reward function as $r_{\mathcal{A}}^t=-F''(\widehat{w}^{t+1}_g)$, $F''(w):=\min_{c\in C} [\frac{1}{M_1}\sum_{i=1}^{M_1}\frac{1}{|D_i'|}\sum_{j=1}^{|D_i'|}\ell(w,(\hat{x}_i^j,c))]-\frac{1}{M_2}\sum_{i=M_1+1}^M f(\omega, D_i)$. $F''$ differs $F'$ only in the first $M_1$ clients, where we use a strong target label (the minimizer) as a surrogate to the true label $c^*$.

\begin{figure*}
    \centering
    \includegraphics[width=0.7\textwidth]{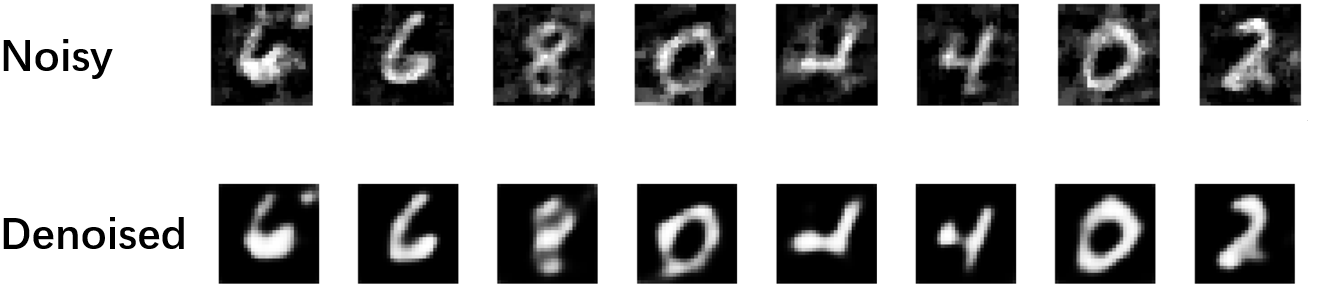}
    \caption{\small{Examples of reconstructed images using inverting gradient (before and after denoising)}}
    \label{fig: mnist_distribution.png}
\end{figure*}

\begin{figure*}
    \centering
    \includegraphics[width=0.4\textwidth]{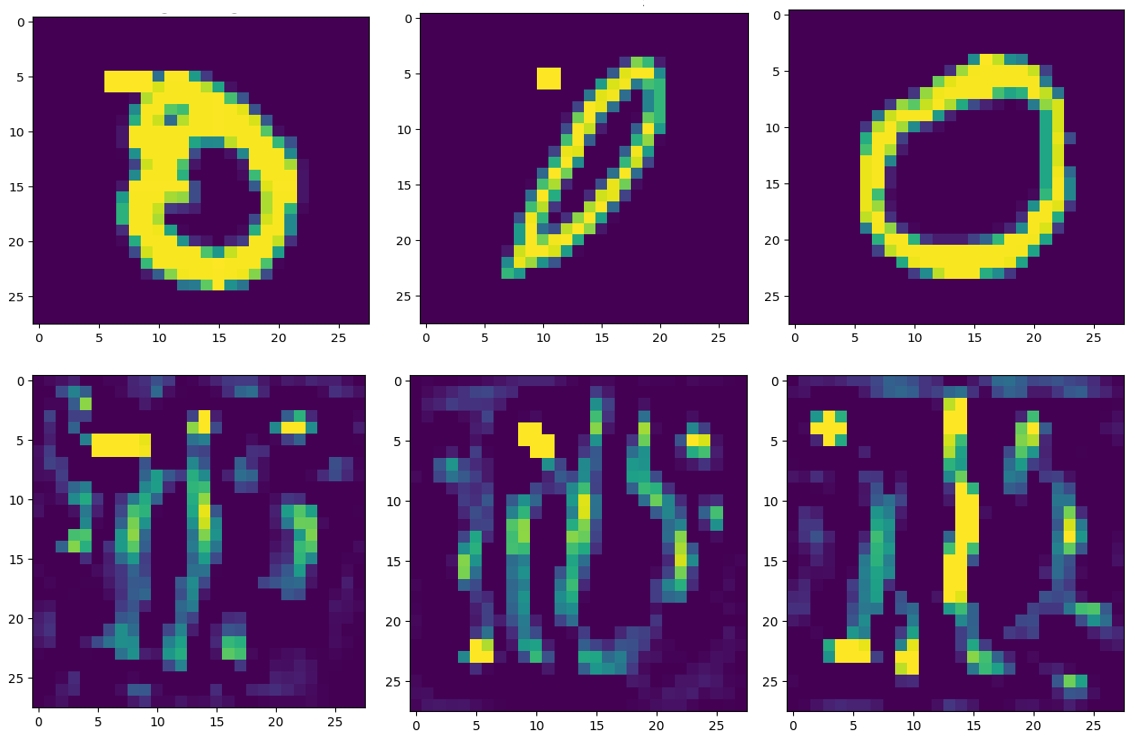}
    \caption{\small{Reversed MNIST backdoor trigger patterns.
    Original triggers (first row). Reversed triggers (second row)}}
    \label{fig:reversed_all.png}
\end{figure*}

\paragraph{Inverting gradient/reverse engineering}
In invert gradient, we set the step size for inverting gradients $\eta' = 0.05$,  the total variation parameter $\beta =0.02$, optimizer as Adam, the number of iterations for inverting gradients $max\_iter = 10,000$, and learn the data distribution from scratch. The number of steps for distribution learning is set to $\tau_E=100$. 32 images are reconstructed (i.e., $B'=32$) and denoised in each FL epoch. If no attacker is selected in the current epoch, the aggregate gradient estimated from previous model updates is reused for reconstructing data. To build the denoising autoencoder, a Gaussian noise sampled from $0.3\mathcal{N}(0,1)$ is added to each dimension of images in $D_{reconstructed}$, which are then clipped to the range of [0,1] in each dimension. The result is shown in Fig.~\ref{fig: mnist_distribution.png}.

In the process of reverse engineering, we use Neural Cleanse \cite{wang2019neural} to find hidden triggers (See Fig.~\ref{fig:reversed_all.png}) connected to backdoor attacks. This method is essential for uncovering hidden triggers and for preventing such attacks. 
In particular, we use the global model, root-generated data, and inverted data as inputs to reverse backdoor triggers. The Neural Cleanse class from ART is used for this purpose. The reverse engineering process in this context involves using the generated backdoor method from the Neural Cleanse defense to find the trigger pattern to which the model is sensitive. The returned pattern and mask can be visualized to understand the nature of the backdoor. 

\paragraph{Online adaptation and execution}
During the online adaptation stage, the defender starts by using the meta-policy learned from the pre-training stage to interact with the true FL environment while collecting new samples $\{s,a,\widetilde{r},s'\}$. Here, the estimated reward $\widetilde{r}$ is calculated using the self-generated data and simulated triggers from the pertaining stage, as well as new data inferred online through methods such as inverting gradient~\cite{geiping2020inverting} and reverse engineering~\cite{wang2019neural}. Inferred data samples are blurred using data augmentation~\cite{shorten2019survey} real distributions) while protecting clients' privacy. For a fixed {number} of FL rounds ({e.g., $50$ for MNIST and $100$ for CIFAR-10 in our experiments}), the defense policy will be updated using gradient ascents from the collected trajectories. Ideally, the defender's adaptation time  (including the time for collecting new samples and updating the policy) should be significantly less than the whole FL training period so that the defense execution will not be delayed. In real-world FL training, the server typically waits for up to $10$ minutes before receiving responses from the clients~\cite{bonawitz2019towards,kairouz2021advances}, enabling defense policy's online update with enough episodes.

\section{Additional Experiment Results}
\label{app:add-exp}
\paragraph{More untargetd model poisoning/backdoor results.} As shown in Fig.~\ref{fig:additional}, similar to results in Fig.~\ref{fig:untargeted} as described in Section~\ref{sec:exp}, meta-SG plus achieves the best performance (slightly better than meta-SG) under IPM attacks for both MNIST and CIFAR-10. On the other hand, meta-SG performs the best (significantly better than meta-RL) against RL-based attacks for both MNIST and CIFAR-10. Notably, Krum can be easily compromised by RL-based attacks by a large margin.  In contrast, meta-RL gradually adapts to adaptive attacks, while meta-SG displays near-immunity against RL-based attacks. In addition, we illustrate results under backdoor attacks and defenses on MNIST in~\cref{table:2}.

\begin{figure*}[t]
    \centering
  \begin{subfigure}{0.24\textwidth}
      \centering
          \includegraphics[width=\textwidth]{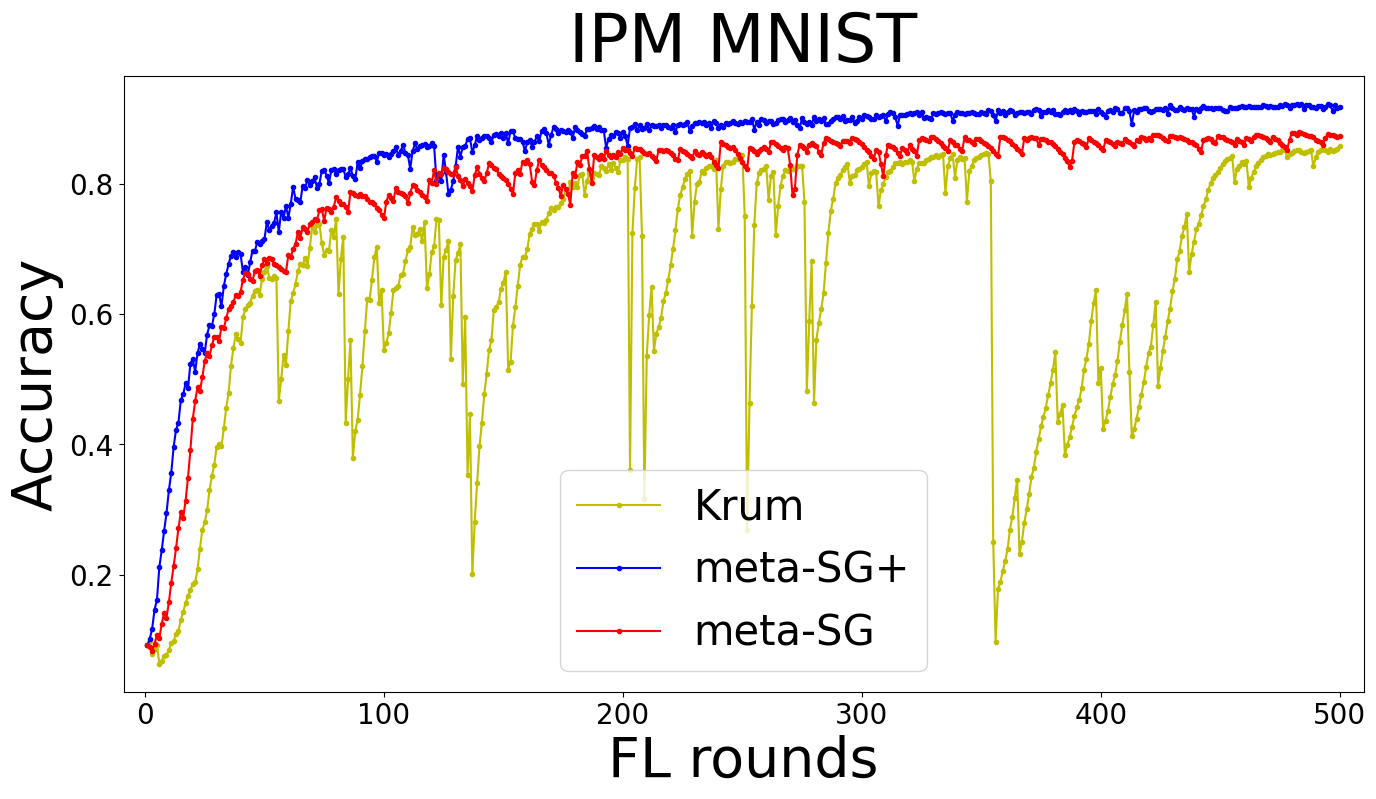}
  \end{subfigure}
  \hfill
    \begin{subfigure}{0.24\textwidth}
      \centering
          \includegraphics[width=\textwidth]{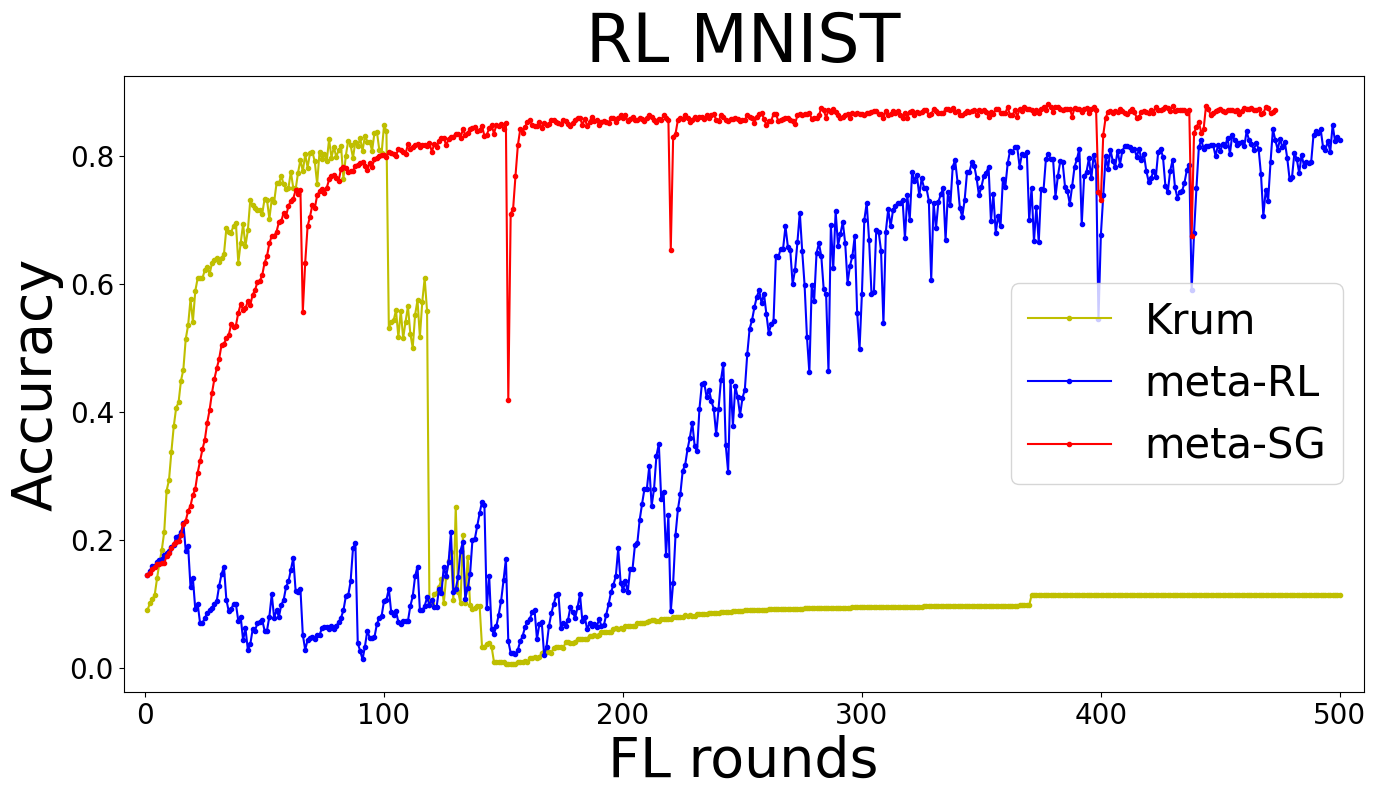}
  \end{subfigure}
  \hfill
    \begin{subfigure}{0.24\textwidth}
      \centering
          \includegraphics[width=\textwidth]{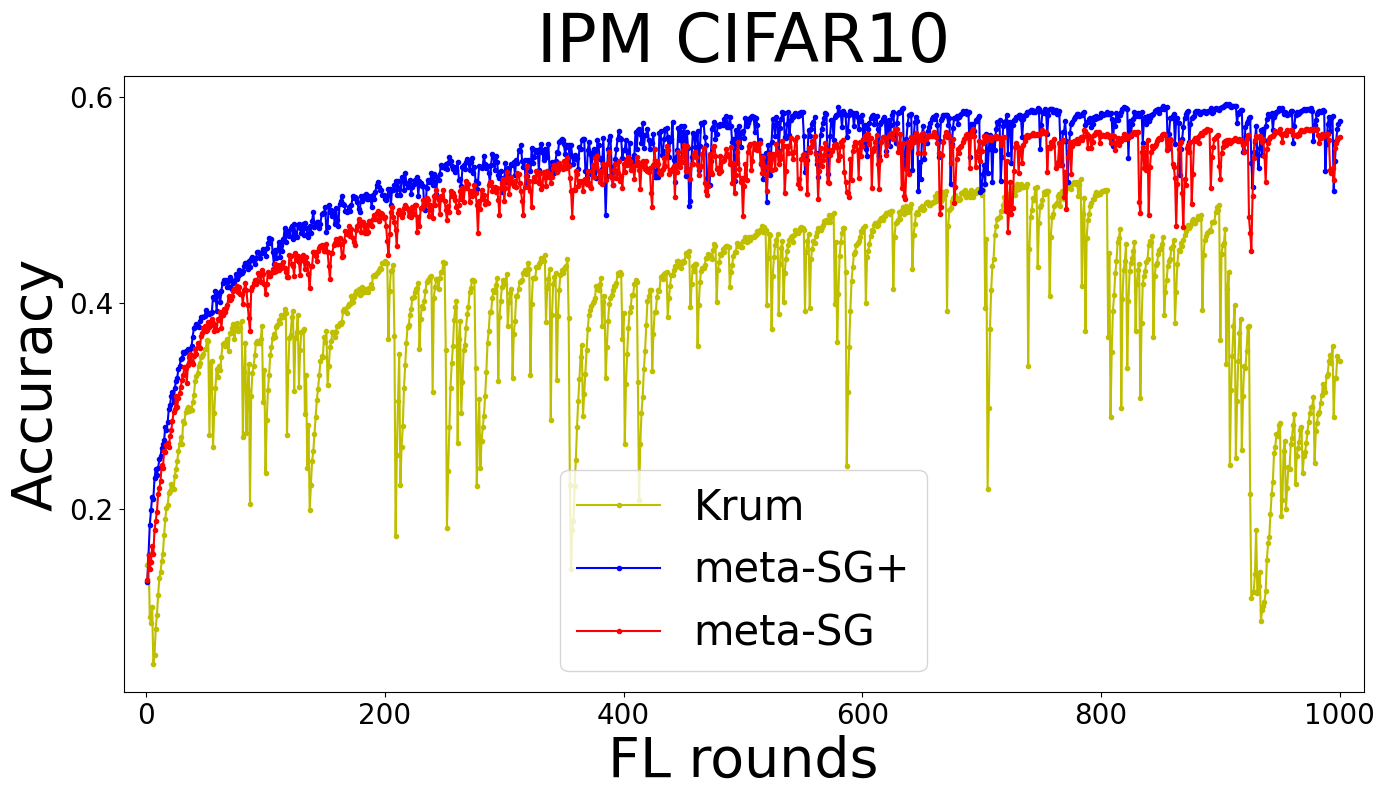}
  \end{subfigure}
  \hfill
    \begin{subfigure}{0.24\textwidth}
      \centering
          \includegraphics[width=\textwidth]{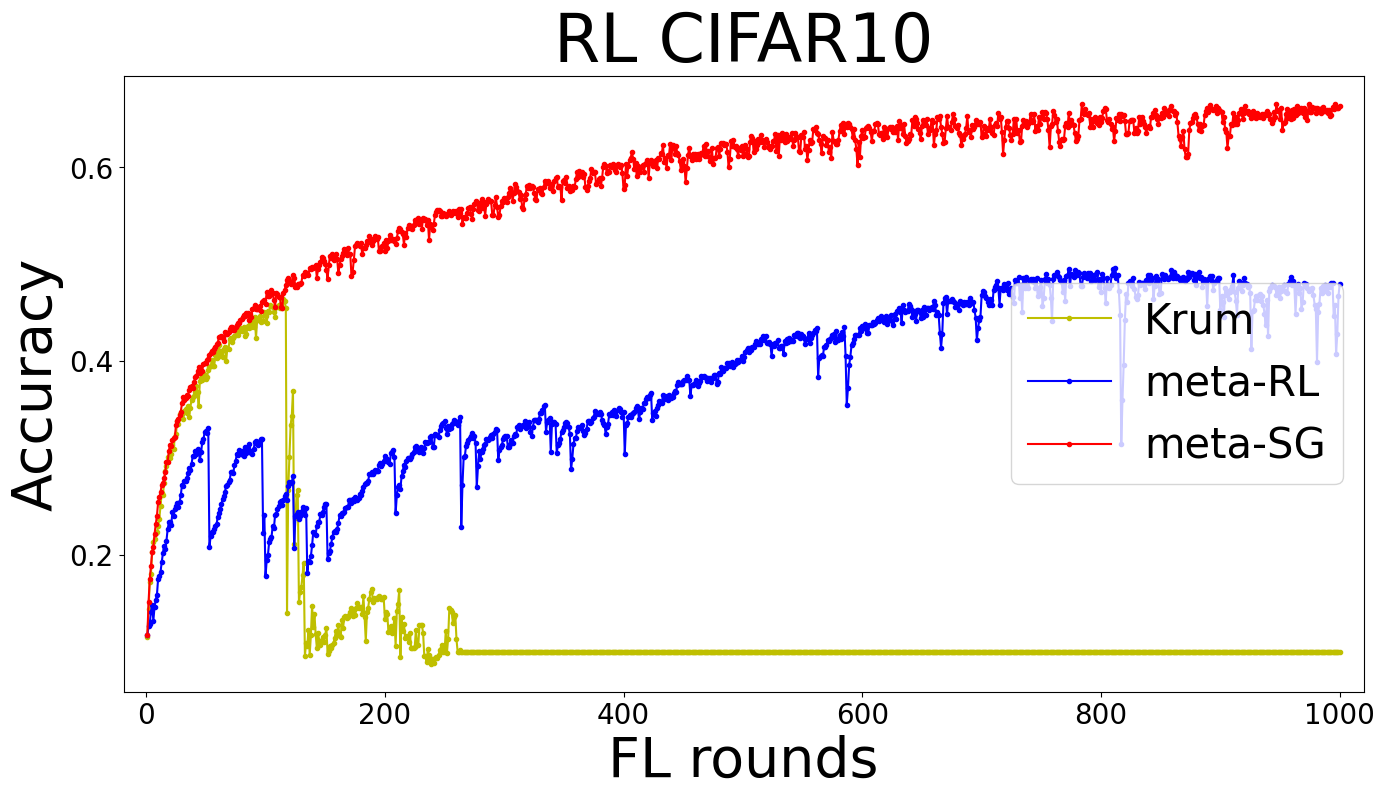}
  \end{subfigure}
  \caption{\small{Comparisons of defenses against untargeted model poisoning attacks (i.e., IPM and RL) on MNIST and CIFAR-10. RL-based attacks are trained before FL round 0 against the associate defenses (i.e., Krum and meta-policy of meta-RL/meta-SG). All parameters are set as default and all random seeds are fixed.}}\label{fig:additional}
\end{figure*}

\begin{table}[h!]
\centering
\begin{tabular}{@{\extracolsep{1pt}}lcccc}
    \toprule
        Bac & \multicolumn{1}{c}{Krum} & \multicolumn{1}{c}{CRFL} & \multicolumn{1}{c}{Meta-SG (ours)} \\
        \midrule
        BFL & $0.8257$ & $0.4253$  & $0.0086$ \\
        DBA & $0.4392$ & $0.215$ & $0.2256$ \\
        BRL & $0.9901$ & $0.8994$ & $0.2102$ \\
    \bottomrule 
\end{tabular}
\vspace{0.1cm}
\caption{Comparisons of average backdoor accuracy (lower the better) after 250 FL rounds under backdoor attacks and defenses on MNIST. All parameters are set as default and all random seeds are fixed.}
\label{table:2}
\vspace{-0.2cm}
\end{table}

\begin{table*}[h!]
\centering
\begin{tabular}{@{\extracolsep{1pt}}lccccc}
    \toprule
        Acc & \multicolumn{1}{c}{NA/FedAvg} & \multicolumn{1}{c}{Root data} & \multicolumn{1}{c}{Generated data} & \multicolumn{1}{c}{Pre-train only} &\multicolumn{1}{c}{Online-adapt only} \\
        \midrule
        MNIST & $0.9016$ & $0.4125$  & $0.5676$ & $0.6125$ & $0.4134$\\
        CIFAR-10 & $0.7082$ & $0.2595$ & $0.3833$ & $0.1280$ & $0.3755$\\
    \bottomrule 
\end{tabular}
\vspace{0.1cm}
\caption{Ablation studies of only using root data/generated dataset in simulated environment to learn the FL model and the defense performance under IPM of directly applying meta-policy learned from pre-training without adaptation/starting online adaptation from a randomly initialized defense policy. Results are average globel model accuracy after $250$ ($500$) FL rounds on MNIST (CIFAR-10). All parameters are set as default and all random seeds are fixed..}
\label{table:3}
\vspace{-0.2cm}
\end{table*}

\paragraph{Importance of pre-training and online adaptation}
As shown in \Cref{table:3}, the pre-training is to derive defense policy rather than the model itself. Directly using those shifted data (root or generated) to train the FL model will result in model accuracy as low as $0.2$-$0.3$ ($0.4$-$0.5$) for CIFAR-10 (MNIST) in our setting.  
Pre-training and online adaptation are indispensable in the proposed framework. Our experiments in \Cref{table:3} indicate that directly applying defense learned from pre-training w/o online adaptation and adaptation from randomly initialized defense policy w/o pre-training both fail to address malicious attacks, resulting in global model accuracy as low as $0.3$-$0.6$ ($0.1$-$0.4$) on MNIST (CIFAR-10). In the absence of adaptation, meta policy itself falls short of the distribution shift between the simulated and the real environment. Likewise, the online adaptation fails to attain the desired defense policy without the pre-trained policy serving as a decent initialization.

\paragraph{Biased/Limited  root data}
We evaluate the average model accuracy after 250 FL epochs under the meta-SG framework against the IPM attack, using root data with varying i.i.d. levels (as defined in the experiment setting section). Here, q = 0.1 (indicating the root data is i.i.d.) serves as our baseline meta-SG, as presented in the paper. We designate class 0 as the reference class. For instance, when q = 0.4, it indicates a $40\%$ probability for each data labeled as class 0 within the root data, while the remaining $60\%$ are distributed equally among the other classes. We observe that when q is as high as 0.7, there is one class (i.e., 3) missing in the root data. Although, through inverting methods in online adaptation, the defender can learn the missing data in the end, it suffered the slower adaptation compared with a good initial defense policy.
In addition, we test the average model accuracy after 250 FL epochs under meta-SG against IPM attack using different numbers of root data (i.e., 100, 60, 20), where 100 root data is our original meta-SG setting in the rest of paper. We overserve that when number of root data is 20, two classes of data are missing (i.e., 1 and 5).

\begin{table}[h!]
\centering
\subfloat[Ablation study of biased root data.]{
\begin{tabular}{@{\extracolsep{1pt}}lcccc}
    \toprule
        Biased Level & \multicolumn{1}{c}{q = 0.1} & \multicolumn{1}{c}{q = 0.4} & \multicolumn{1}{c}{q = 0.7}\\
        \midrule
        Acc & $0.8951$ & $0.8612$  & $0.7572$\\
    \bottomrule 
\end{tabular}
}
\qquad\qquad
\subfloat[Ablation study of limited root data.]{%
\begin{tabular}{@{\extracolsep{1pt}}lcccc}
    \toprule
        Number of Root Data & \multicolumn{1}{c}{100} & \multicolumn{1}{c}{60} & \multicolumn{1}{c}{20}\\
        \midrule
        Acc & $0.8951$ & $0.8547$ & $0.6902$\\
    \bottomrule 
\end{tabular}
}
\vspace{0.1cm}
\caption{Results of the average model accuracy on MNIST after 250 FL epochs under meta-SG against IPM attack using root data with (a) different i.i.d levels and (b) different numbers of root data. All random seeds are fixed and all other parameters are set as default.}
\label{table:5}
\vspace{-0.2cm}
\end{table}

\begin{table}[h!]
\centering
\begin{tabular}{@{\extracolsep{1pt}}lcccc}
    \toprule
        Acc/Bac & \multicolumn{1}{c}{NormBound 0.2} & \multicolumn{1}{c}{NormBound 0.1} & \multicolumn{1}{c}{NormBound 0.05}\\
        \midrule
        DBA & $0.6313/0.9987$ & $0.5192/0.6994$  & $0.3610/0.4392$\\
        IPM+BFL & $0.6060/0.5123$ & $0.4917/0.2104$ & $0.3614/0.2253$\\
        \toprule
        Acc/Bac & \multicolumn{1}{c}{NeuroClip 10} & \multicolumn{1}{c}{NeuroClip 6} & \multicolumn{1}{c}{NeuroClip 1}\\
        \midrule
        DBA & $0.6221/0.9974$ & $0.6141/0.9984$ & $0.2515/0.0002$\\
        IPM+BFL & $0.1/0.0020$ & $0.1/0$ & $0.1/0$\\
    \bottomrule 
\end{tabular}
\vspace{0.1cm}
\caption{Results of manually tuning norm threshold~\cite{sun2019can} and clipping range~\cite{wang2023mm}. All other parameters are set as default and all random seeds are fixed.}
\label{table:4}
\vspace{-0.2cm}
\end{table}

\paragraph{Generalization to unseen adaptive attacks}
We thoroughly search related works considering adaptive attacks in FL and find very limited works (with solid and lightweight open-source implementation) that can be used as our benchmark. As a result, we introduce two new benchmark adaptive attack methods in the testing stage as unseen adaptive attacks: (1) adaptive LMP!\cite{fang2020local}, which requires access to normal clients’ updates in each FL round, and (2) RL attack~\cite{li2022learning} restricted 1-dimensional action space (i.e., adaptive scalar factor) compared with the baseline 3-dimensional RL attack~\cite{li2022learning} showing in our paper. The defender in pre-training only interacts with the 3-dimensional RL attack. We test the average model accuracy after 250 FL epochs under meta-SG against different (unseen) adaptive attacks. What is interesting here is that meta-SG can achieve even better performance against unseen attacks.

\begin{table}[h!]
\centering
\begin{tabular}{@{\extracolsep{1pt}}lcccc}
    \toprule
        Attack Methods & \multicolumn{1}{c}{Model Acc} \\
        \midrule
        3-dimensional RL & $0.8652$\\
        Adaptive LMP & $0.8692$\\
        1-dimensional RL & $0.8721$\\
    \bottomrule 
\end{tabular}
\vspace{0.1cm}
\caption{Comparisons of average model accuracy after 250 FL rounds under different adaptive attacks on MNIST. All parameters are set as default and all random seeds are fixed.}
\label{table:6}
\vspace{-0.2cm}
\end{table}

\end{document}